\documentclass[twoside,11pt]{article}
\usepackage[utf8]{inputenc}
\usepackage{amsmath}
\usepackage{amssymb}
\usepackage{mathtools}
\usepackage{bm}
\usepackage{algorithm}
\usepackage{algpseudocode}
\usepackage{url}
\usepackage{graphicx}

\newcommand*{\Ex}[2]{\mathop{\mathbb{E}}_{#1}{\left[#2\right]}}
\newcommand*{\KL}[2]{D_{\rm KL}\left[#1 || #2\right]}
\newcommand{\inv}{^{-1}}
\newcommand{\vzero}{\mathbf{0}}

\newcommand{\vb}{\mathbf{b}}
\newcommand{\vf}{\mathbf{f}}
\newcommand{\vk}{\mathbf{k}}
\newcommand{\vm}{\mathbf{m}}

\newcommand{\vw}{\mathbf{w}}
\newcommand{\vx}{\mathbf{x}}
\newcommand{\vy}{\mathbf{y}}
\newcommand{\vz}{\mathbf{z}}

\newcommand{\vH}{\mathbf{H}}
\newcommand{\vI}{\mathbf{I}}
\newcommand{\vK}{\mathbf{K}}
\newcommand{\vM}{\mathbf{M}}

\newcommand{\vW}{\mathbf{W}}
\newcommand{\vSigma}{\boldsymbol{\Sigma}}

\newcommand{\vmu}{\boldsymbol{\mu}}
\newcommand{\vnu}{\boldsymbol{\nu}}
\newcommand{\vzeta}{\boldsymbol{\zeta}}

\newcommand{\vpsi}{\boldsymbol{\psi}}
\newcommand{\vPsi}{\boldsymbol{\Psi}}
\newcommand{\cL}{\mathcal{L}}
\newcommand{\cD}{\mathcal{D}}

\newcommand{\cR}{\mathcal{R}}

\newcommand{\cX}{\mathcal{X}}
\newcommand{\cY}{\mathcal{Y}}
\newcommand{\argmax}{\mathop{\rm arg~max}\limits}
\newcommand{\argmin}{\mathop{\rm arg~min}\limits}

\newcommand{\T}{^{\rm T}}
\algnewcommand{\IIf}[1]{\State\algorithmicif\ #1\ \algorithmicthen}
\algnewcommand{\EndIIf}{\unskip\ \algorithmicend\ \algorithmicif}

\usepackage{jmlr2e}

\jmlrheading{1}{2021}{1-32}{4/21}{10/00}{Hideaki Ishibashi and Hideitsu Hino}

\ShortHeadings{Stopping Criterion for Active Learning Based on Error Stability}{Ishibashi and Hino}
\firstpageno{1}

\begin{document}

\title{Stopping Criterion for Active Learning Based on Error Stability}

\author{\name Hideaki Ishibashi \email ishibashi@brain.kyutech.ac.jp \\
       \addr Department of Human Intelligence Systems\\
       Kyushu Institute of Technology\\
       Kitakyushu-shi, Fukuoka 808-0135, Japan
       \AND
       \name Hideitsu Hino \email hino@ism.ac.jp \\
       \addr Department of Statistical Modeling\\
       The Institute of Statistical Mathematics\\
       Tachikawa, Tokyo 190-8562, Japan, and\\
       RIKEN AIP, Nihonbashi, Tokyo 103-0027, Japan
       }

\editor{}

\maketitle

\begin{abstract}%   <- trailing '%' for backward compatibility of .sty file
\end{abstract}
Active learning is a framework for supervised learning to improve the predictive performance by adaptively annotating a small number of samples. To realize efficient active learning, both an acquisition function that determines the next datum and a stopping criterion that determines when to stop learning should be considered. In this study, we propose a stopping criterion based on error stability, which guarantees that the change in generalization error upon adding a new sample is bounded by the annotation cost and can be applied to any Bayesian active learning. We demonstrate that the proposed criterion stops active learning at the appropriate timing for various learning models and real datasets.


\section{Introduction}
In supervised learning, increasing the size of the training dataset would improve the predictive accuracy. However, a sufficient number of data often cannot be obtained since the annotation cost is high. Active learning (AL) is a framework for acquiring data to improve the generalization error of a predictor and it is effective when only a small number of data is available~\citep{Settles2009,Dasgupta2011,hino2020active}. The effectiveness of AL depends on both the acquisition function used for selecting effective data to improve the generalization error of the predictor and the stopping criterion used for determining the stopping timing. However, most conventional AL methods adopt the ``fixed budget'' approach, which acquires samples until the size of the labeled dataset reaches to a predetermined budget. The fixed budget approach considers only the annotation cost and the generalization error is not considered; hence, it often causes overexploitation or a deficiency. Therefore, a stopping criterion considering the generalization error is required to utilize AL effectively.

The optimal stopping timing of AL is determined subjectively by users since there is a trade-off between decreasing the annotation cost and increasing the performance of the predictor. Therefore, it is essential for a stopping criterion of AL to have the ability to determine whether to aggressively or conservatively stop learning, taking into account the trade-off~\citep{Altschuler2019}. However, it is not easy for a stopping criterion to satisfy this requirement for the following two reasons. First, it is not easy to evaluate the performance of the predictor in AL. Although the generalization error is estimated by using the test dataset in a standard setting of supervised learning, it is not appropriate for AL, which aims to learn from a small amount of data. Second, it is not easy to compare the annotation cost and performance numerically on the same scale since the units of measurement of the annotation cost and performance are different. 

In this study, we propose a stopping criterion based on the upper bound of the change in generalization error upon adding a new sample. The proposed criterion guarantees that the change in generalization error is less than a pre-determined threshold. The change of generalization errors guaranteed by the proposed criterion is normalized with respect to the range of the generalization error. Therefore, the proposed criterion stops learning at a similar timing for any dataset when the same threshold is used. Furthermore, the proposed criterion can be applied to arbitrary Bayesian predictive models. We demonstrate the effectiveness of the proposed criteria for several Bayesian ALs.

The contributions of this study are as follows.
\begin{enumerate}
    \item {\it Versatile stopping criteria for Bayesian active learning are proposed.}\\
    The proposed criterion can be applied to arbitrary Bayesian AL active learning algorithms. The criterion controls the stop timing of AL based on the stability of the generalization error. We experimentally demonstrate that the proposed criterion has a high correlation with the generalization error and its threshold can be determined without depending on the dataset as long as the same learning model is used.
    \item {\it A bound of the difference between generalization errors is proved.}\\
    We prove that the difference between expected generalization errors with respect to Bayes posteriors can be bounded based on a probably approximately correct (PAC) Bayesian framework. Unlike conventional PAC-Bayesian learning, our bound does not assume independence among samples. In this sense, the bound is suitable for AL. Moreover, we can guarantee that the proposed bound converges to zero when the posterior converges.
    \item {\it The proposed criteria are applied to several Bayesian AL algorithm.}\\
    In this study, we demonstrate that the proposed criteria can stop the following Bayesian AL algorithms at the appropriate timing: Bayesian linear regression, Bayesian logistic regression, Gaussian process regression, and dropout-based Bayesian deep learning. In particular, we derive analytical expressions for both the Kullback-Leibler (KL) divergence between GP posteriors and the bound of the KL-divergence between posteriors of deep Bayesian learning.
\end{enumerate}

The rest of the paper is organized as follows. Section~2 summarizes the AL framework and the existing measures used for stopping AL algorithms. Section~3 defines the optimal stopping timing. Sections~4 describes the method for automatically stopping AL based on error stability and another interpretation of the proposed method from the viewpoint of a martingale. Sections~5 demonstrates the effectiveness of the proposed method for four learning models. Section~6 is devoted to conclusions.
We note that a preliminary version of this work is presented in~\citep{Ishibashi2020}.

\section{Active learning and its stopping criteria}
\subsection{Active learning}
Let $\vx \in \cX$ and $y \in \cY$ be an input variable and the corresponding output variable, respectively. The purpose of supervised learning is to estimate a predictor $f : \vx \mapsto \mathbb{E}[y|\vx]$ from a training dataset. Active learning is a framework for selecting a small number of datasets, which is useful for improving the generalization error of the predictive model by iterating the following two processes: (i) estimate the predictor from current training data, and (ii) acquire new training data by maximizing an acquisition function. We denote the acquisition function by $g$. Then, the next data is selected as
\begin{equation*}
    \vx^\ast = \argmax_{\vx \in \cX}g\left(\vx \mid f\right).
\end{equation*}
We note that there are batch active learning methods, but for simplicity, we consider only one sample at a time for annotation. The acquisition functions are classified into two approaches, namely, an informative sample selection approach and a representative sample selection approach~\citep{Dasgupta2011}. The informative sample selection approach selects the most uncertain datum for the predictor~\citep{Lewis1994,Scheffer2001,Yang2015,Seung1992,Freund1992,Houlsby2011,Kirsch2019}. On the other hand, the representative sample selection approach selects a datum representing the overall input distribution~\citep{Nguyen2004,Settles2008}. These two approaches can be combined to achieve the optimal acquisition function~\citep{Xu2003,Donmez2007,Huang2010,Karzand2020}. Recently, as another approach, methods of learning acquisition functions have also been proposed~\citep{Konyushkova2017,Sener2018,THK2021}.

\subsection{Conventional stopping criteria for active learning}
Most active learning methods adopt the fixed-budget approach, which stops learning when the number of annotated data reaches a predetermined size. However, this approach tends to cause undersampling or oversampling since it is rare to know the appropriate sampling size for satisfactory prediction accuracy in advance. While the predictor may be useless because of the lack of generalization error in the case of undersampling, the efficient sampling in AL may be wasted in the case of oversampling. Therefore, a stopping criterion considering the generalization error is required.

The conventional stopping criterion for AL can be classified into three approaches: accuracy-based, confidence-based and stability-based approaches. In the accuracy-based approach, predictive error is evaluated by using unlabeled or past training data. A typical method is to evaluate the predictive error by using queried or selected unlabeled data~\citep{Zhu2007,Zhu2008,Zhu2008b,Laws2008}. In the confidence-based approach, the stopping timing is determined by the confidence of prediction with respect to unlabeled data. For example, the margin of a support vector machine (SVM), entropy, mutual information or the agreement of learners is used for evaluating uncertainty~\citep{Schohn2000,Vlachos2008,Zhu2007,Krause2007,Tomanek2007,Olsson2009}. In the stability-based approach, active learning is stopped by evaluating the difference between values before and after obtaining a new training datum. Stopping criteria based on stability such as the parameter change given new data, the agreement between the most recent learner and the previous learner and the predicted change in F-measure have been proposed~\citep{Bloodgood2009,Bloodgood2013,Altschuler2019}. However, except for a small number of methods~\citep{Krause2007,Bloodgood2013,Altschuler2019}, most of these methods lack theoretical underpinning. Moreover, the stopping criterion proposed by~\citet{Krause2007} assumes Gaussian process regression as a predictive model and it is necessary to discretize the domain of explanatory variables. The criteria proposed by~\citet{Bloodgood2013} and \citet{Altschuler2019} can be applied to any classification model but not to regression models.

Theoretically, a stopping criterion for disagreement-based active learning~\citep{Balcan2009, Hanneke2014} could be applicable to wide variety of learning models. It is also equipped with a learning-theoretical stopping criterion. However, it is difficult to apply disagreement-based active learning to practical problems.

\subsection{PAC-Bayesian learning}

As another approach, we can consider terminating active learning by using PAC-Bayesian learning. Let $\cD$ be the true distribution for a pair of input variable $\vx\in\cX$ and its corresponding output variable $y\in\cY$. We assume that a training dataset $S=\{(\vx_i,y_i)\}^n_{i=1}$ is generated by $\cD$. We specify the dataset at time $t$ in AL as $S_t$, which is the union of the initial dataset $S_0=\{(\vx_i,y_i)\}^{n_0}_{i=1}$ and the acquired dataset $\{(\vx_i,y_i)\}^t_{i=1}$, and the size of $S_t$ is denoted by $n_t$. Denoting a predictor parameterized by $\theta$ as $f_\theta$, we respectively define the training and generalization errors for $f_\theta$ as
\begin{equation*}
    \hat{\cL}(\theta) \coloneqq \frac{1}{n}\sum^n_{i=1}l(f_\theta(\vx_i),y_i), \quad
    \cL(\theta) \coloneqq \Ex{\cD}{l(f_\theta(\vx),y)}
\end{equation*}
where $l(f_\theta(\vx),y) \in [a,b]$ is the loss function for the predictor. Let $p(\theta|S)$ be a Bayesian posterior distribution given $S$. In the PAC-Bayesian learning framework~\citep{McAllester1999,Guedj2019,Germain2016}, we consider evaluating the upper bound of the expectation of the generalization error with respect to the Bayesian posterior by using the expectation of the training error and the Kullback--Leibler (KL) divergence between the posterior and the prior~\citep{McAllester1999}. For example, let the expected training error and generalization error with respect to the Bayesian posterior $p(\theta|S)$ be
\begin{equation*}
    \hat{\cL}(p(\theta|S)) \coloneqq \Ex{p(\theta|S)}{\hat{\cL}(\theta)}, \quad \cL(p(\theta|S)) \coloneqq \Ex{p(\theta|S)}{\cL(f)},
\end{equation*}
and the KL-divergence is defined as $\KL{p(\theta)}{q(\theta)}\coloneqq \Ex{p(\theta)}{\log p(\theta)/q(\theta)}$. Then, \cite{McAllester1999} proved that the following bound holds with probability at least $1-\delta$:
\begin{equation*}
   \cL(p(\theta|S))\leq \hat{\cL}(p(\theta|S))+\sqrt{\frac{\KL{p(\theta|S)}{p(\theta)}+\log{\frac{2\sqrt{n}}{\delta}}}{2n}},
\end{equation*}
where $p(\theta)$ is any prior probability distribution independent of $S$. This bound is only applicable to classification problems, but several bounds for regression problems have been considered by~\citet{Germain2016}.

PAC-Bayesian bounds are promising tools for developing a stopping criterion for active learning, but conventional PAC-Bayesian bounds suffer the following two drawbacks. First, most of the PAC-Bayesian bounds assume that samples are realizations of i.i.d. random variables. This assumption is not suitable for active learning since a new sample depends on previous data samples. PAC-Bayesian bounds for non-i.i.d. samples have been proposed~\citep{Alquier2018,Seldin2012}. These approaches assume that the difference between the generalization error and the training error follows $\alpha$-mixing or a martingale, but it is not trivial to apply them to an AL. Therefore, to the best of our knowledge, there is no suitable PAC-Bayesian bound for active learning. Second, it is difficult to determine a suitable threshold for stopping AL because the expected generalization error depends on the dataset and model. This makes it difficult for users to set an appropriate parameter for the dataset at hand. It is also important, depending on the application, to stop active learning rapidly without waiting the convergence of learning or to stop learning once confident of convergence of the generalization error. 

In this study, we propose a stopping criterion based on the error stability of the expected generalization error. The proposed criterion does not require any assumption with respect to models like PAC-Bayesian learning. On the other hand, the proposed criterion does not assume that samples are i.i.d. random variables, unlike PAC-Bayesian learning. Furthermore, it is easy to determine a threshold of the proposed criterion, since the range of the threshold is restricted to $[0,1]$ for any dataset.

\section{Definition of optimal stopping timing based on error stability}
We denote the posterior given $S_t$ by $p(\theta|S_t)$, where $S_t$ is the training dataset at time $t$ in active learning. Let $\cL(\theta) \in [a,b]$ be the generalization error for $\theta$ and $\Ex{p(\theta|S_t)}{\cL(\theta)}$ be the expected generalization error for the posterior. Assuming that the sampling cost is constant  $\kappa\in[0,\infty)$, the standard definition of the optimal stopping timing is
\begin{align}
    t^\ast \coloneqq \argmin_t\left\{\Ex{p(\theta|S_t)}{\cL(\theta)}+\kappa t\right\},
    \label{eq_standard_optimal_stopping_timing}
\end{align}
which balances the expected generalization error and sampling cost. 
However, it is impossible to obtain $t^\ast$ for the following two reasons. First, it is difficult to calculate the first term of Eq.~\eqref{eq_standard_optimal_stopping_timing} since it is not realistic to estimate the generalization error by using a test dataset in active learning problems. Second, it is not easy to know the appropriate sampling cost $\kappa$ in a comparable unit to the generalization error. 

In this study, we consider estimating the optimal stopping timing by using the difference between the generalization errors before and after obtaining a new training datum, which is denoted by $\Delta\cL_i\coloneqq\Ex{p(\theta|S_i)}{\cL(\theta)}-\Ex{p(\theta|S_{i-1})}{\cL(\theta)}$. Note that $\Ex{p(\theta|S_0)}{\cL(\theta)}$ is a constant; thus, the following equation holds:
\begin{align}
    t^*=& \argmin_t\{\Ex{p(\theta|S_t)}{\cL(\theta)}+\kappa t\} \notag \\
    =& \argmin_t\left\{\sum^t_{i=1}\frac{\Delta\cL_i}{\gamma}+\frac{\kappa}{\gamma} t\right\},
    \label{eq_diff_optimal_stopping_timing}
\end{align}
where $\gamma$ is a parameter used to normalize the range of $\lambda \coloneqq \kappa/\gamma$ to $[0,1]$. The optimal stopping timing is equal to the timing satisfying $\Delta\cL_i/\gamma\leq\lambda$ when the first derivative of $\Ex{p(\theta|S_t)}{\cL(\theta)}$ is a monotonically decreasing function, but this does not hold in general. Therefore, instead of assuming the condition, we suppose that $\Delta\cL_i$ gradually decreases with time. Then, the above optimal stopping timing can be approximated as
\begin{equation}
    t^* \approx \min_t\left\{t>0 \middle| \frac{|\Delta\cL_t|}{\gamma}\leq \lambda\right\}.
    \label{eq_optimal_stopping_timing}
\end{equation}
With this definition, we can stop active learning at the optimal stopping timing by monitoring $|\Delta \cL_t|/\gamma$ in the sequential process of adding a new sample and updating the predictive model. Since the condition $|\Delta\cL_t|\leq \kappa$ is called error stability in~\citet{Bousquet2002}, we call the stopping criterion based on Eq.~\eqref{eq_optimal_stopping_timing} the error stability based stopping criterion. We note that in the above formulation, there is no need to determine the sampling cost $\kappa$. Instead, we have to set the threshold $\lambda$, which should be easier than determining the sampling cost for each model and dataset because the range of the threshold is restricted to $[0,1]$ and it has an intuitive interpretation as explained in the next section.

\section{Stopping criterion based on error stability}
The error stability based criterion in Eq.~\eqref{eq_optimal_stopping_timing} contains the difference in generalization error $\Delta \cL_t$, which is not directly available. In this section, we derive the upper bound of $\Delta \cL_t$, estimable by using the posterior distribution of Bayesian predictive models.

\subsection{Proposed stopping criterion}
In this study, we consider Bayesian predictive models. Let $p(\theta)$ and $p(\theta|S)$ respectively be the prior and posterior distributions given a set of observations $S$ defined as
\begin{equation*}
    p(\theta | S)=\frac{1}{Z} \exp\{-\rho\hat{\cL}(\theta)\}p(\theta),
\end{equation*}
where $\rho \geq 0$ is a parameter controlling the trade-off between training error and the prior and $Z$ is the normalization constant corresponds to the Bayesian marginal likelihood. Then, the following theorem holds:
\begin{theorem}
Let $p(\theta|S)$ and $p(\theta|S')$ be the posteriors given $S$ and $S'$, respectively. For the generalization error $\cL(\theta) \in [a,b]$, the following inequality holds\footnote{More generally, for any $v\geq\mathbb{E}[\cL^2(\theta)]$, the following inequality holds:
\begin{equation*}
    -\frac{v}{b-a}(\exp\left\{W_0(u')+1\right\}-1)\leq \Ex{p(\theta|S)}{\cL(\theta)}-\Ex{p(\theta|S')}{\cL(\theta)} \leq \frac{v}{b-a}(\exp\left\{W_0(u)+1\right\}-1),
\end{equation*}
where $u'\coloneqq((b-a)^2\KL{p(\theta|S')}{p(\theta|S)}-v)/ve$ and $u\coloneqq((b-a)^2\KL{p(\theta|S)}{p(\theta|S')}-v)/ve$.
}:
\begin{equation}
    -(b-a)r(p(\theta|S'),p(\theta|S))\leq \Ex{p(\theta|S)}{\cL(\theta)}-\Ex{p(\theta|S')}{\cL(\theta)} \leq (b-a)r(p(\theta|S),p(\theta|S')),
    \label{eq_gap_bound}
\end{equation}
where
\begin{align*}
    r(p(\theta),q(\theta)):=\exp\left\{W_0\left(\frac{\KL{p(\theta)}{q(\theta)}-1}{e}\right)+1\right\}-1
\end{align*}
and $e$ is the base of the natural logarithm. Here, $W_0(\cdot)$ is the principal branch of the Lambert $W$ function~\citep{Corless1996}.
\label{theorem_gap_bound}
\end{theorem}
\begin{proof}
 See Appendix~\ref{sec_proof_theorems}.
\end{proof}
\begin{figure}[t!]
    \centering
    \includegraphics[width=8cm]{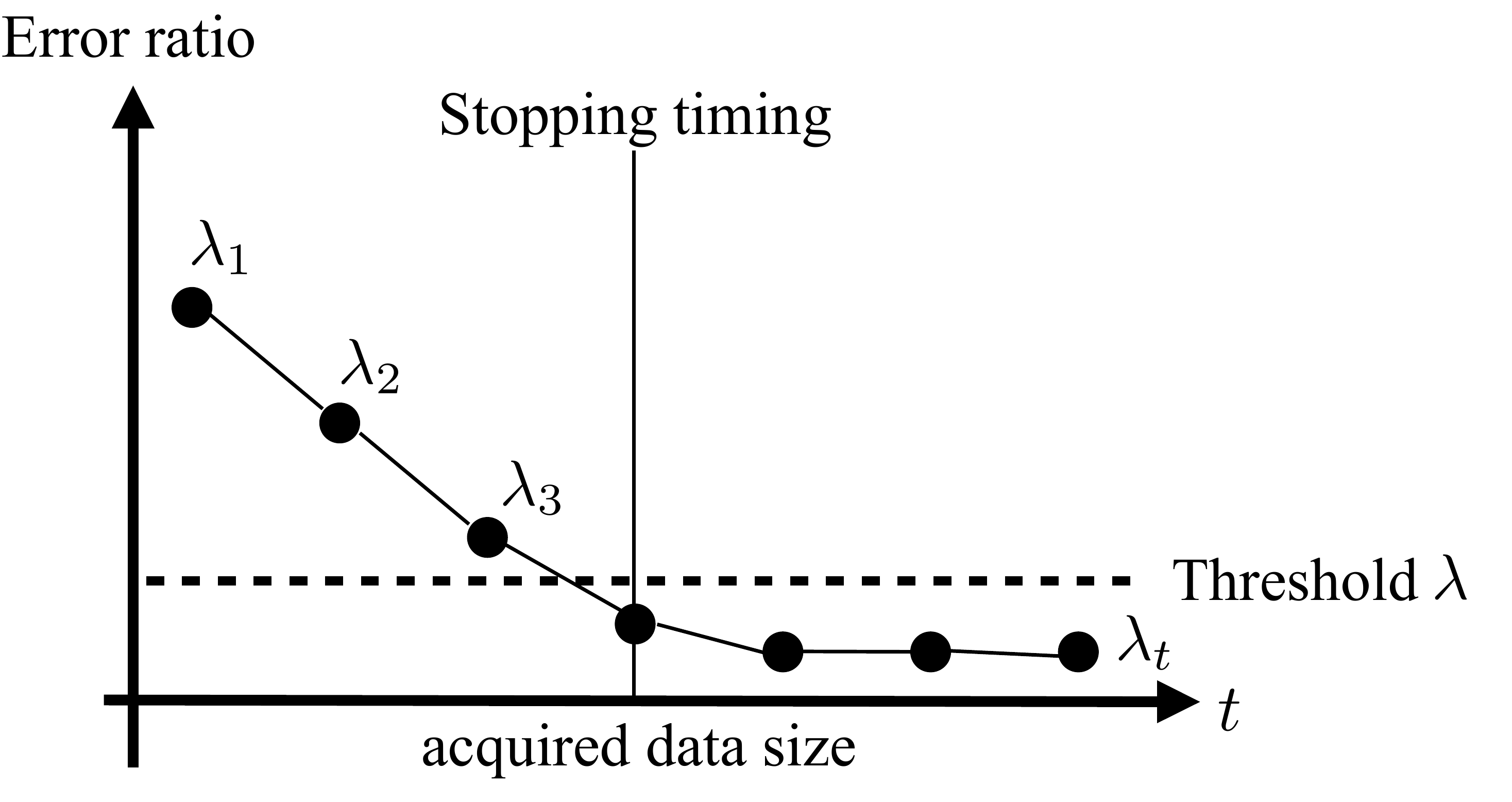} \\
    \caption{Concept of the proposed criterion. We stop active learning when the error ratio is less than a predetermined threshold $\lambda$ as above.}
    \label{fig_concept_of_the_proposed_criterion}
\end{figure}
Let $S_t$ be a training dataset at time $t$ and $p(\theta|S_t)$ be the posterior given $S_t$, and define
\begin{equation*}
    r_t = r(p(\theta|S_t),p(\theta|S_{t-1}))+r(p(\theta|S_{t-1}),p(\theta|S_r)).
\end{equation*}
Then, by using Theorem~\ref{theorem_gap_bound}, the upper bound of $|\Delta\cL_t|/\gamma$ is derived as
\begin{equation*}
    \frac{|\Delta\cL_t|}{\gamma} \leq \frac{(b-a)r_t}{\gamma}.
\end{equation*}
To remove the range $(b-a)$ from the upper bound, we define $\gamma = (b-a) \tilde{\gamma}$. 
In this study, as shown in Fig.~\ref{fig_concept_of_the_proposed_criterion}, we propose to stop active learning when the following condition holds:
\begin{equation}
    \lambda_t \coloneqq \frac{r_t}{\tilde{\gamma}} \leq \lambda,
    \label{eq_error_ratio}
\end{equation}
where the error stability is guaranteed to satisfy $|\Delta\cL_t|/(b-a)\tilde{\gamma}\leq \lambda$. When $r_t$ is larger than $r_1$, it is not less than the threshold; hence, we choose $\tilde{\gamma}=r_1$ in order to ensure that the range of $\lambda$ is $[0,1]$ and call $\lambda_t$ the ``error ratio''\footnote{In practice, since $r_1$ can be extremely larger than the other $r_t, t > 1$, the learning process may stop  before the convergence of generalization error when $\tilde{\gamma}=r_1$. To avoid this problem, we iterate at least $10$ acquisitions and set $\tilde{\gamma}=\min \{r_t\}^m_{t=1}$, where $m=1$ for experiments with a Bayesian deep neural network model, and $m=10$ with other models in our implementation.}. The concrete algorithm is shown in Algorithm~\ref{alg_clt_stopping_criterion}.

The proposed criterion has three favorable properties. First, the proposed criterion does not require any assumption including independence between samples. This is suitable for an active learning framework, since acquired samples affect the selection of the next sample. Second, $\lambda_t$ goes to zero because $W_0\left(-1/e\right)=-1$ holds when the KL-divergence becomes zero. Therefore, the proposed criterion guarantees that after observing a sufficient number of samples, the active learning is terminated. Finally, the proposed criterion is based on the upper bound of the gap between the normalized generalization errors without specifying the range of the generalization error. This is a major advantage of the proposed criterion because, in general, it is difficult to know the range of the generalization error in advance. 

\begin{remark}
In this paper, to apply AL, we explain the proposed bound as the bound for the gap between expected generalization errors with respect to Bayes posteriors. However, Theorem~\ref{theorem_gap_bound} can be applied to any measurable function and any probability density function. 
\end{remark}

\begin{algorithm}[t]
\caption{Error stability based stopping criterion}
\label{alg_clt_stopping_criterion}
\begin{algorithmic}
    \State Input $\lambda>0$, $S_0$, $\tilde{\gamma}$
    \State $t \leftarrow 0$
    \Comment Sample initial dataset
    \State $p_t(\theta)\leftarrow \frac{1}{Z} \exp\{-\rho\hat{\cL}(\theta)\}p(\theta)$
    \Comment Calculate posterior
    \State $\tilde{\lambda}_t \leftarrow 1$
    \While{$\lambda < \lambda_t$}
        \State $t \leftarrow t+1$
        \State $\vx_t=\argmax_{\vx\in\cX}g(\vx|\theta)$
        \Comment Sample by AL
        \State $S_t \leftarrow S_{t-1} \cup \{(\vx_t,y_t)\}$
        \Comment Update dataset
        \State $p_t(\theta)\leftarrow \frac{1}{Z} \exp\{-\rho\hat{\cL}(\theta)\}p(\theta)$
        \Comment Calculate posterior
        \State $r(p_t,p_{t-1}) \leftarrow \exp\left\{W_0\left(\frac{\KL{p_t(\theta)}{p_{t-1}(\theta)}-1}{e}\right)+1\right\}-1$
        \State $r(p_{t-1},p_t) \leftarrow \exp\left\{W_0\left(\frac{\KL{p_{t-1}(\theta)}{p_t(\theta)}-1}{e}\right)+1\right\}-1$
        \State $\lambda_t \leftarrow \frac{r(p_t,p_{t-1})+r(p_{t-1},p_t)}{\tilde{\gamma}}$
        \Comment Calculate error ratio
    \EndWhile
\end{algorithmic}
\end{algorithm}

\subsection{Interpretation of the proposed stopping criterion from the viewpoint of martingale theory}
The proposed criterion can be interpreted from the viewpoint of martingale theory. Denoting the expected generalization error with respect to the posterior distribution $q_t(\theta)$ by $E_t$, we consider the sequence $E^t_1=\{E_1,E_2,\ldots,E_t\}$. 
Then, the following theorem holds:
\begin{theorem}
\label{theorem_stopping_timing_azuma}
The sequence $E^t_1$ is assumed to be a supermartingale\footnote{The assumption is reasonable because the expected generalization error is expected to decrease with increasing sample size.}, namely, $\mathbb{E}[E_i | E^{i-1}_1] \leq E_{i-1}$. For any $0<\eta<1$ and $0<\delta<1$ satisfying
\begin{equation*}
    \lambda_t\leq \sqrt{-\frac{2}{\log(\delta/2)}}\eta=\lambda,
\end{equation*}
the following inequality holds with probability at least  $1-\delta$:
\begin{align*}
    \left|E_t-\mathbb{E}[E_t|E^{t-1}_1]\right| \leq (b-a)\gamma\eta.
\end{align*}
\end{theorem}
\begin{proof}
 See Appendix~\ref{sec_proof_theorems}.
\end{proof}

$\mathbb{E}[E_i|E^{i-1}_1]$ is the conditional expectation of the expected generalization error $E_i$ with respect to a new sample generated by random sampling given $E^{i-1}_1$. We can conclude that the stopping criterion guarantees the effectiveness of active learning since $E_i-\mathbb{E}[E_i|E^{i-1}_1]<0$ indicates that the expected generalization error by using AL is larger than the conditional expectation of the expected generalization error.

\section{Experimental results}

In this section, the versatility of the proposed criterion is demonstrated through active learning experiments using datasets from the UCI machine learning repository~\citep{Dua2017}\footnote{Source code to reproduce the experimental results are available from~\url{https://github.com/hideaki-ishibashi/stopping_AL}}. As Bayesian active learners, four AL models are considered: Bayesian ridge regression (BRR), Bayesian logistic regression (BLR), Gaussian process regression (GPR) and Bayesian deep neural network (BDNNs). A description of the datasets is given in Table~\ref{tbl_dataset}. Every feature of these datasets is normalized to have zero mean and a standard deviation of one.

\begin{table}[thb]
  \centering
  \caption{Description of datasets.}
  \scalebox{0.9}{
  \begin{tabular}{l|c|c|c|c}
                  Name of dataset  & Sample size & Feature dimension & Test size & AL model \\ \hline \hline
    {\tt{Power plant}} & 9568 & 4 & 2000 & BRR, GPR \\ \hline
    {\tt{Protein}} & 45730 & 9 & 2000 & BRR, GPR \\ \hline
    {\tt{Gas emission}} & 36733 & 8 & 2000 & BRR, GPR \\ \hline
    {\tt{Grid stability for regression}} & 10000 & 11 & 2000 & BRR, GPR \\ \hline
    {\tt{Grid stability for classification}} & 10000 & 11 & 5000 & BLR \\ \hline
    {\tt{Skin}} & 245057 & 3 & 5000 & BLR \\ \hline
    {\tt{HTRU2}} & 17897 & 7 & 5000 & BLR \\ \hline
    {\tt{MNIST}} & 70000 & 784 & 10000 & BDNNs \\
    \end{tabular}
  }
  \label{tbl_dataset}
\end{table}

\subsection{Evaluation measure}
It is difficult to evaluate the effectiveness of the stopping timing because of the lack of a subjective, ground truth optimal stopping timing.
In fact, the appropriate threshold to stop learning should be problem-dependent. The proposed stopping criterion is based on the estimate of the upper bound of the generalization error. To demonstrate the validity of the proposed method, we evaluate the correlation coefficient between the error ratio and the expected generalization error, which is estimated by using a sufficient number of test data. When the error ratio is highly correlated with the expected generalization error, it provides additional evidence that our proposed method is suitable for stopping active learning based on the expected generalization error. Let $E^t_1=\{E_1,E_2,\ldots,E_t\}$ and $\Lambda^t_1=\{\lambda_1,\lambda_2,\ldots,\lambda_t\}$ be sequences of expected generalization errors and error ratios, respectively. Since $\lambda_t$ satisfying $\lambda_t\neq \min \Lambda^t_1$ does not affect the stopping timing of AL, we evaluate the correlation between the following two sets:
\begin{align*}
    \hat{E}^t_1 &= \left\{E_i \mid i = \argmin \Lambda^i_1, \quad i\in\{1,2,\ldots,t\}\right\}, \\
    \hat{\Lambda}^t_1 &= \left\{\lambda_i \mid i= \argmin \Lambda^i_1, \quad i\in\{1,2,\ldots,t\}\right\}.
\end{align*}

The form of the likelihood function is highly dependent on the predictive model, and the appropriate threshold will be different for different models. However, when we use the same predictive model, regardless of the dataset, it is desirable that the same threshold results in approximately the same stopping timing in terms of the convergence of the generalization error. In this experiment, we verify that the proposed criterion stops AL at about the same timing for various datasets when the learning model is fixed. We set thresholds $\lambda = 0.02, 0.015, 0.01$ for BRR; $\lambda = 0.3, 0.2, 0.1$ for BLR; $\lambda = 0.05, 0.04, 0.03$ for GPR; and $\lambda = 0.2, 0.15, 0.1$ for BDNNs.

\subsection{Active learning models}

We consider the following four active learning models: 

\mbox{}\\
\noindent{\textbf{1. Bayesian ridge regression}}\mbox{}\\
In BRR, the predictor $f : \mathcal{X} \rightarrow \mathcal{Y}$ is modeled by a linear combination of $J$ basis functions, $\vpsi=(\psi_1,\psi_2,\ldots,\psi_J)\T$, that is,
\begin{equation*}
    f(\vx) = \vw\T\vpsi(\vx),
\end{equation*}
where $\vw \in \mathbb{R}^{J}$. We assume that the likelihood function $p(y|\vw,\vx)$ given $S$ and the prior distribution $p(\vw)$ for parameter $\vw$ are written as
\begin{align*}
    p(\vy|\vw)&=\prod^n_{i=1}\mathcal{N}(y_i|\vw\T\vpsi(\vx_i),\beta^{-1}), \\
    p(\vw)&=\mathcal{N}(\vw|\vzero,\alpha\inv\vI),
\end{align*}
where $\vy=(y_1,y_2,\ldots,y_n)\T$, $\beta$ is the accuracy of Gaussian noise and $\alpha$ is the accuracy of the prior. Let $\vPsi=(\vpsi(\vx_1),\vpsi(\vx_2),\ldots,\vpsi(\vx_n))$ be the matrix of feature vectors. Then, the posterior of parameter $\vw$ is derived as
\begin{align*}
    p(\vw|S)=&\mathcal{N}(\vw|\vmu,\vSigma), \\
    \vmu=& \beta\vSigma\vPsi\vy, \\
    \vSigma\inv=&\beta\vPsi\vPsi\T+\alpha\vI.
\end{align*}
We define the acquisition function of BRR by the variance of the predictive distribution
\begin{equation*}
    g(\vx) = \vpsi\T(\vx)\vSigma\vpsi(\vx).
\end{equation*}

To apply the proposed criterion, we have to calculate the KL-divergence between $p(\vw|S)=\mathcal{N}(\vw|\vmu,\vSigma)$ and $p(\vw|S')=\mathcal{N}(\vw|\vmu',\vSigma')$. In BRR, since the posterior becomes a normal distribution, the KL-divergence is explicitly written as
\begin{equation*}
    \KL{p(\vw|S')}{p(\vw|S)} =\frac{1}{2}\left\{{\rm Tr}[\vSigma\inv\vSigma']-\log{\vSigma\inv\vSigma'}+(\vmu-\vmu')\T\vSigma\inv(\vmu-\vmu')-J\right\}.
\end{equation*}

In this experiment, we use the additive model for radial basis function (RBF) bases to define the predictor $f$ as
\begin{equation*}
    f(\vx) = \sum^D_{d=1}\sum^M_{m=1}w_{md}\psi_{md}(\vx),
\end{equation*}
where $\psi_{md}(\vx)$ is the $d$th dimension of the $m$th basis,
\begin{equation*}
    \psi_{md}(\vx)=\exp\left(-\frac{1}{2l^2}(x_d-\xi_m)^2\right).
\end{equation*}
The $M$ centers of bases $\xi_m$ are common for all dimensions and are arranged at equal intervals in the range of the observed explanatory variables. The bandwidth parameter $l$ is set to be $l=\Delta_\xi$, where $\Delta_{\xi}$ is the length between adjacent centers $\xi_m$. The hyperparameters $\alpha$ and $\beta$ are estimated by maximizing the marginal likelihood with the training dataset, that is, by solving the following self-consistent equations:
\begin{align*}
    \alpha =& \frac{\tau}{\vmu\T\vmu},% \\
    \quad 
    \frac{1}{\beta} = \frac{1}{N_T-\tau}(\vy-\vmu\T\vPsi)^2,
\end{align*}
where $\tau=\sum_i\tau_i/(\tau_i+\alpha)$ with the $i$th eigenvalue $\tau_i$ of the matrix $\beta\vPsi\T\vPsi$.

\mbox{}\\
\noindent{\textbf{2. Bayesian logistic regression}}\mbox{}\\
BLR considers the two-class classification problem. When $S$ is observed, the predictor $p(y=1|\vx)$ is modeled by using a linear combination of $J$ basis functions and the logistic function $\sigma(a)=1/(1+\exp(-a))$ as
\begin{equation*}
    p(y=1|\vx) = \sigma(\vw\T\vpsi(\vx)).
\end{equation*}
The likelihood function and the prior distribution of $\vw$ are respectively assumed to have the following forms:
\begin{align*}
    p(\vy|\vw) =& \prod^t_{i=1}{\rm Bern}(y_i|p_i), \\
    p(\vw) =& \mathcal{N}(\vw|\vzero,\alpha^{-1}\vI),
\end{align*}
where $\alpha$ is the accuracy of the prior. Since the posterior distribution of the Bayesian logistic regression model does not have a closed-form representation, we use the Laplace approximation as
\begin{equation}
    p(\vw|S)\approx \mathcal{N}(\vw|\vw_{\rm MAP},\vH),
    \label{eq_logistic_posterior} 
\end{equation}
where $\vw_{\rm MAP}$ is the MAP estimate of the posterior and $\vH$ is the Fisher information matrix for the posterior at $\vw_{\rm MAP}$. We use the entropy of the predictor as the acquisition function:
\begin{equation*}
    g(\vx) = -\sum_{c\in \{0,1\}} p(y=c|\vx)\log p(y=c|\vx).
\end{equation*}

It is easy to calculate the KL-divergence for BLR since the posterior is a normal distribution. From Eq.~\eqref{eq_logistic_posterior}, the KL-divergence between $p(\vw|S)\approx \mathcal{N}(\vw|\vw_{\rm MAP},\vH)$ and $p(\vw|S')\approx \mathcal{N}(\vw|\vw_{\rm MAP}',\vH')$ is derived as
\begin{align*}
    &\KL{p(\vw|S')}{p(\vw|S)} \\
    =& \frac{1}{2}\left\{{\rm Tr}[\vH\inv\vH']-\log{\vH\inv\vH'}+(\vw_{\rm MAP}-\vw_{\rm MAP}')\T\vH\inv(\vw_{\rm MAP}-\vw_{\rm MAP}')-J\right\}.
\end{align*}

We use the additive model of RBF bases and adopt the same basis functions as for BRR. 

\mbox{}\\
\noindent{\textbf{3. Gaussian process regression}}\mbox{}\\
Let $S$ be the observed dataset. In Gaussian process regression, the loss function is assumed to be the negative log likelihood of the Gaussian distribution with accuracy $\beta$, and the prior distribution is obtained as
\begin{equation*}
  \left[
    \begin{array}{c}
      \vy  \\
      f(\vx)
    \end{array}
  \right]
  \sim \mathcal{N}\left(
  \left[
    \begin{array}{c}
      \vzero  \\
      0
\end{array}
  \right]
  ,
  \left[
    \begin{array}{cc}
      \vK+\beta^{-1}\vI & \vk(\vx) \\
      \vk\T(\vx) & k(\vx,\vx)
    \end{array}
  \right]\right),
\end{equation*}
where $k(\vx,\vx)$ is the kernel function, $\vk(\vx)=k(X,\vx)$, $\vK=k(X,X)$, $X=\{\vx_1,\vx_2,\ldots,\vx_n\}$ and $\vy=(y_1,y_2,\cdots,y_n)$. Then, the posterior is defined as
\begin{equation*}
    f(\vx) | \vx, S \sim \mathcal{N}(\mu(\vx),\sigma(\vx,\vx)),
\end{equation*}
where $\mu(\vx) = \vk\T(\vx)(\vK+\beta^{-1}\vI)^{-1}\vy$ and $ \sigma(\vx,\vx) = k(\vx,\vx) - \vk\T(\vx)(\vK+\beta^{-1}\vI)^{-1}\vk(\vx)$.  We adopt the variance of the predictive distribution as the acquisition function:
\begin{equation*}
    g(\vx) = k(\vx,\vx) - \vk\T(\vx)(\vK+\beta^{-1}\vI)^{-1}\vk(\vx).
\end{equation*}

It is not easy to calculate the KL-divergence between GP posteriors since the KL-divergence diverges to infinity in general, but it is computable when the prior of the GP posteriors is fixed. Let $q(f|S_t)$ and $q(f|S_{t-1})$ be the posteriors of $f$ given $S_t$ and $S_{t-1}$, respectively. Then, the following equalities hold:
\begin{align*}
    \KL{p(f|S_{t-1})}{p(f|S_t)}=&\frac{1}{2}\beta \sigma_{t-1}(\vx_{n_t},\vx_{n_t})-\frac{1}{2}\log{(1+\beta \sigma_{t-1}(\vx_{n_t},\vx_{n_t}))}  \\
    &+\frac{1}{2}\frac{\beta \sigma_{t-1}(\vx_{n_t},\vx_{n_t})}{\sigma_{t-1}(\vx_{n_t},\vx_{n_t})+\beta^{-1}}(y_{n_t}-\mu_{t-1}(\vx_{n_t}))^2,\\
    \KL{p(f|S_t)}{p(f|S_{t-1})}=&\frac{1}{2}\log{(1+\beta\sigma_{t-1}(\vx_{n_t},\vx_{n_t}))}-\frac{1}{2}\frac{\sigma_{t-1}(\vx_{n_t},\vx_{n_t})}{\sigma_{t-1}(\vx_{n_t},\vx_{n_t})+\beta\inv}  \\
    &+\frac{1}{2}\frac{\sigma_{t-1}(\vx_{n_t},\vx_{n_t})}{(\sigma_{t-1}(\vx_{n_t},\vx_{n_t})+\beta\inv)^2}(y_{n_t}-\mu_{t-1}(\vx_{n_t}))^2.
\end{align*}
Details of the derivation are described in Appendix~\ref{sec_kl_gps}.

In this study, the kernel function for the prior is the following Gaussian kernel:
\begin{equation*}
    k(\vx,\vx')=\exp\left(-\frac{1}{2l^2}\|\vx-\vx'\|^2\right).
\end{equation*}
The hyperparameters $l$ and $\beta$ are chosen by maximizing the marginal likelihood with training data.

\mbox{}\\
\noindent{\textbf{4. Dropout-based Bayesian deep learning}}\mbox{}\\
We consider the dropout-based BDNNs of $L$-layer perceptron~\citep{Gal2016}, where the $l$th layer has $H_l$ neurons composed of a weight vector denoted by $\vW_l \in \mathbb{R}^{H_l \times H_{l-1}}$ and a bias denoted by $\vb_l \in \mathbb{R}^{H_l}$. The $0$th layer is the input layer and $H_0$ is the number of its neurons. 
In each epoch during the training, several elements of the weight vectors of the dropout-based BDNNs model are randomly set to zero. The dropped out weight vectors of the $l$th layer are described as $\vW_{l}=\vM_{l}{\rm diag}([e_{lh}]^{H_{l-1}}_{h=1})$, where $e_{lh} \sim {\rm Bern}(p_l)$. Let $\sigma_l : \mathbb{R}^{H_{l-1}}\rightarrow \mathbb{R}^{H_l}$ be the $l$th layer's activation function. The output of the $l$th layer is described as
\begin{equation*}
    \vz_l=\sqrt{\frac{1}{H_l}}\sigma_l(\vW_{l}\vz_{l-1}+\vb_l),
\end{equation*}
where $\vzeta_{l-1}$ is the $(l-1)$th layer's output and $\vzeta_0$ is the input vector $\vx \in \mathcal{X}$. Therefore, the predictor is modeled as
\begin{equation*}
    f(\vx)=\sqrt{\frac{1}{H_L}}\sigma_L(\vW_L\vz_L+\vb_L).
\end{equation*}
As the acquisition function, we use batchBALD~\citep{Kirsch2019}, which is an extension of BALD~\citep{Houlsby2011}, to select multiple points simultaneously. It uses the mutual information $I(y;\theta|S,\vx)$ between the output $y$ for the input $\vx$ given training data $S$ and model parameter $\theta$.

Let $\theta$ be the whole parameter of the model. We assume that the prior is $p(\theta)=\prod^L_{l=1}\prod^{H_{l-1}}_{h=1} p(\vw_{lh})p(\vb_l)$, where $p(\vw_{lh})=p(\vb_l)=\mathcal{N}(\vzero,\vI)$. The posterior probability $q(\theta)$ is, by using the mean field approximation, written as the product of marginals:
\begin{equation*}
    q(\theta)=\prod^L_{l=1}p(\vW_l)p(\vb_l).
\end{equation*}
In this formula,
\begin{align*}
    q(\vW_l)=\prod^{H_{l-1}}_{h=1}q(\vw_{lh}), \quad 
    q(\vb_l)=\mathcal{N}(\vnu_l,\sigma^2_l\vI)
\end{align*}
and
\begin{equation*}
    q(\vw_{lh})=p_l\mathcal{N}(\vm_{lh},\sigma^2_l\vI)+(1-p_l)\mathcal{N}(\vzero,\sigma^2_l\vI).
\end{equation*}
Hence $q(\theta)$ is shown to be a mixture of Gaussians. 

The analytical formula for the KL-divergence between mixture distributions is not known, but its upper bound can be derived by using the chain rule for the  KL-divergence~\citep{Do2003,Hershey2007}. Let $p(\theta)=\sum^K_{k=1}\pi_k p_k(\theta)$ and $q(\theta)=\sum^K_{k=1}\omega_k q_k(\theta)$ be any Gaussian mixture models with $K$ components. The KL-divergence between $p(\theta)$ and $q(\theta)$ is bounded as
\begin{equation*}
    \KL{p(\theta)}{q(\theta)}\leq\sum^K_{k=1}\pi_k\left[\KL{p_k(\theta)}{q_k(\theta)}+\log\frac{\pi_k}{\omega_{k}}\right],
\end{equation*}
By using this bound, the following inequality is derived: 
\begin{align*}
    \KL{p(\theta)}{q(\theta)} \leq& \frac{1}{2}\sum^L_{l=1}\left\{\frac{p_l}{\sigma'^2_l}\|\vM_l-\vM'_l\|^2_F+\frac{1}{\sigma'^2_l}\|\vnu_l-\vnu'_l\|^2+H_{l-1}D_{kl}[p_l||q_l]\right\} \\
    &+\sum^L_{l=1}\frac{(1+p_lH_{l-1})H_{l}}{2}\left\{\frac{\sigma^2_l}{\sigma'^2_l}-\log{\frac{\sigma^2_l}{\sigma'^2_l}}-1\right\},
\end{align*}
where $\vM_l=(\vm_{l1},\vm_{l2},\ldots,\vm_{lH_{l-1}})$. We assume that $\sigma^2_l=\sigma^{\prime 2}_l$ and $p_l=p_{l^{\prime}}$ for any $l \neq l^{\prime}$. Then, the above inequality is reduced to
\begin{equation*}
    \KL{p(\theta)}{q(\theta)} \leq \sum^L_{l=1}\frac{1}{2\sigma^2_l}\left\{p_l\|\vM_l-\vM'_l\|^2_F+\|\vnu_l-\vnu'_l\|^2\right\},
\end{equation*}
which is a weighted sum of the estimated weight matrices and the squared error of the bias.
A detailed derivation is given in Appendix~\ref{sec_kl_bdls}.

\begin{figure}[t!]
    \centering
    \includegraphics[width=12cm]{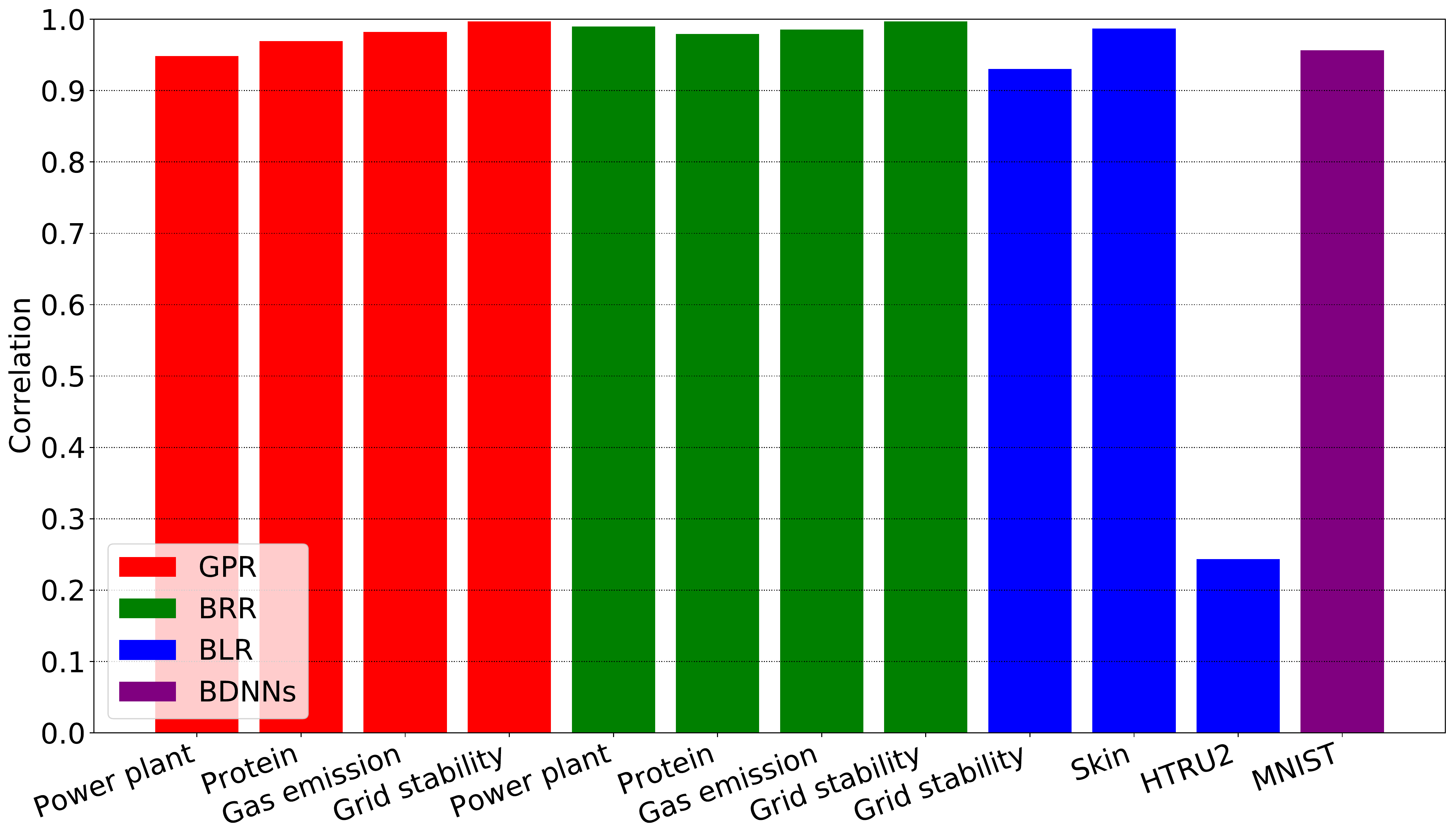} \\
    \caption{Correlation coefficients between expected generalization error and error ratio. The color of the bars indicates a type of learning model (red : GPR, green : BRR, blue : BLR, purple : BDNNs.)}
    \label{fig_correlation_coefficient}
\end{figure}

\subsection{Results}

Figure~\ref{fig_correlation_coefficient} and \ref{fig_scatter_plot} show correlation coefficients and scatter plots between the expected generalization error and the error ratio defined by Eq.~\eqref{eq_error_ratio}, respectively. The correlation coefficients are greater than 0.9 except for BLR applied to the {\tt{HTRU2}} data.  Fig.~\ref{fig_gene_error_BLR}(c) indicates that the expected generalization error increases in the latter half of training in active learning. This is known as the ``less is more'' phenomenon~\citep{Schohn2000}. Since the error ratio bounds the absolute value of the difference between generalization errors, we cannot distinguish whether the difference is positive or negative. Hence, the error ratio has a high correlation with the generalization error, except in the case where the ``less is more'' phenomenon occurs. Even in such cases, the error ratio correlates with the generalization error until the generalization error starts to increase as shown in Fig.~\ref{fig_scatter_plot}(k).

The expected generalization errors evaluated by using the test data and the stopping timings determined by using the proposed criterion with each threshold for each model are shown in Fig.~\ref{fig_gene_error_GPR}, Fig.~\ref{fig_gene_error_BRR}, Fig.~\ref{fig_gene_error_BLR} and Fig.~\ref{fig_gene_error_BDNN}. From  Fig.~\ref{fig_gene_error_GPR} and Fig.~\ref{fig_gene_error_BRR}, the proposed method tends to terminate active learning at about the same timing for any dataset for GPR and BRR when using the same threshold. As shown in Fig.~\ref{fig_gene_error_BLR}, the ``less is more'' phenomenon is occurred in {\tt{Power plant}} and {\tt{Gas emission}} for BLR. Thus, the optimal stopping timing is the timing minimizing the generalization error or earlier. As shown in Figs.~\ref{fig_scatter_plot}(i)--(k), the threshold of the proposed criterion corresponding to the timing minimizing the generalization error depends on the dataset. Therefore, the proposed criterion with a threshold cannot always stop active learning at the optimal stopping timing, but it can stop it at a reasonably good timing as shown in Fig.~\ref{fig_gene_error_BLR}. From Fig.~\ref{fig_gene_error_BDNN}(a), while the proposed method stops active learning when the threshold is $0.2$, the proposed method cannot stop active learning when the thresholds are $0.15$ and $0.1$. This is due to the fact that the proposed criterion converges to around $0.2$ as shown in Fig.~\ref{fig_gene_error_BDNN}(b). Therefore, we cannot guarantee that the proposed criterion stops active learning when the threshold of the proposed criterion approaches zero for BDNNs, unlike for the other models. However, as long as the threshold is set to be large, we can stop active learning at an appropriate timing for BDNNs.

It is shown that the error ratio has a high correlation with the generalizaton error for various datasets and AL models, and does not depend on the dataset.

\begin{figure}[t!]
    \centering
    \begin{tabular}{cccc}
    \includegraphics[width=3.4cm]{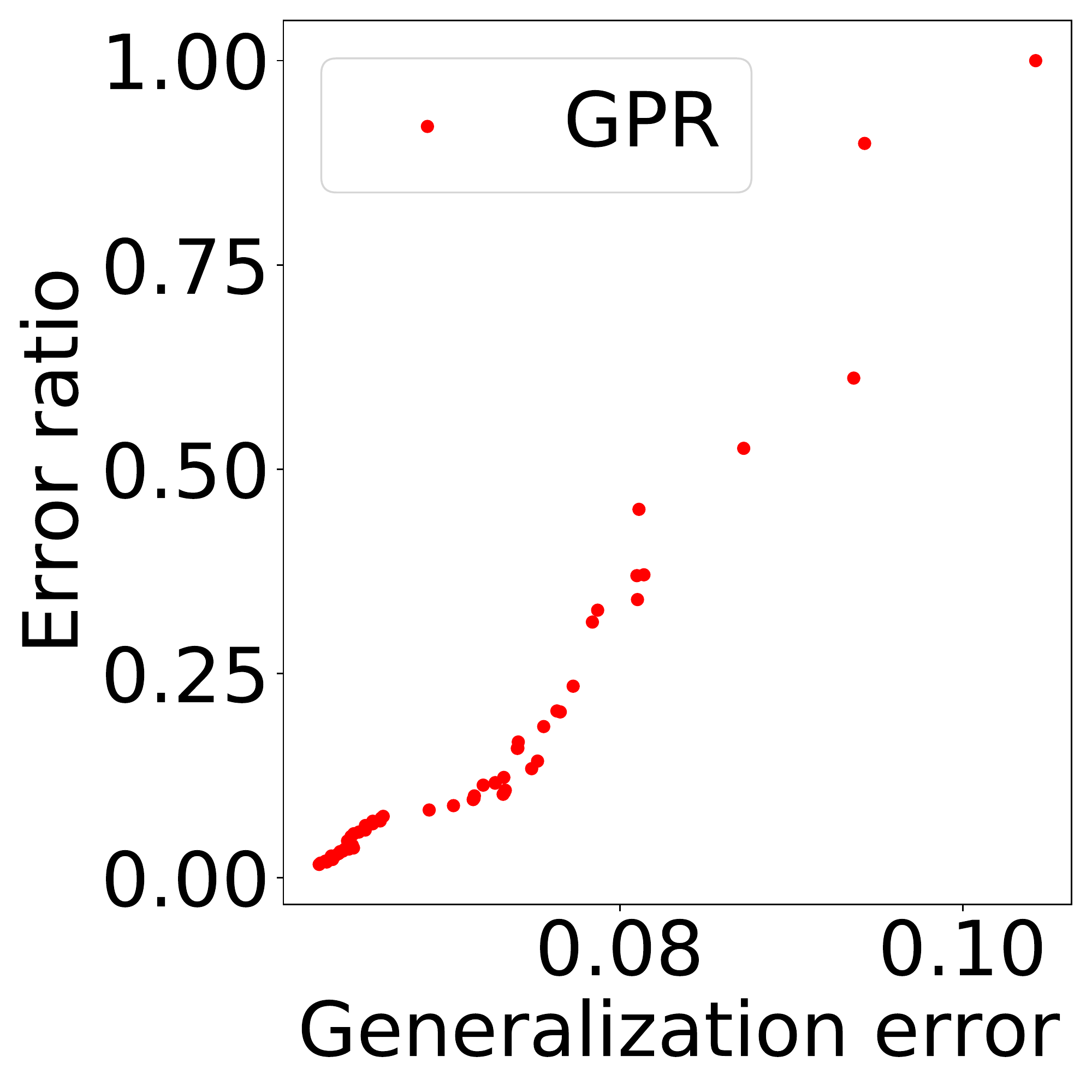}&
    \includegraphics[width=3.4cm]{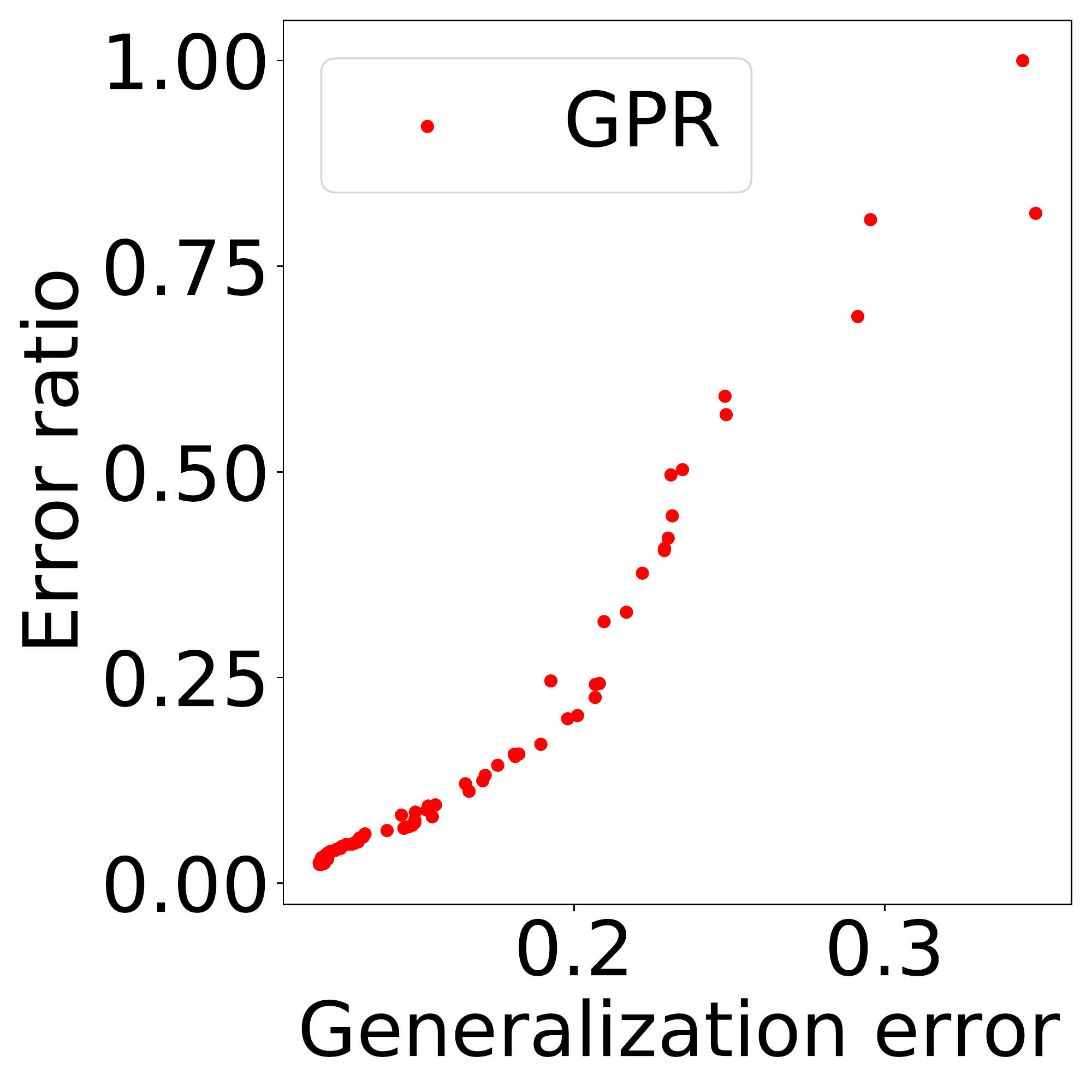}&
    \includegraphics[width=3.4cm]{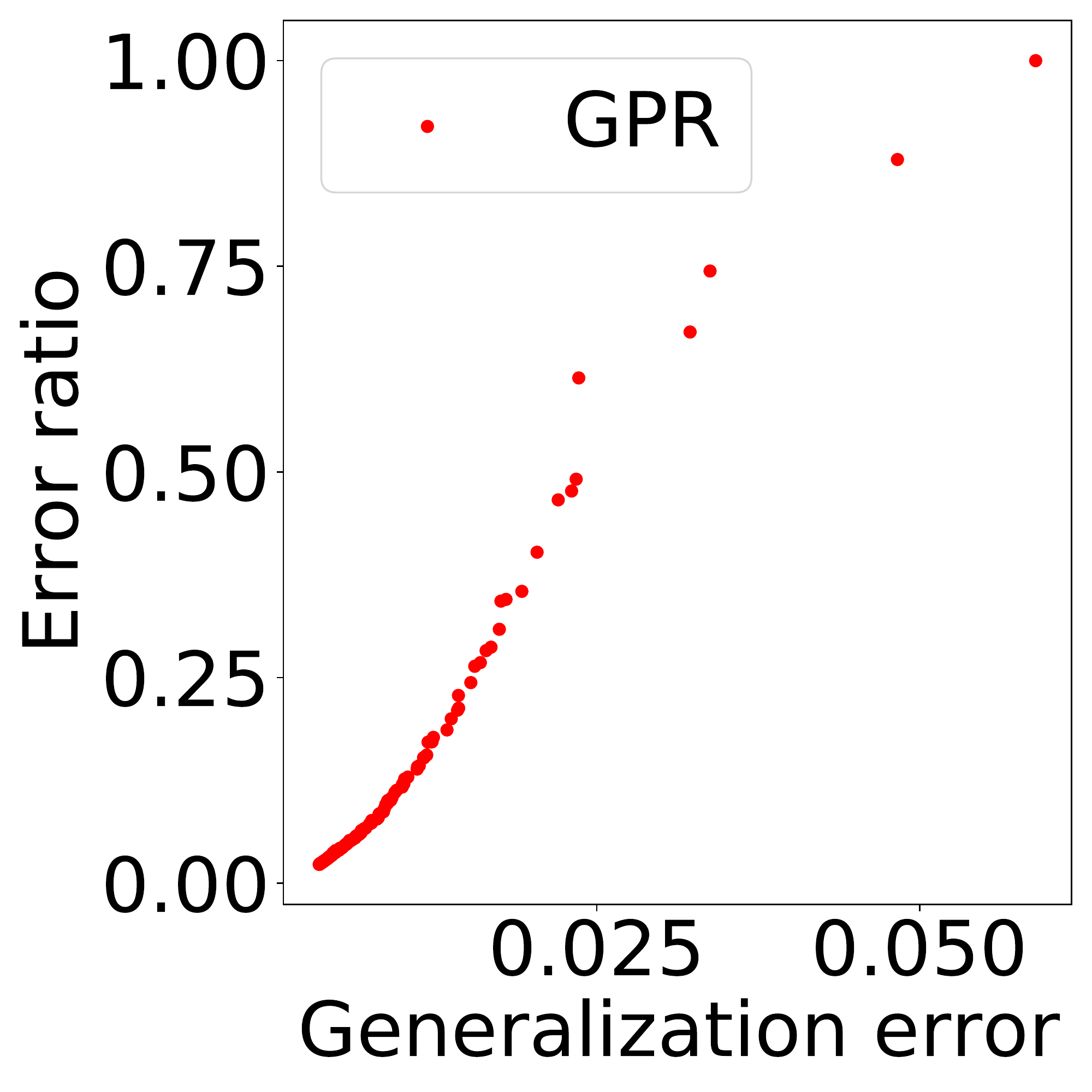}&
    \includegraphics[width=3.4cm]{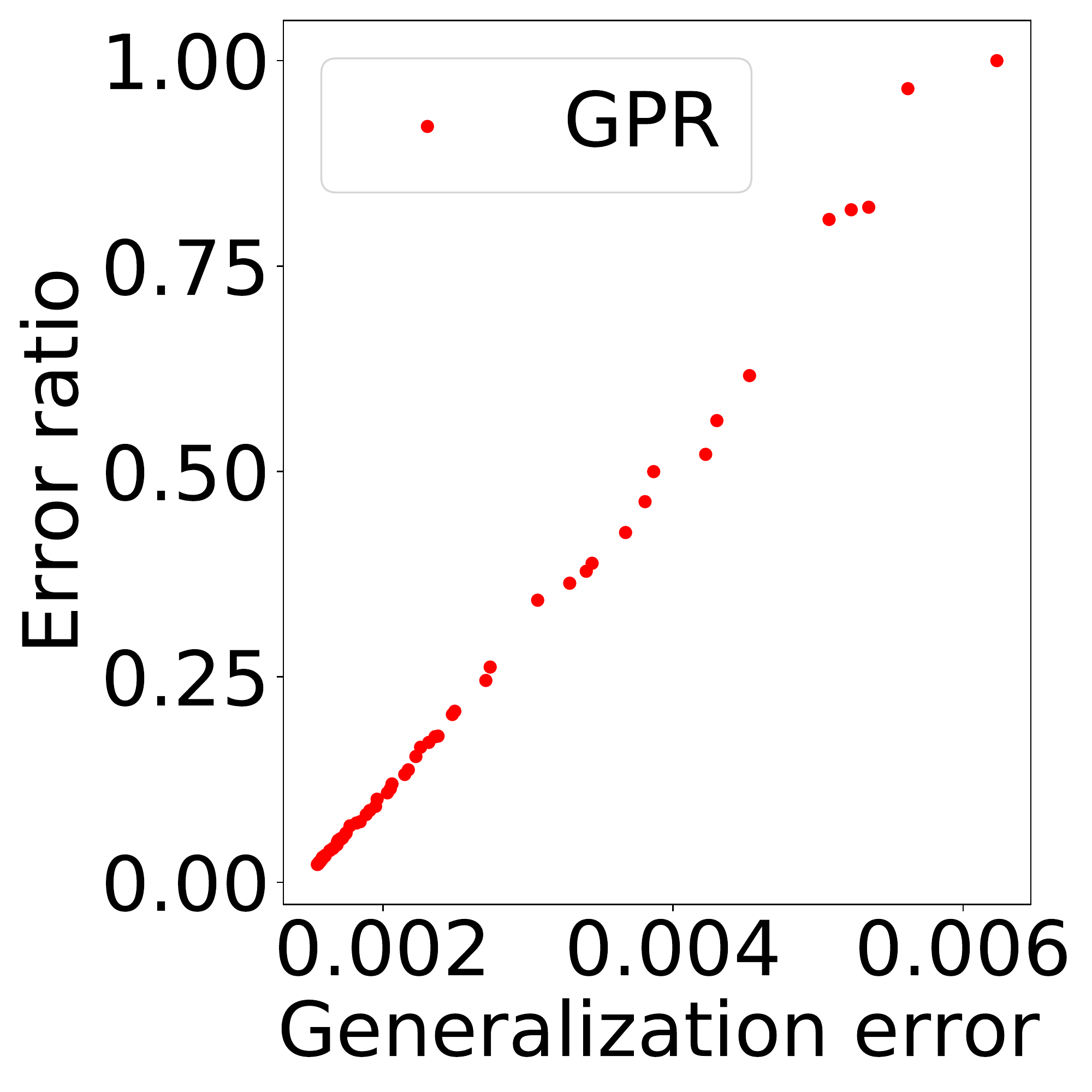}\\
    (a) {\tt{Power plant(r)}}& (b) {\tt{Protein(r)}}& (c) {\tt{Gas emission(r)}}& (d) {\tt{Grid stability(r)}} \\
    \includegraphics[width=3.4cm]{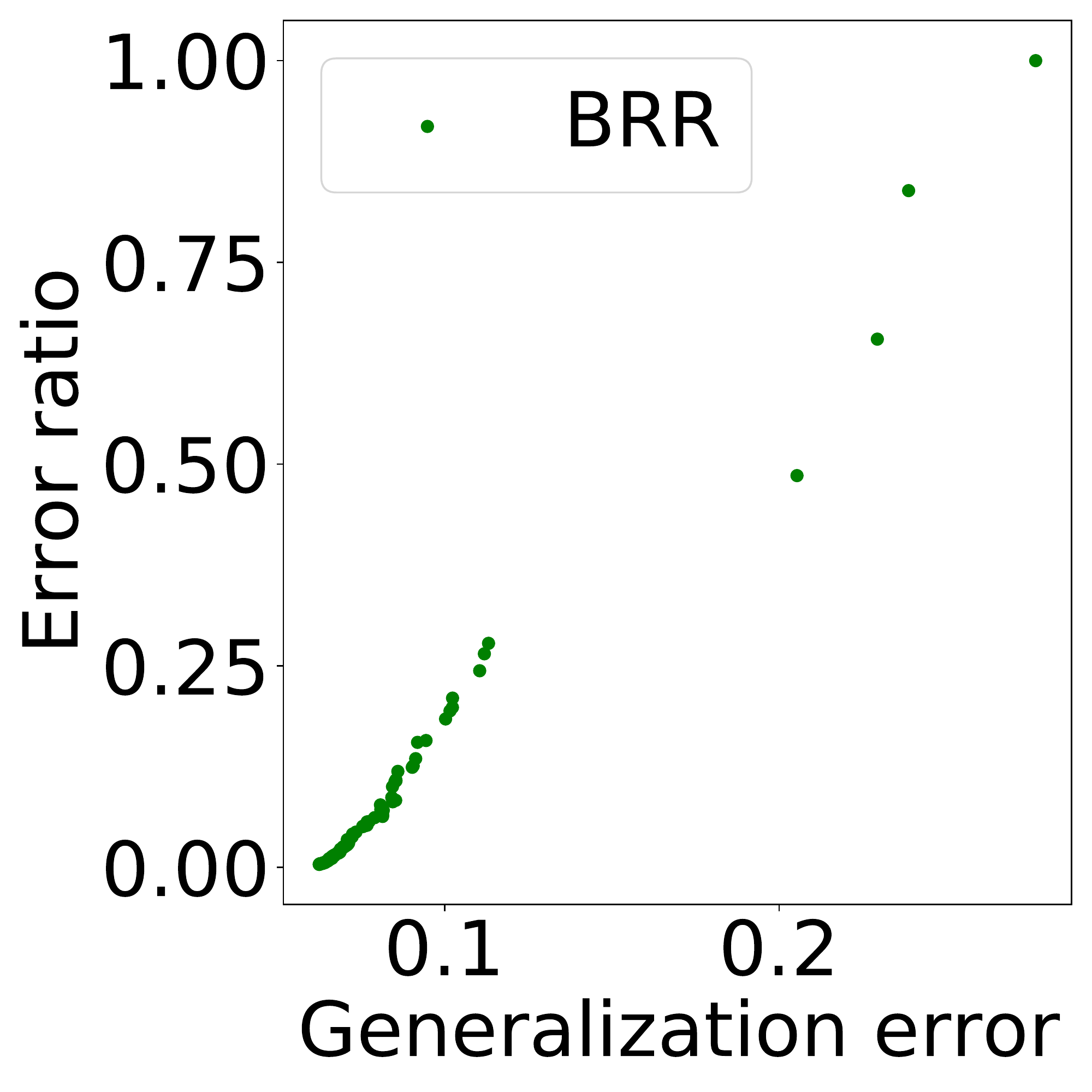}&
    \includegraphics[width=3.4cm]{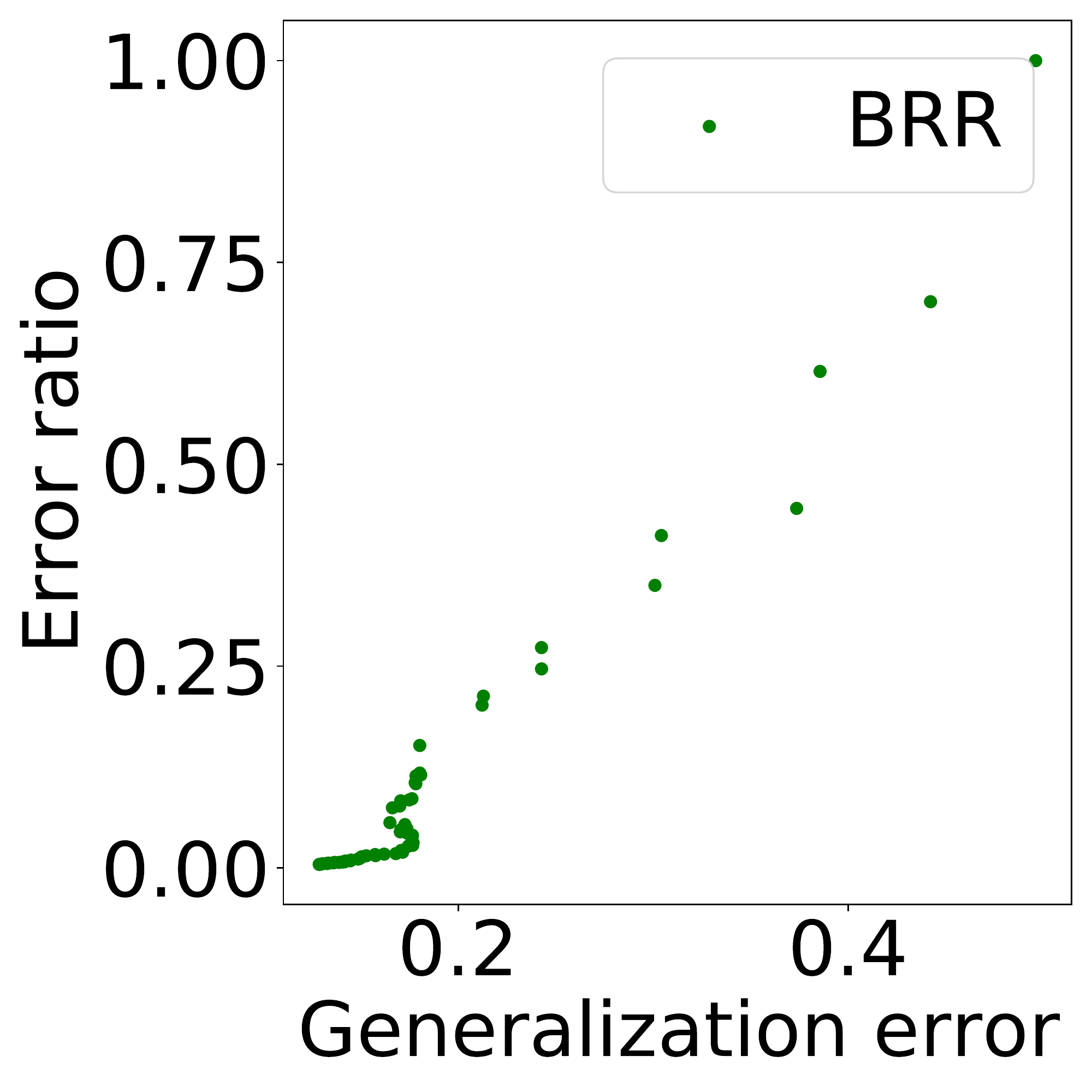}&
    \includegraphics[width=3.4cm]{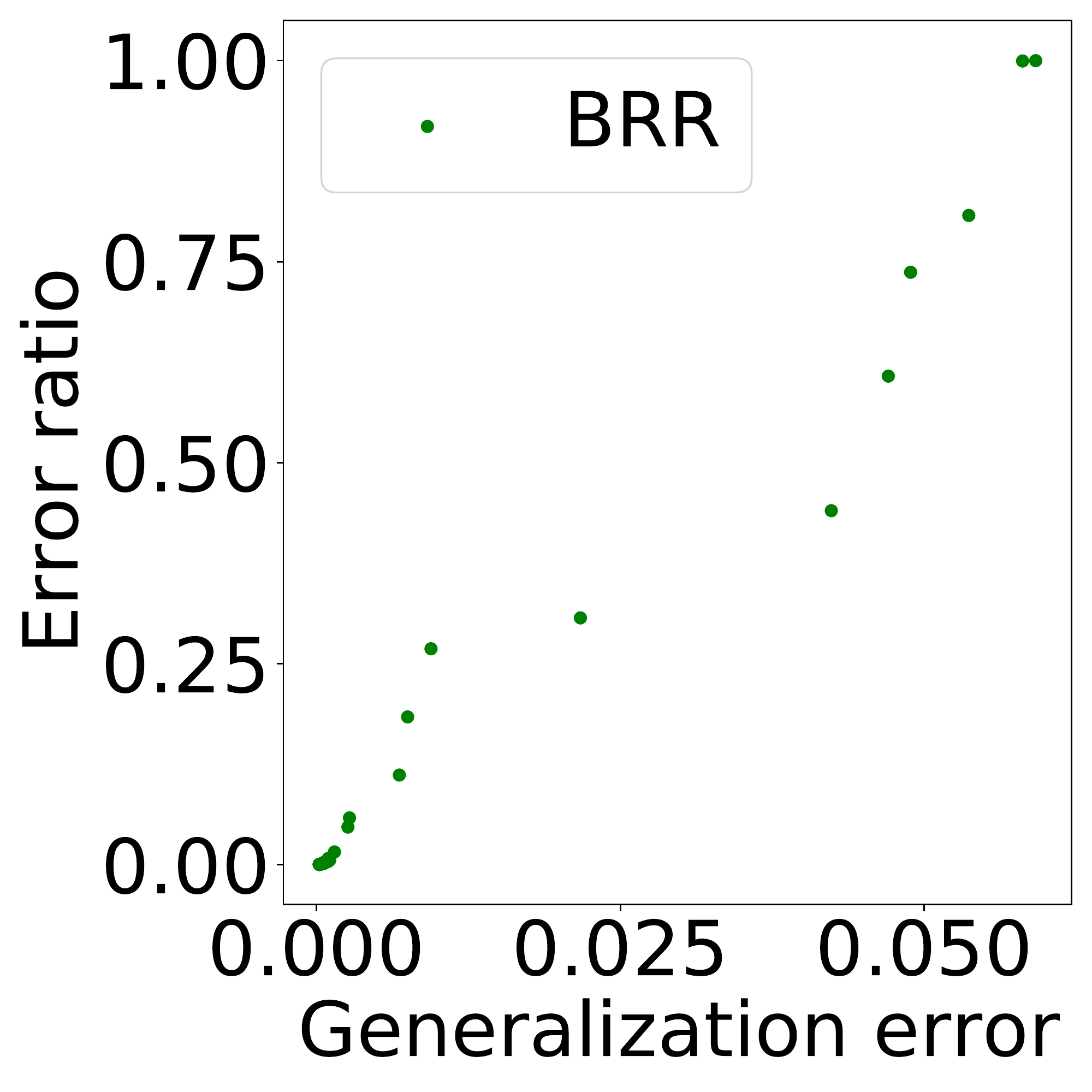}&
    \includegraphics[width=3.4cm]{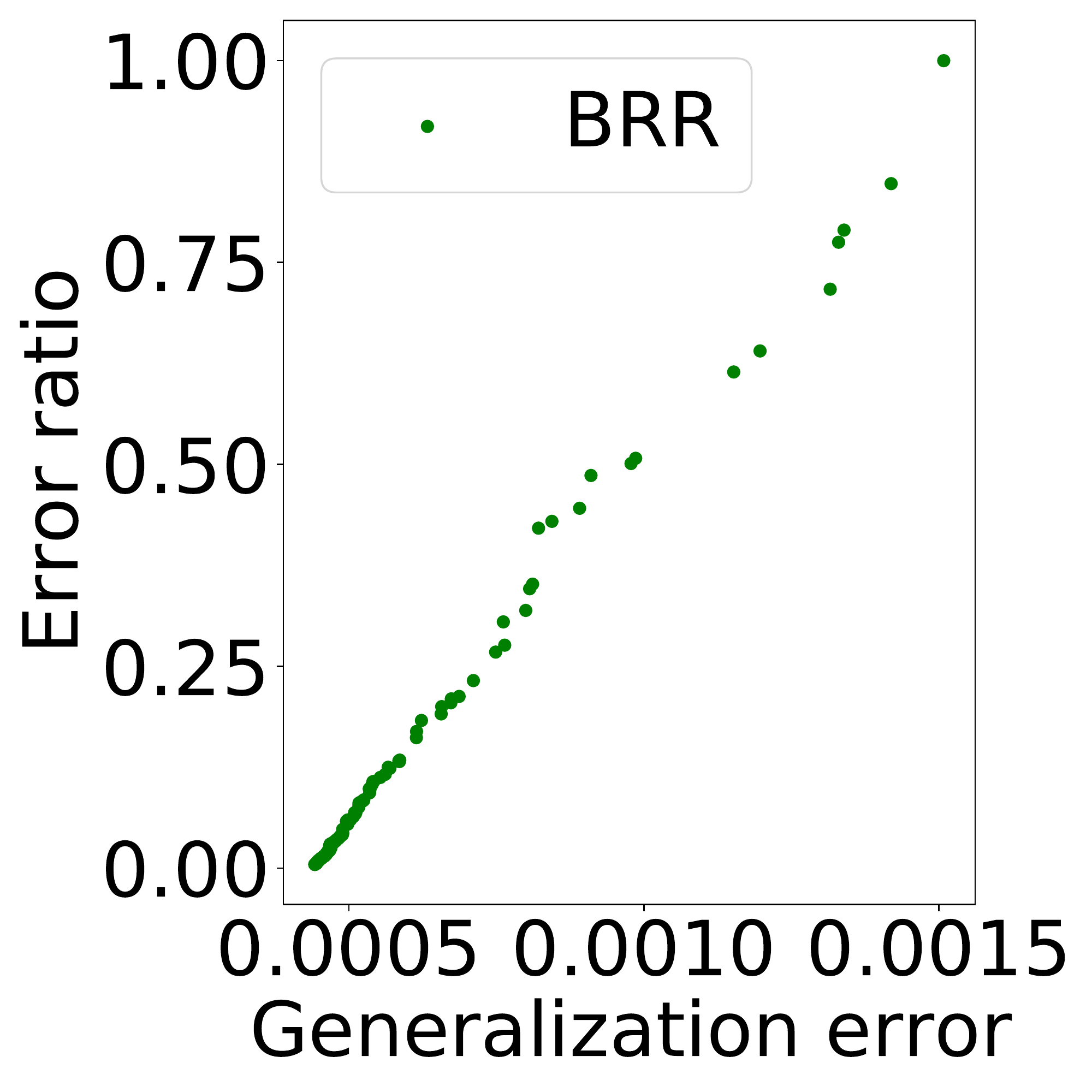} \\
    (e) {\tt{Power plant(r)}}& (f) {\tt{Protein(r)}}&
    (g) {\tt{Gas emission(r)}}& (h) {\tt{Grid stability(r)}} \\
    \includegraphics[width=3.4cm]{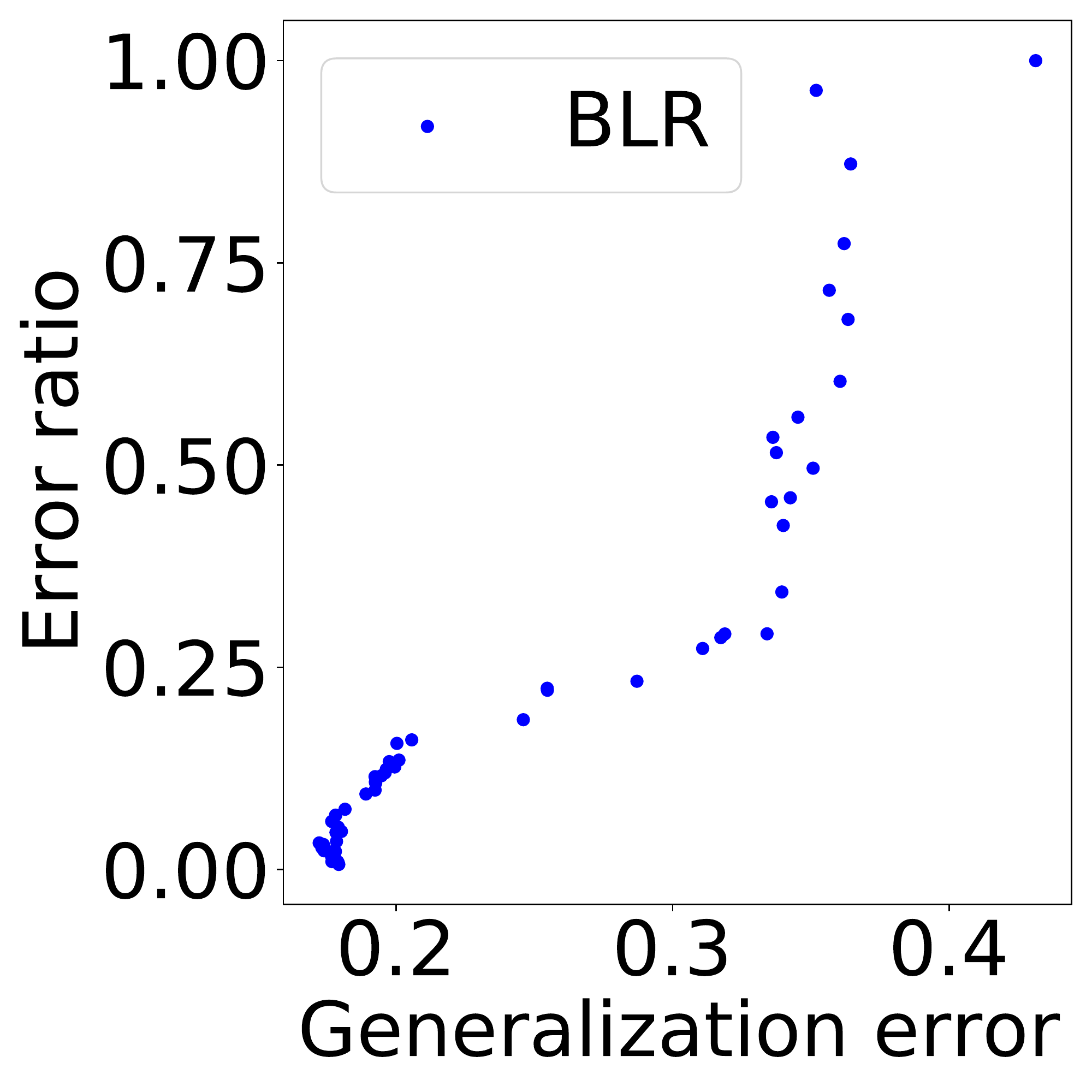}&
    \includegraphics[width=3.4cm]{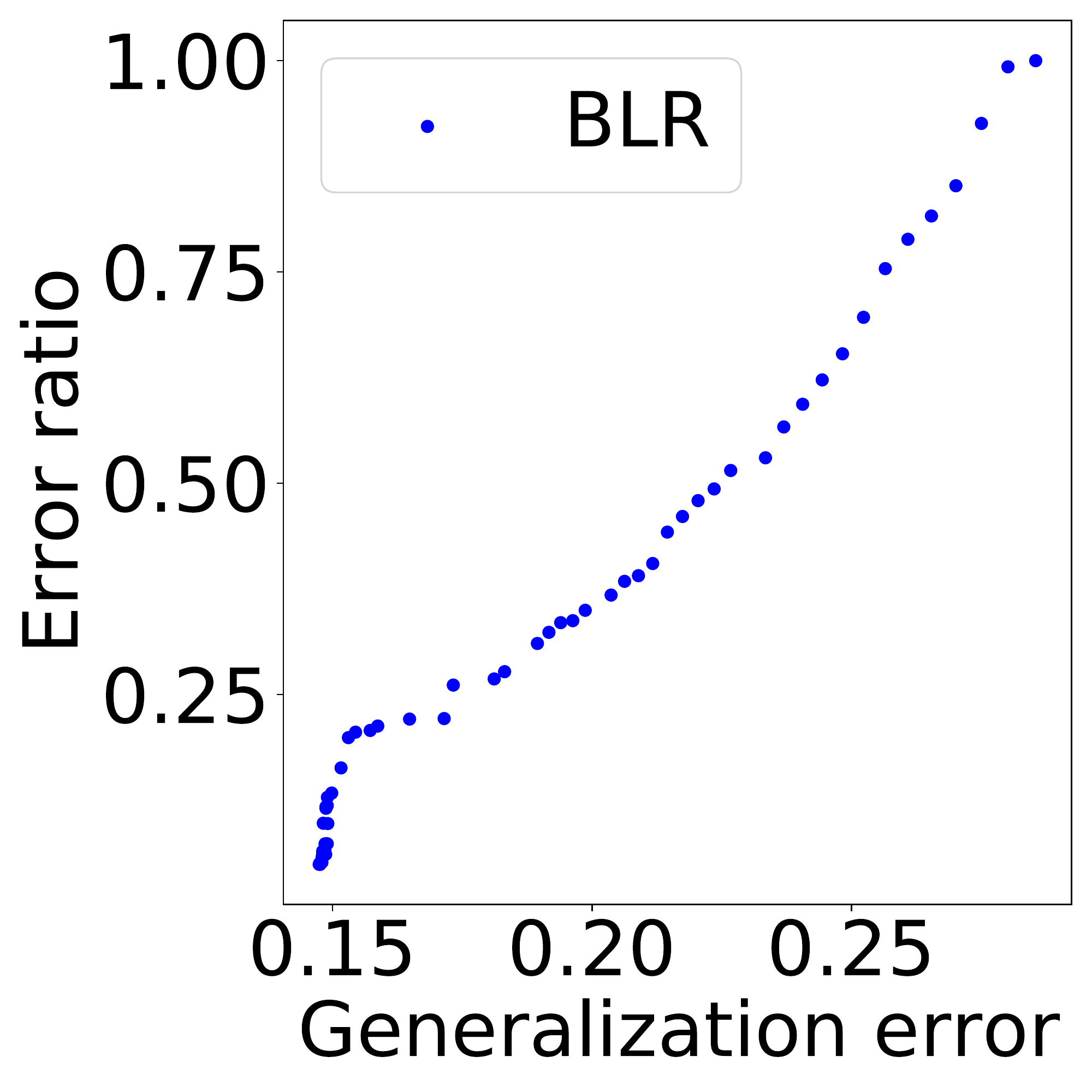}&
    \includegraphics[width=3.4cm]{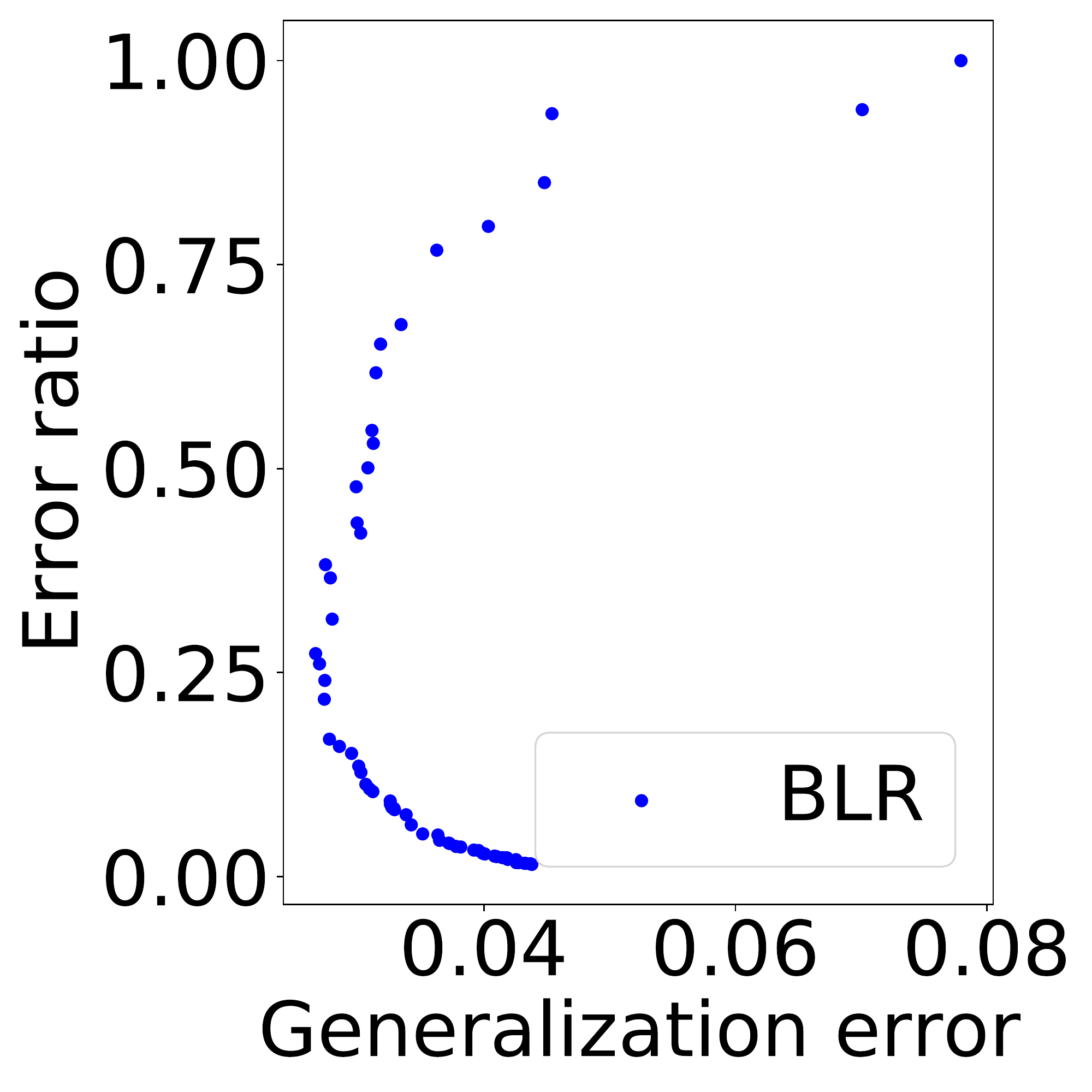}&
    \includegraphics[width=3.4cm]{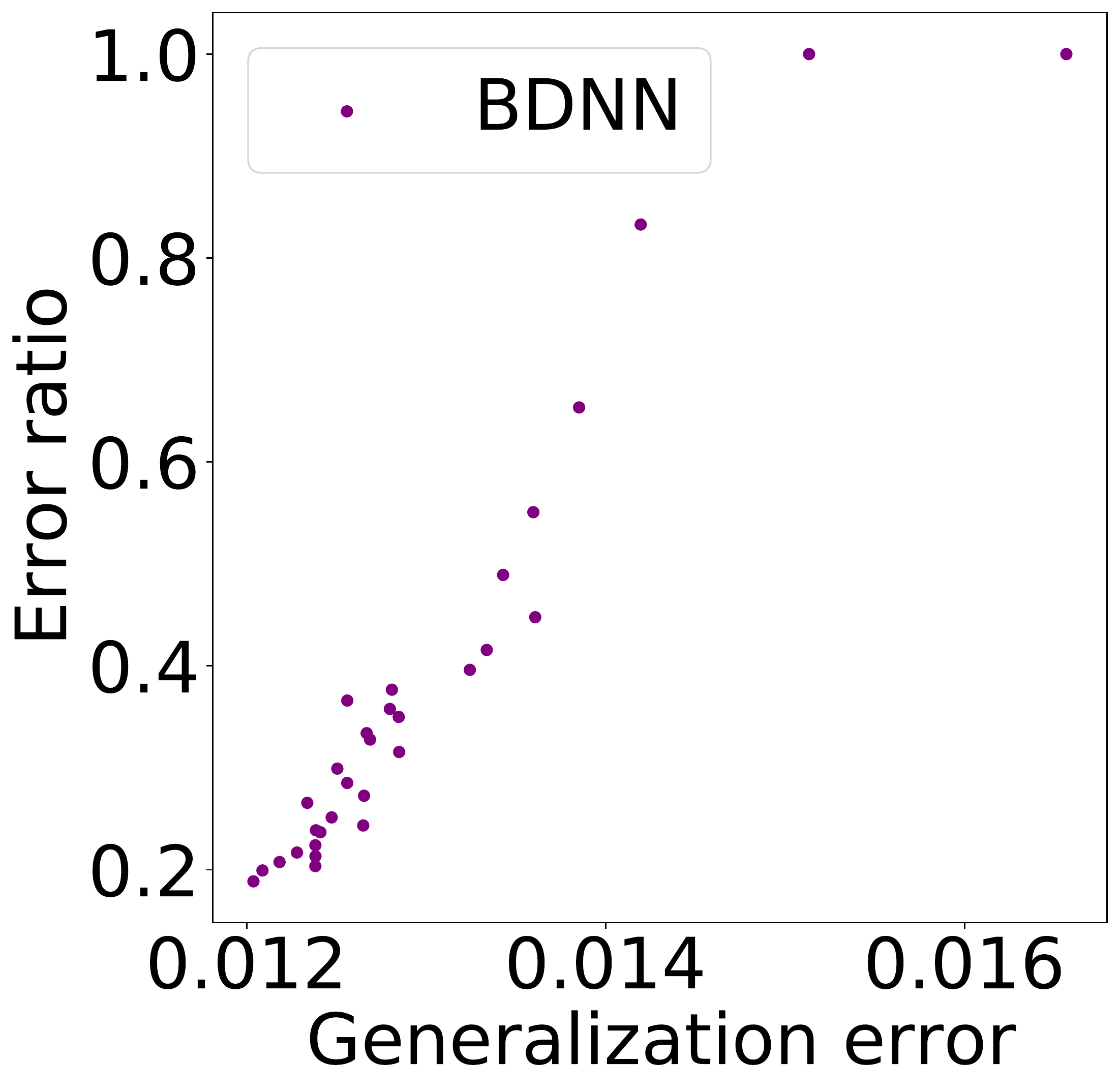} \\
    (i) {\tt{Grid stability(c)}}& (j) {\tt{Skin(c)}}& (k) {\tt{HTRU2(c)}}& (l) {\tt{MNIST(c)}} \\
    \end{tabular}
    \caption{Scatter plots of expected generalization error and error ratio. The names of the datasets are followed by (r) and (c) to show whether the dataset is for regression or classification. 
    The color of the plots indicates the type of learning model (red : GPR, green : BRR, blue : BLR, purple : BDNNs).}
    \label{fig_scatter_plot}
\end{figure}

\begin{figure}[t!]
    \centering
    \begin{tabular}{cc}
    \includegraphics[width=6.6cm]{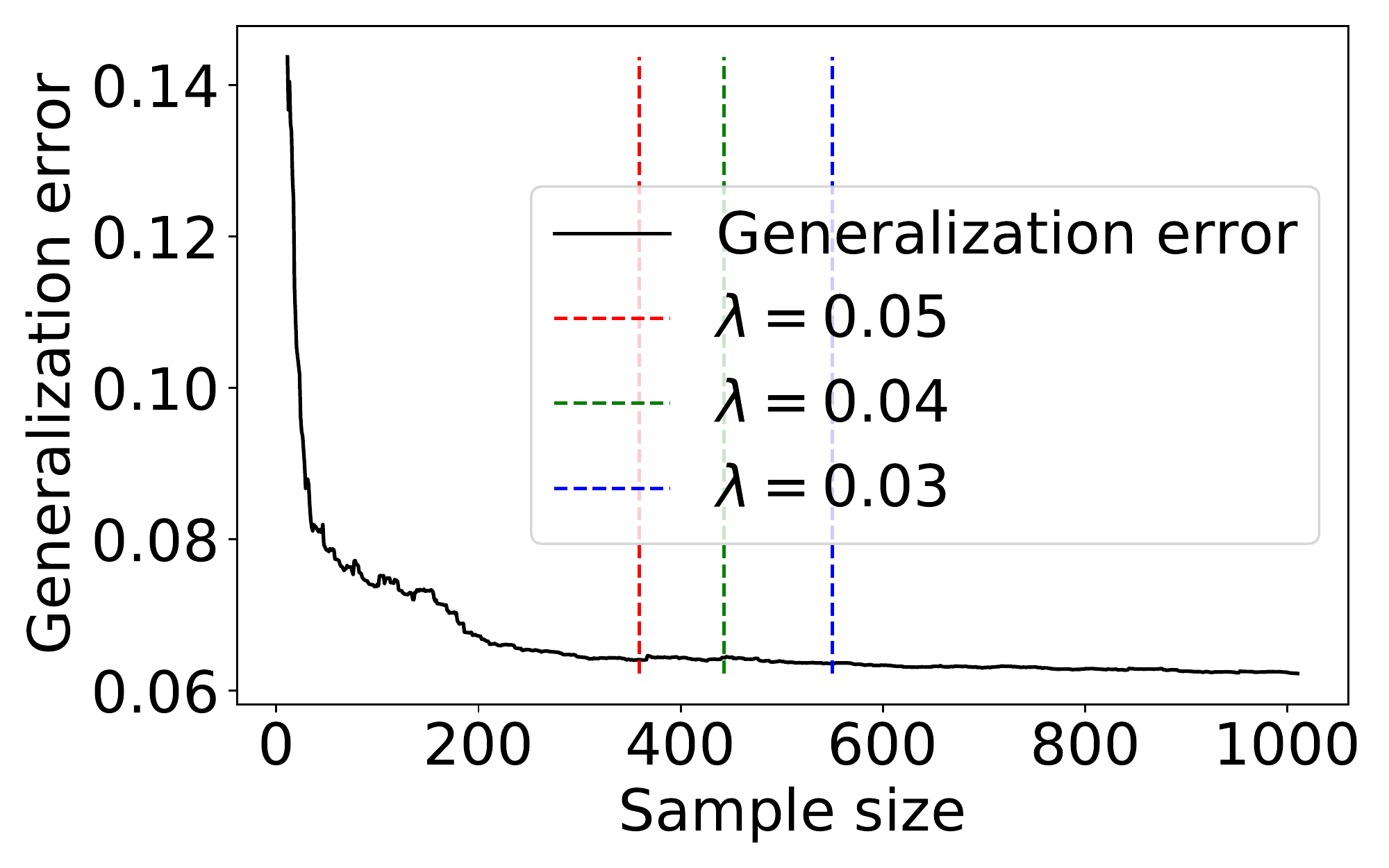}&
    \includegraphics[width=6.6cm]{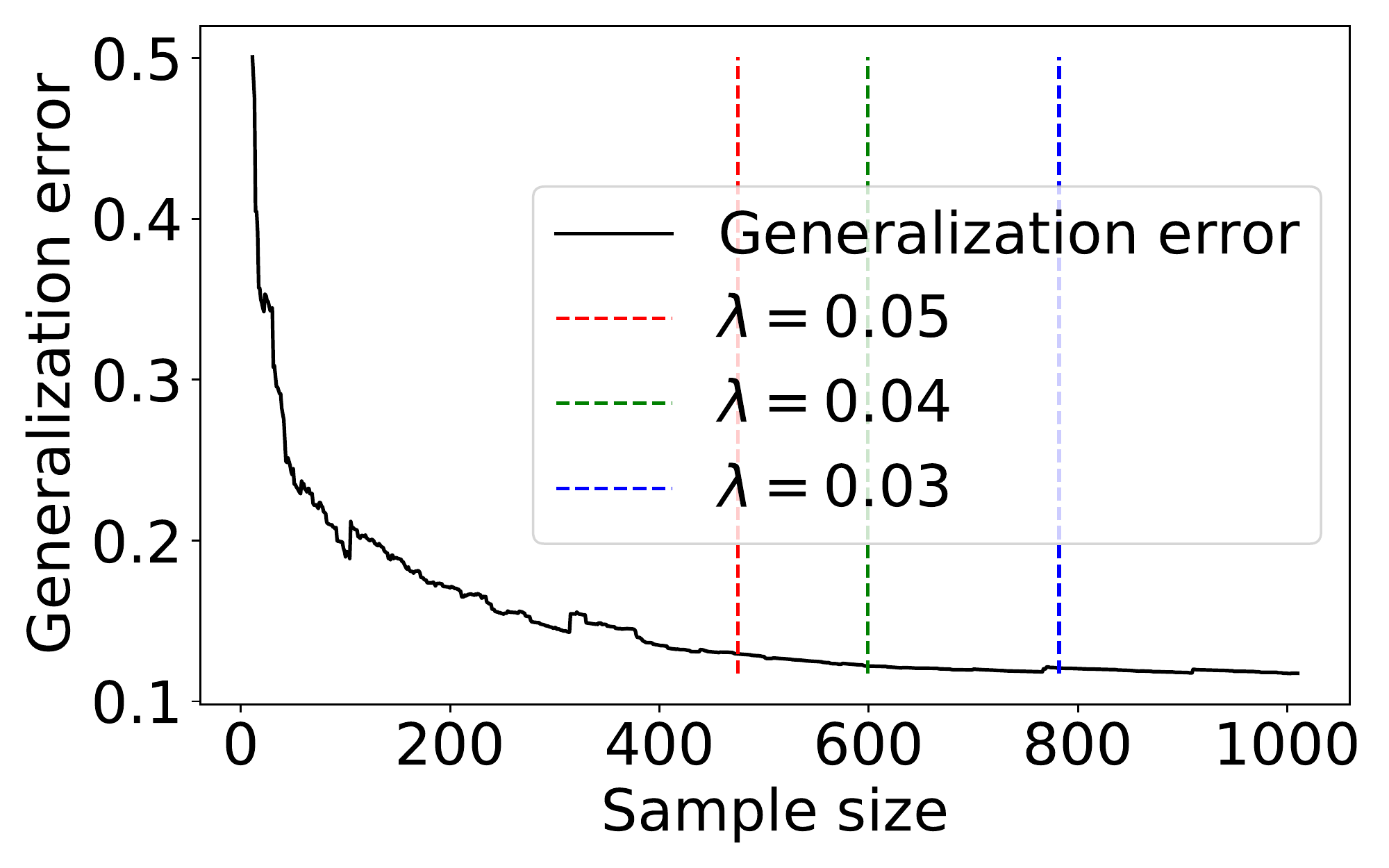} \\
    (a) {\tt{Power plant}}& (b) {\tt{Protein}} \\
    \includegraphics[width=6.6cm]{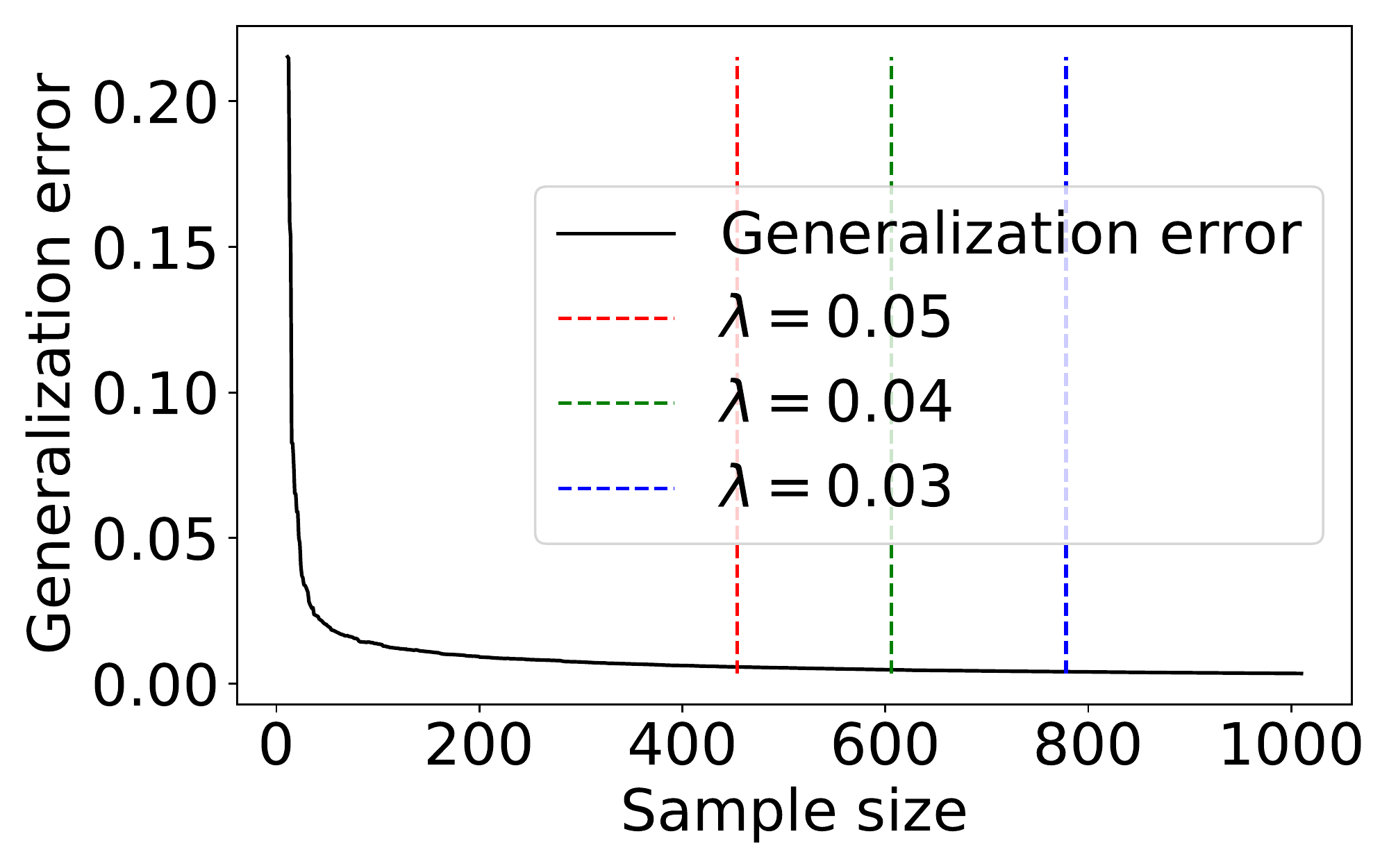}&
    \includegraphics[width=6.6cm]{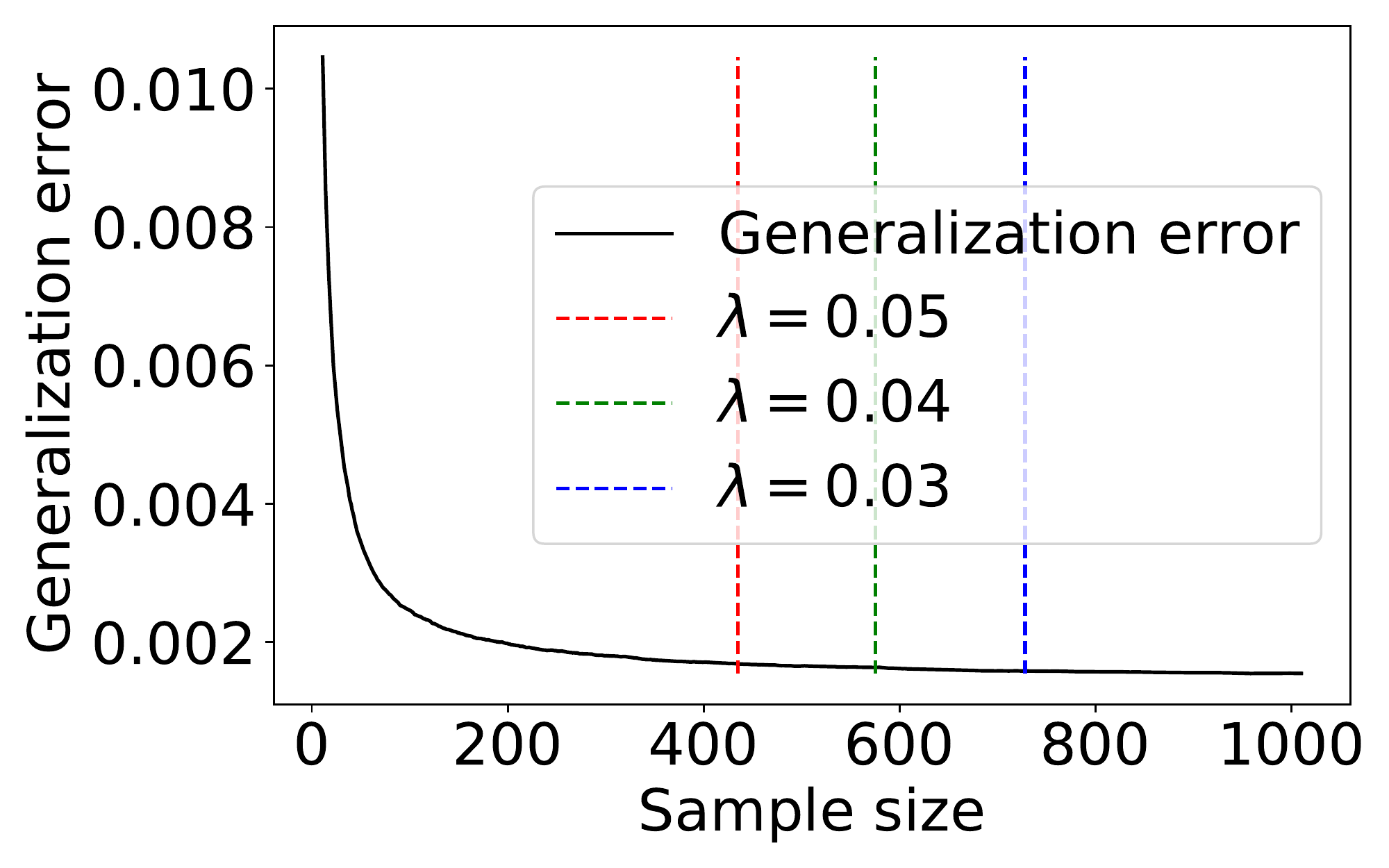} \\
    (c) {\tt{Gas emission}}& (d) {\tt{Grid stability}} \\
    \end{tabular}
    \caption{Expected generalization error and stopping timing for GPR.}
    \label{fig_gene_error_GPR}
\end{figure}

\begin{figure}[t!]
    \centering
    \begin{tabular}{cc}
    \includegraphics[width=6.6cm]{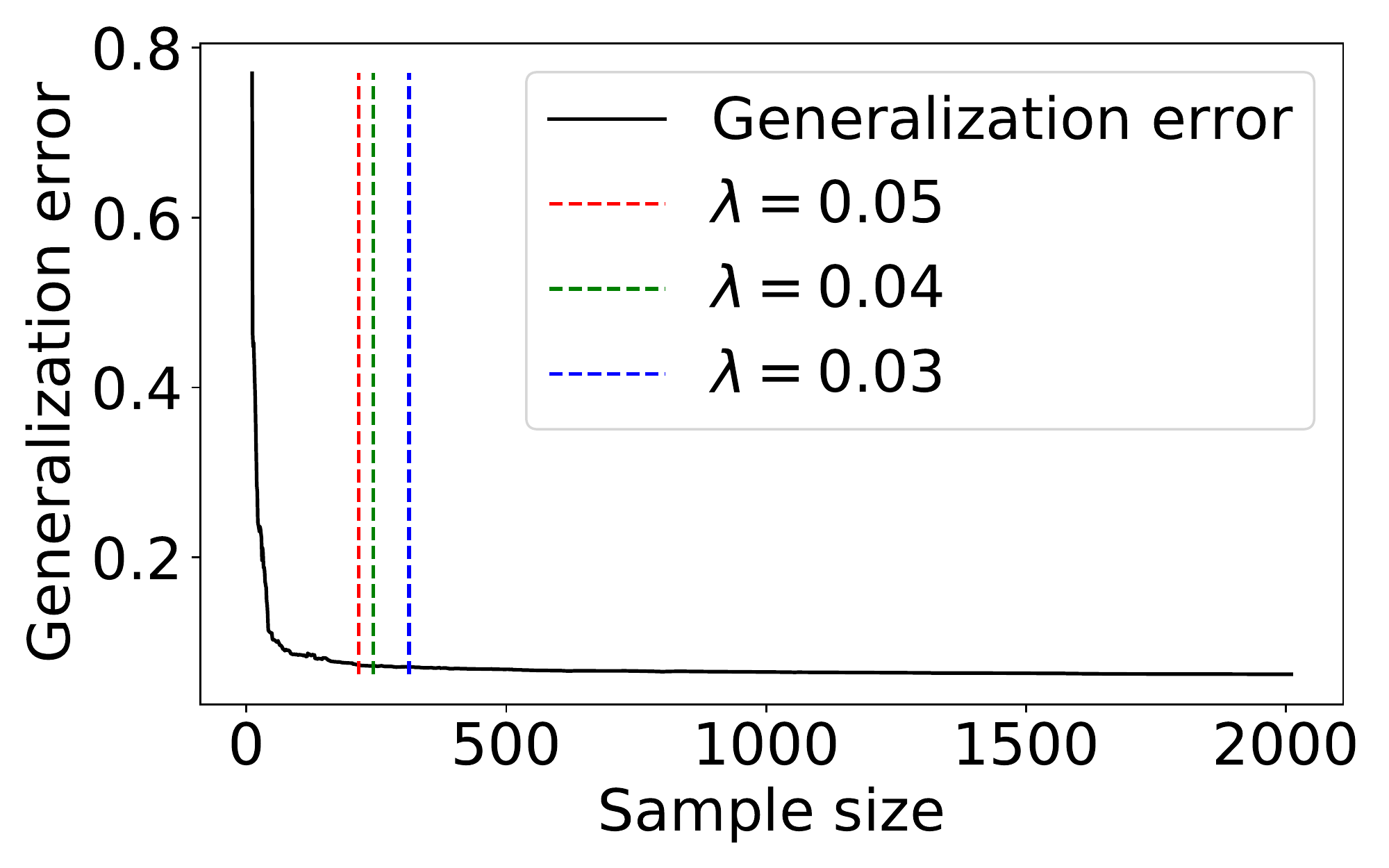}&
    \includegraphics[width=6.6cm]{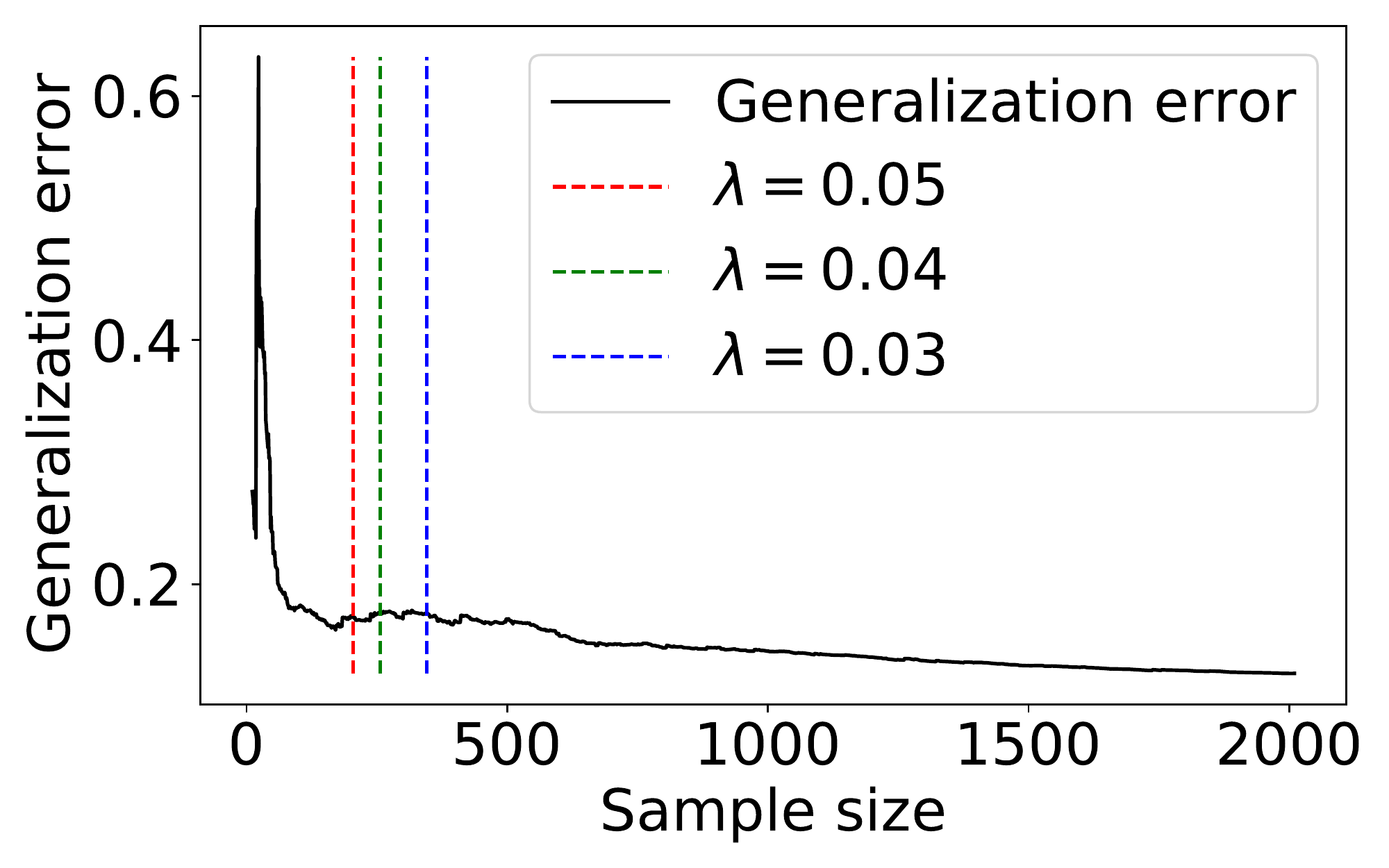} \\
    (a) {\tt{Power plant}}& (b) {\tt{Protein}} \\
    \includegraphics[width=6.6cm]{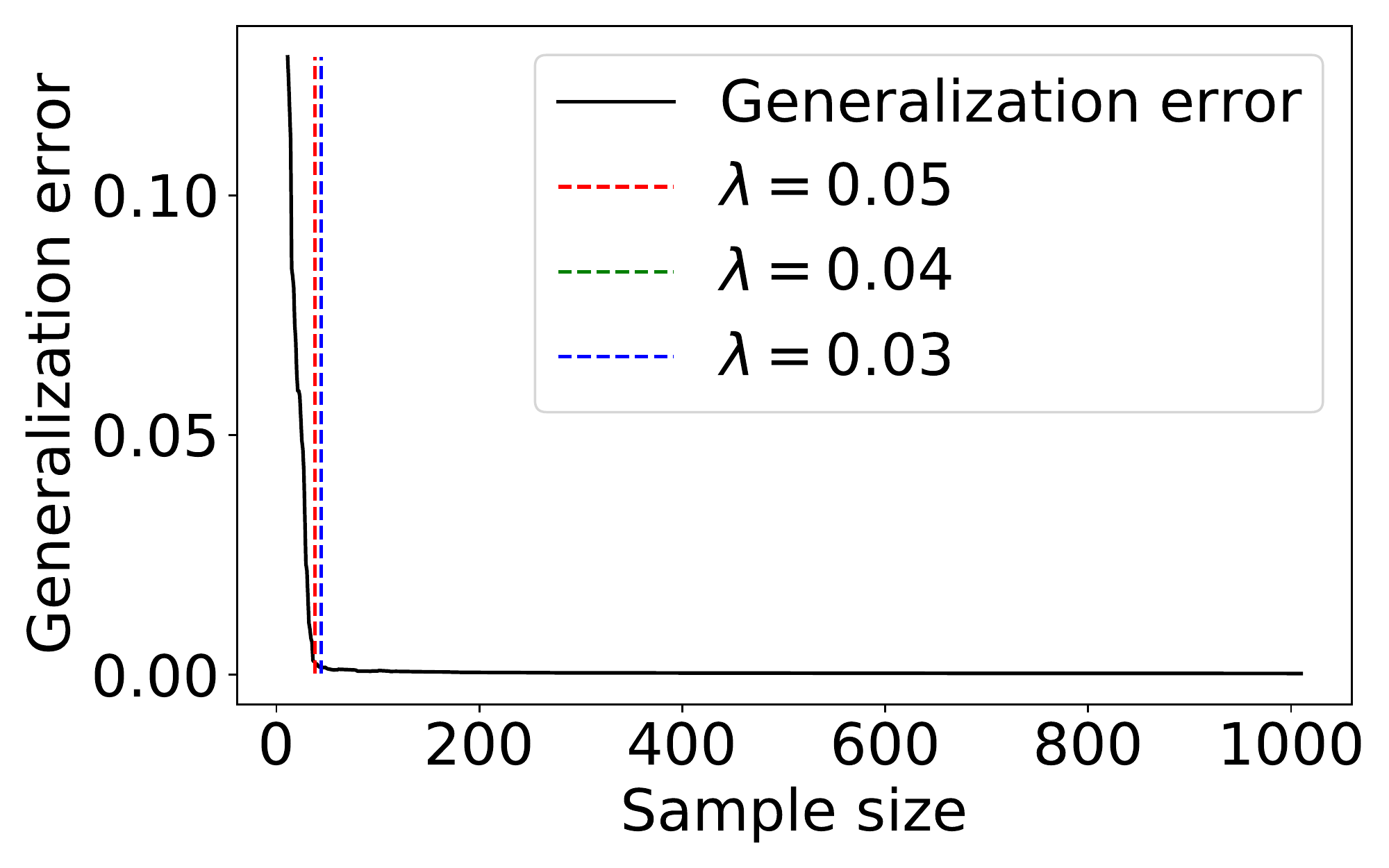}&
    \includegraphics[width=6.6cm]{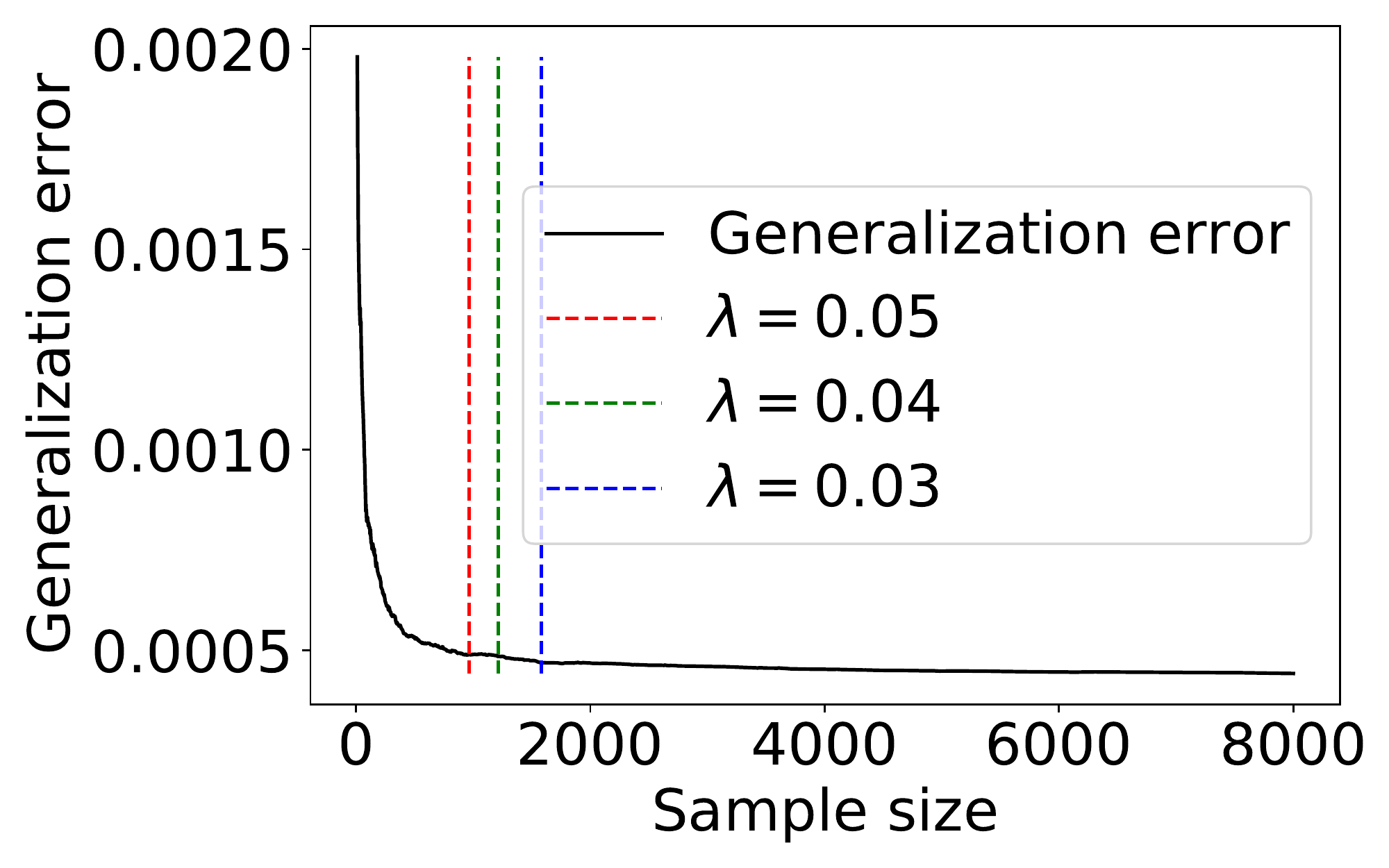} \\
    (c) {\tt{Gas emission}}& (d) {\tt{Grid stability}} \\
    \end{tabular}
    \caption{Generalization error and stopping timing for BRR.}
    \label{fig_gene_error_BRR}
\end{figure}

\begin{figure}[t!]
    \centering
    \begin{tabular}{cc}
    \includegraphics[width=6.6cm]{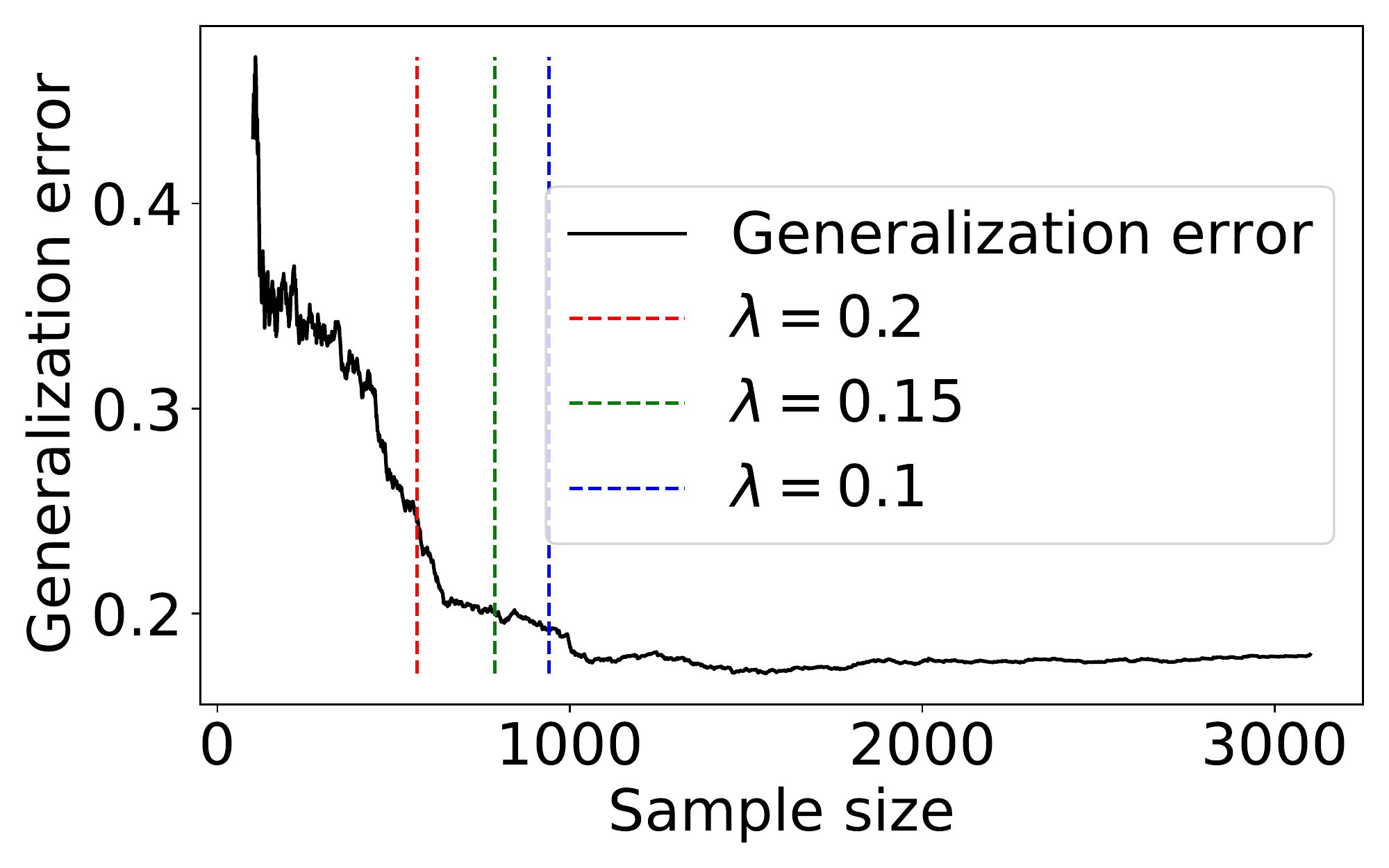}&
    \includegraphics[width=6.6cm]{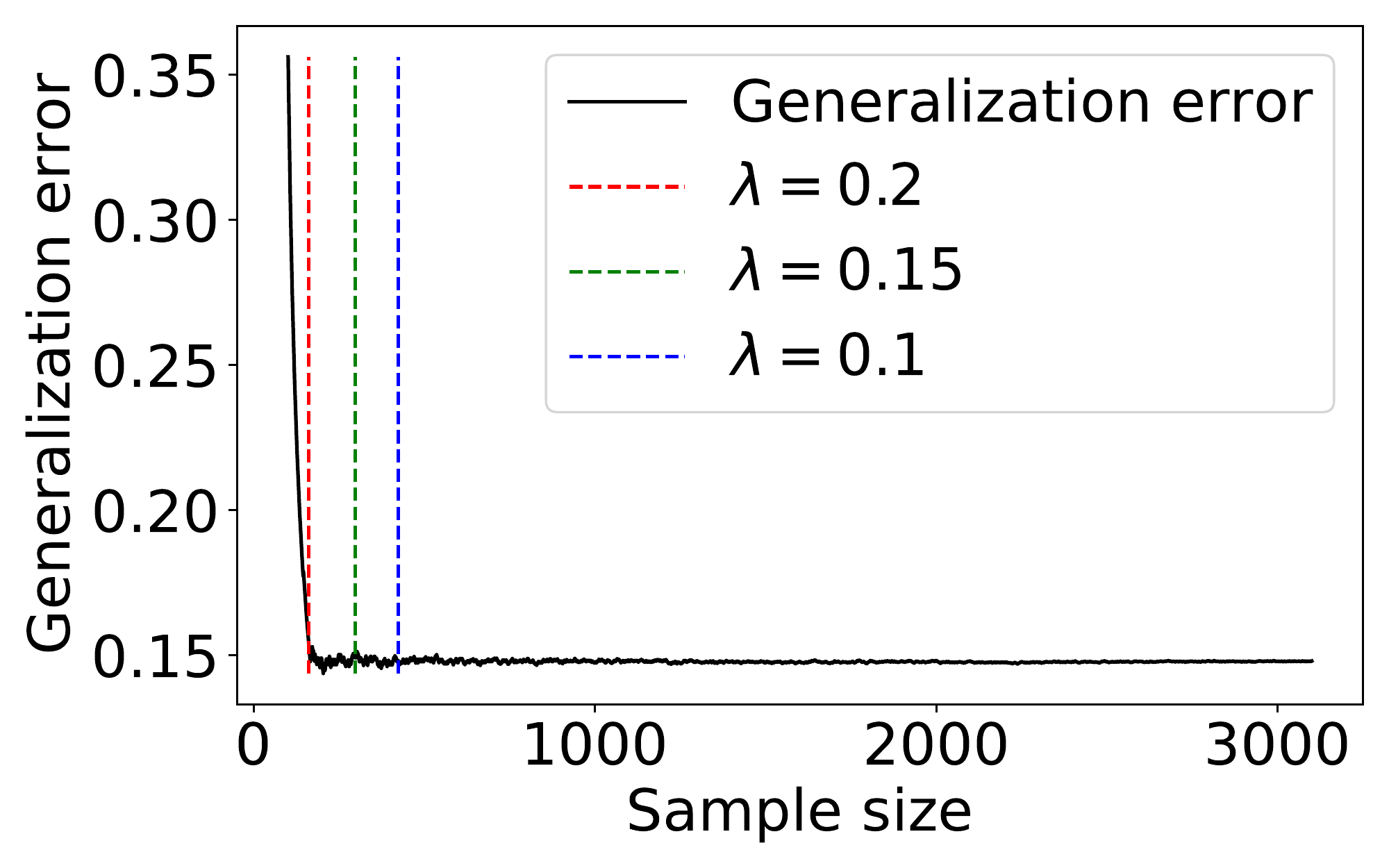} \\
    (a) {\tt{Power plant}}& (b) {\tt{Protein}} \\    \includegraphics[width=6.6cm]{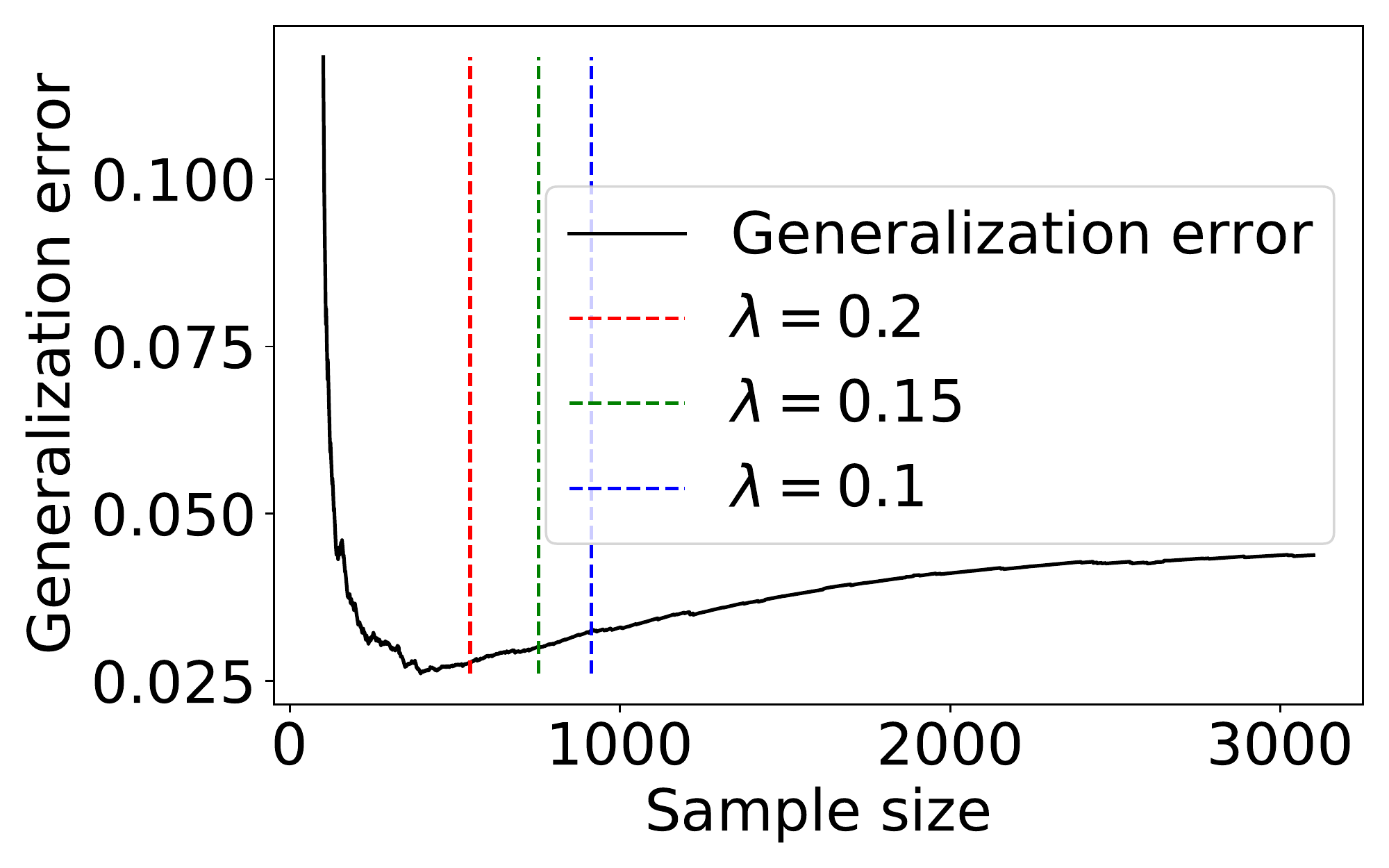}& \\
    (c) {\tt{Grid stability}}& \\
    \end{tabular}
    \caption{Expected generalization error and stopping timing for BLR.}
    \label{fig_gene_error_BLR}
\end{figure}

\begin{figure}[t!]
    \centering
    \begin{tabular}{cc}
    \includegraphics[width=6.6cm]{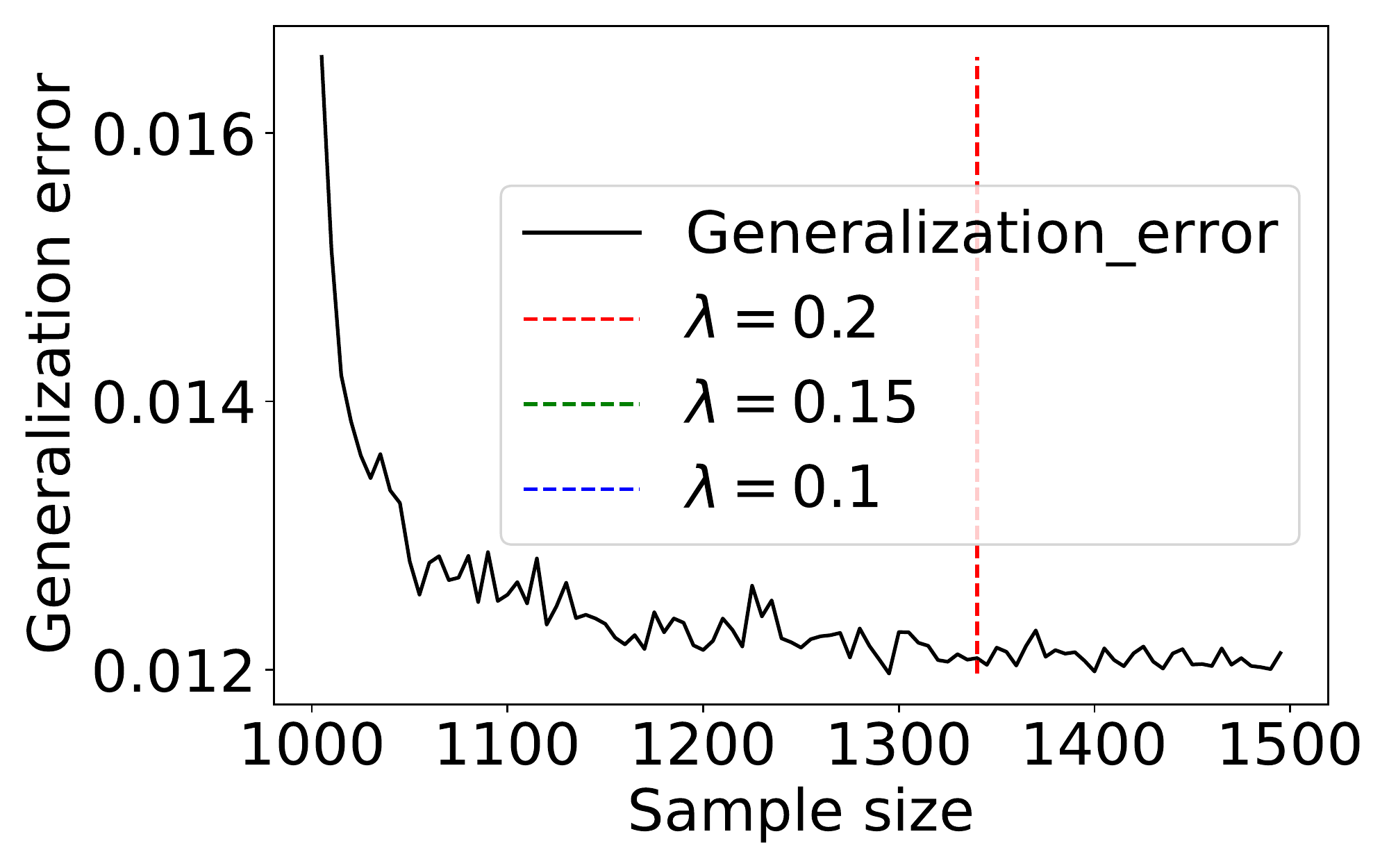}&
    \includegraphics[width=6.6cm]{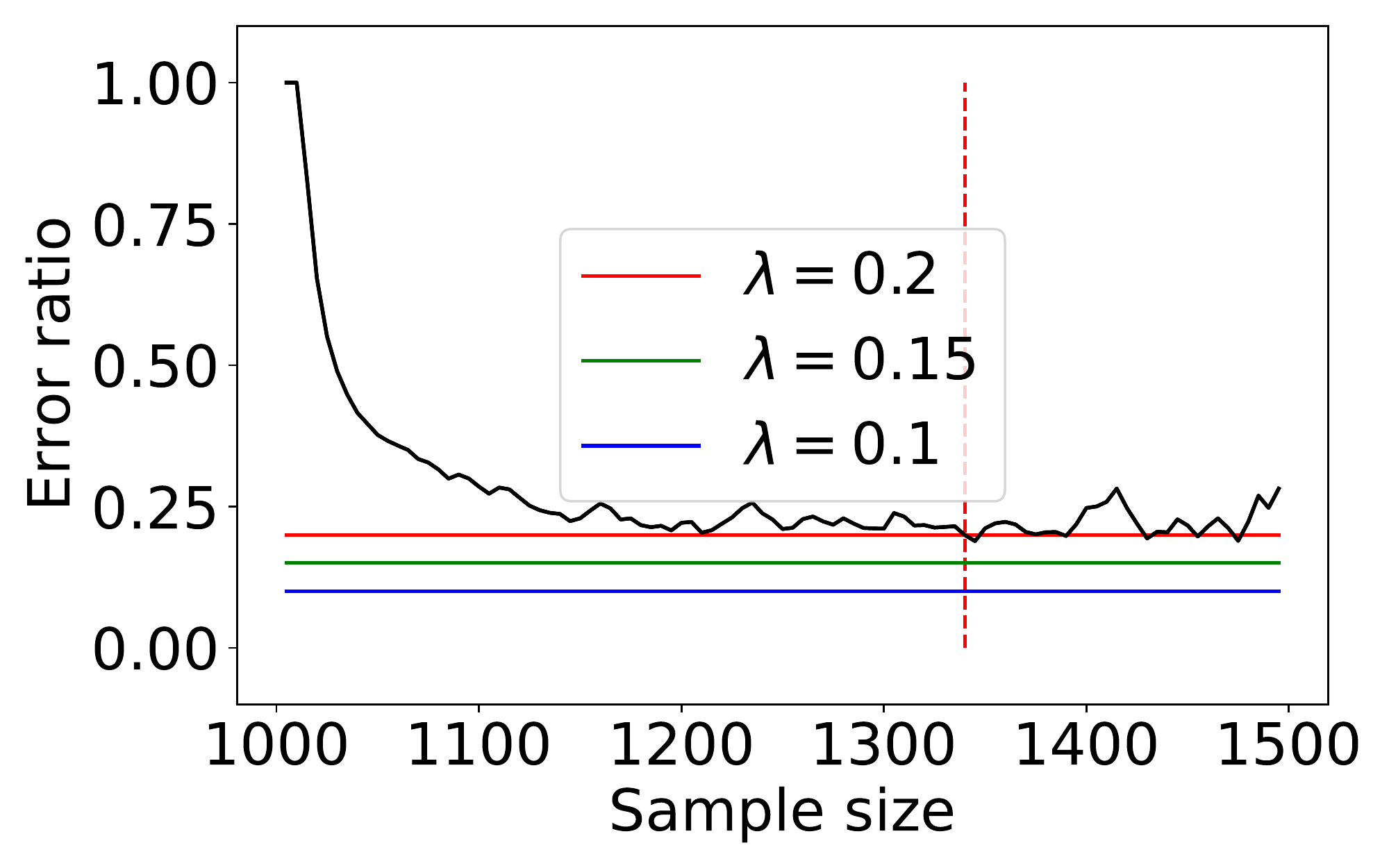} \\
    (a) Expected generalization error& (b) Error ratio
    \end{tabular}
    \caption{Expected generalization error, error ratio and stopping timing for BDNNs.}
    \label{fig_gene_error_BDNN}
\end{figure}

\section{Conclusion}
In this study, we proposed a stopping criterion for active learning based on error stability. The proposed measure of error stability, i.e., the error ratio, can be applied to any posterior distribution like PAC-Bayesian learning. Unlike the conventional PAC-Bayesian approach, the error ratio does not require any assumption including independence between samples. Furthermore, it is easy to determine the threshold for the proposed criterion, since the range of the threshold is normalized to $[0,1]$ for any dataset. Moreover, to apply the proposed criterion to Gaussian process regression and Bayesian deep neural networks, we derived analytical expressions for both the KL-divergence between GP posteriors and the upper bound of the KL-divergence between the posteriors of dropout-based Bayesian deep neural networks.

In the experiments, to demonstrate that the proposed criterion can be widely applied to various active learning methods, we applied the criterion to the following four models: Bayesian ridge regression, Bayesian logistic regression, Gaussian process regression and Bayesian deep neural networks. We demonstrated that the error ratio has a high correlation with the generalization error except in the case where the ``less is more" phenomenon occurs. Furthermore, we also demonstrated that the proposed criterion can stop learning at similar timings for various datasets when the same threshold is used.

Although we have explained the proposed bound as the bound for the gap between expected generalization errors with respect to Bayes posteriors, Theorem~\ref{theorem_gap_bound} can be applied to any measurable function and any probability density function. It is expected that the bound will be applied to stop various online learning algorithms such as Bayesian optimization, reinforcement learning and the multi-armed bandit. The applicability of the derived bound to other learning frameworks is an important future work. 

\appendix
\section{Proof of Theorem~\ref{theorem_gap_bound} and \ref{theorem_stopping_timing_azuma}}
\label{sec_proof_theorems}
We demonstrate the following two lemmas to prove Theorem~\ref{theorem_gap_bound}.
\begin{lemma}(\citet{Donsker1975,McAllester2003})
Let $h:\Theta \rightarrow \mathbb{R}$ be any measurable function. Then, the following inequality holds:
\begin{equation*}
    \Ex{p(\theta)}{h(\theta)} \leq \KL{p(\theta)}{p'(\theta)} + \log{\Ex{p'(\theta)}{e^{h(\theta)}}}.
\end{equation*}
Here, $p$ and $p'$ are probability distributions on $\Theta$.
\label{lemma_change_of_measure}
\end{lemma}

\begin{lemma}(\citet{Boucheron2013})
Let $X_1,X_2,\ldots,X_t$ be independent random variables with $0\leq X_i \leq b, \;b>0$, let $v\geq \mathbb{E}[X^2_i]$, and define $\phi(\lambda) \coloneqq e^\lambda-\lambda-1$. For any $\lambda>0$, the following inequality holds:
 \begin{equation*}
     \sum^t_{i=1}(\log{\mathbb{E}[e^{\lambda X_i}]-\lambda \mathbb{E}[X_i]})\leq\frac{v}{b^2}\phi(\lambda b).
 \end{equation*}
 \label{lemma_bennett}
\end{lemma}

\textbf{Proof of Theorem~\ref{theorem_gap_bound}}
We denote a difference between $\Ex{q(\theta|S)}{\cL(\theta)}$ and $\Ex{q(\theta|S')}{\cL(\theta)}$ by $\cR(q(\theta|S),q(\theta|S'))$. Supposing that $\tilde{\cL}(\theta) \coloneqq \cL(\theta) - a$, the range of $\tilde{\cL}$ is $[0,b']$, where $b' = b-a$. From Lemma~\ref{lemma_change_of_measure} and \ref{lemma_bennett}, for any $s>0$ we can prove the following inequality:
\begin{align*}
    &s \cR(q(\theta|S),q(\theta|S')) \\
    =& s\left(\Ex{q(\theta|S)}{\cL(\theta)} - a + a - \Ex{q(\theta|S')}{\cL(\theta)}\right) \\
    =& s\left(\Ex{q(\theta|S)}{\tilde{\cL}(\theta)} - \Ex{q(\theta|S')}{\tilde{\cL}(\theta)}\right) \\
    \leq& \KL{q(\theta|S)}{q(\theta|S')}+\log{\Ex{q(\theta|S')}{e^{s \tilde{\cL}(\theta)}}} - \Ex{q(\theta|S')}{s \tilde{\cL}(\theta)} \\
    \leq& \KL{q(\theta|S)}{q(\theta|S')}+\frac{v}{{b'}^2}\phi(s b').
\end{align*}
This implies
\begin{equation}
    \cR(q(\theta|S),q(\theta|S')) \leq \frac{1}{s}\KL{q(\theta|S)}{q(\theta|S')}+\frac{v}{s {b'}^2}\phi(s b'),
    \label{eq_theorem_tight_bound_with_lambda}
\end{equation}
and $\partial \cR/ \partial s$ is explicitly written as
\begin{align*}
    \frac{\partial \cR}{\partial s} =& \frac{\partial}{\partial s} \frac{v \phi(s b') + {b'}^2 D_{\rm KL}}{s {b'}^2}\\
    =& \frac{v(b' e^{s b'}-b')s {b'}^2 - \{v(e^{s b'}-s b'-1) + {b'}^2 D_{\rm KL}\}{b'}^2}{s^2{b'}^4} \\
    =& \frac{vs b' e^{s b'}  - v e^{s b'}+v - {b'}^2 D_{\rm KL}}{s^2{b'}^2},
\end{align*}
where $D_{\rm KL}\coloneqq \KL{q(\theta|S)}{q(\theta|S')}$. Then, equating $\partial \cR/ \partial s$ to zero, we obtain
\begin{align}
    \frac{\partial \cR}{\partial s}=\frac{vs b' e^{s b'}  - v e^{s b'}+v - {b'}^2 D_{\rm KL}}{s^2{b'}^2}&=0 \label{eq_theorem_tight_bound_lambda_1}\\
    (s b'-1) e^{s b'-1}&=\frac{{b'}^2 D_{\rm KL}-v}{v e} \label{eq_theorem_tight_bound_lambda_2}\\
    s b'-1 &=W\left(\frac{{b'}^2 D_{\rm KL}-v}{v e}\right) \label{eq_theorem_tight_bound_lambda_3}\\
    s &=\frac{1}{b'}\left(W\left(\frac{{b'}^2 D_{\rm KL}-v}{v e}\right)+1\right). \label{eq_theorem_tight_bound_lambda_4}
\end{align}
We note that Eq~\eqref{eq_theorem_tight_bound_lambda_3} is a consequence of the fact that the inverse function of $y=xe^x$ is denoted by the Lambert $W$ function. Since $s>0$, we determine that $W$ is the principal branch of the Lambert $W$ function:
\begin{equation}
        s =\frac{1}{b'}\left(W_0\left(\frac{{b'}^2 D_{\rm KL}-v}{v e}\right)+1\right). \label{eq_theorem_tight_bound_lambda_5}
\end{equation}
Substituting Eq.~\eqref{eq_theorem_tight_bound_lambda_5} into Eq.~\eqref{eq_theorem_tight_bound_with_lambda} and defining $u\coloneqq({b'}^2 D_{\rm KL}-v)/ve$ give the following result:
\begin{align}
    \cR(q(\theta|S),q(\theta|S')) \leq& \frac{b'}{W_0(u)+1}D_{\rm KL}+\frac{v}{b'(W_0(u)+1)}\phi(W_0(u)+1) \label{eq_theorem_tight_bound_1}\\
    =&\frac{b'}{W_0(u)+1}D_{\rm KL}+\frac{v}{b'(W_0(u)+1)}\left(e^{W_0(u)+1}-(W_0(u)+1)-1\right) \label{eq_theorem_tight_bound_2}\\
    =&\frac{{b'}^2D_{\rm KL}-v}{b'(W_0(u)+1)}+\frac{v}{b'(W_0(u)+1)}\left(e^{W_0(u)+1}-(W_0(u)+1)\right) \label{eq_theorem_tight_bound_3}\\
    =&\frac{v}{b'(W_0(u)+1)}\left(eu+e^{W_0(u)+1}\right)-\frac{v}{b'} \label{eq_theorem_tight_bound_4}\\
    =&\frac{v}{b'(W_0(u)+1)}\left(W_0(u)e^{W_0(u)+1}+e^{W_0(u)+1}\right)-\frac{v}{b'} \label{eq_theorem_tight_bound_5}\\
    =&\frac{v}{b'}(e^{W_0(u)+1}-1). \label{eq_theorem_tight_bound_6}
\end{align}
In the above derivation, Eq.~\eqref{eq_theorem_tight_bound_3} is reduced to Eq.~\eqref{eq_theorem_tight_bound_4} by using  the fact that $({b'}^2D_{\rm KL}-v)=veu$ because of the definition of $u$, and Eq.~\eqref{eq_theorem_tight_bound_4} is reduced to Eq.~\eqref{eq_theorem_tight_bound_5} by using the fact that $u=W_0(u)e^{W_0(u)}$. Suppose $v={b'}^2$. Then, since $u=(D_{\rm KL}-1)/e$, the following inequality holds:
\begin{equation}
    \cR(q(\theta|S),q(\theta|S'))\leq b' \left\{\exp\left(W_0\left(\frac{D_{\rm KL}-1}{e}\right)+1\right)-1\right\}.
\end{equation}
The lower bound is proved by changing the sign in the proof for the upper bound, completing the proof of Theorem~\ref{theorem_gap_bound}.

\textbf{Proof of Theorem~\ref{theorem_stopping_timing_azuma}}
We assume $\mathbb{E}[E_i | E^{i-1}_1] \leq E_{i-1}$, namely, $E^t_1$ is assumed to be a supermartingale. We can apply the Doob–Meyer decomposition theorem to uniquely decompose $E_t$ into martingale $M_t$ and non-decreasing predictable process $A_t$ as
\begin{equation*}
    E_t = M_t + A_t.
\end{equation*}
Furthermore, $M_t$ is written as
\begin{equation*}
    M_t = E_0 + \sum^t_{i=1}(E_i-\mathbb{E}[E_i|E^{i-1}_1])
\end{equation*}
and 
\begin{equation*}
    M_t-M_{t-1} = E_t-\mathbb{E}[E_t|E^{t-1}_1],
\end{equation*}
where $M_0=E_0$. From Theorem~\ref{theorem_gap_bound}, we have
\begin{equation}
    -(b-a)r(p(\theta|S_{t-1}),p(\theta|S_t)) - (A_t - A_{t-1}) \leq M_t - M_{t-1} \leq (b-a)r(p(\theta|S_t),p(\theta|S_{t-1})) - (A_t - A_{t-1}).
    \label{eq_range_of_mds}
\end{equation}
By applying the Chernoff bound, the following inequality is derived:
\begin{align*}
    {\rm Pr}\left(M_t-M_{t-1} \geq \epsilon\right) \leq& \min_{s>0}e^{-s\epsilon}\mathbb{E}\left[e^{s(M_t-M_{t-1})}\right] \\
    \leq& \min_{s>0}e^{-s\epsilon}\mathbb{E}\left[\mathbb{E}[e^{s(M_t-M_{t-1})}|M^{t-1}_1]\right].
\end{align*}
Since $\mathbb{E}[M_t-M_{t-1}|M^{t-1}_1]=0$, by applying Hoeffding's lemma and Eq.~\eqref{eq_range_of_mds}, we have the following inequality:
\begin{equation*}
    \mathbb{E}\left[e^{s(M_t-M_{t-1})}|M^{t-1}_1\right]\leq \exp\left(\frac{s^2(b-a)^2r^2_t}{8}\right),
\end{equation*}
where $r_t=r(p(\theta|S_{t-1}),p(\theta|S_t))+r(p(\theta|S_t),p(\theta|S_{t-1}))$. Thus,
\begin{align*}
    {\rm Pr}\left(M_t-M_{t-1} \geq \epsilon\right)
    \leq& \min_{s>0}e^{-s\epsilon}\mathbb{E}\left[\mathbb{E}[e^{s(M_t-M_{t-1})}|M^{t-1}_1]\right] \\
    \leq&\min_{s>0}\exp\left(\frac{s^2(b-a)^2r^2_t}{8}-s\epsilon\right).
\end{align*}
Minimizing for $s$ gives the upper bound
\begin{equation*}
    {\rm Pr}\left(M_t-M_{t-1} \geq \epsilon\right)
    \leq\exp\left(-\frac{2\epsilon^2}{(b-a)^2r^2_t}\right).
\end{equation*}
In the same way as above, we also have the following inequality:
\begin{equation*}
    {\rm Pr}\left(M_t-M_{t-1} \leq -\epsilon\right)
    \leq\exp\left(-\frac{2\epsilon^2}{(b-a)^2r^2_t}\right).
\end{equation*}
Combining the above inequalities, we have
\begin{equation*}
    {\rm Pr}\left(|E_t-\mathbb{E}[E_t|E^{t-1}_1]|\leq \epsilon\right)
    \geq1-2\exp\left(-\frac{2\epsilon^2}{(b-a)^2r^2_t}\right).
\end{equation*}

From the inequality, defining $\delta=\{2\exp(-2\epsilon^2/(b-a)^2r^2_t)\}$, we have the following inequality with at least probability $1-\delta$:
\begin{equation*}
    |E_i-\mathbb{E}[E_t|E^{t-1}_1]| \leq (b-a)\sqrt{-\frac{\log(\delta/2)}{2}r^2_t}.
\end{equation*}
We assume that right hand side of the inequality is smaller than $(b-a)\gamma\eta$:
\begin{equation}
    (b-a)\sqrt{-\frac{\log(\delta/2)}{2}r^2_t} \leq (b-a)\gamma\eta.
    \label{eq_stopping_timing_azuma}
\end{equation}
Then, we guarantee that $|E_t-\mathbb{E}[E_t|E^{t-1}_1]|$ is smaller than $(b-a)\gamma\eta$ with at least probability $1-\delta$ when we stop active learning under the condition of Eq.~\eqref{eq_stopping_timing_azuma}.

Equation~\eqref{eq_stopping_timing_azuma} leads to the following inequality:
\begin{equation*}
    \lambda_t=\frac{r_t}{\gamma} \leq \sqrt{-\frac{2}{\log(\delta/2)}}\eta.
\end{equation*}
Therefore, Theorem~\ref{theorem_stopping_timing_azuma} is proved.

\section{KL-divergence between GPs}
\label{sec_kl_gps}
\begin{lemma}
Let $q(f|S)$ and $q(f|S')$ be the posteriors with respect to $\theta$ given $S=(X,Y)$ and $S'=(X',Y')$, respectively. We assume that the prior of $q(f|S)$ is the same as that of $q(f|S')$. Then, the following inequality holds:
\begin{equation*}
    \KL{q(f|S)}{q(f|S')} = \KL{q(\vf_{X_+}|S))}{q(\vf_{X_+}|S'))},
\end{equation*}
\label{lemma_bound_of_kl_divergence}
where $X_+\coloneqq X\cup X'$ and $\vf_{X_+}\coloneqq f(X_+)$.
\end{lemma}
\begin{proof}
Let $X_\Omega$ be a universal set of input data. We denote $X_\Omega/X_+$ by $X_\ast$. From the chain rule of the KL-divergence~\citep{Gray2011}, the following equation holds:
\begin{align}
    &\KL{q(f|S)}{q(f|S')} \notag \\
    =&\KL{q(\vf_{X_+}|S)}{q(\vf_{X_+}|S')} \notag \\
    &+\Ex{q(\vf_{X_+}|S)}{\KL{q(\vf_{X_\ast}|\vf_{X_+},S)}{q(\vf_{X_\ast}|\vf_{X_+},S')}}.
    \label{chain_rule}
\end{align}
We denote the prior of $q(\vf_{X_\ast},\vf_{X_+}|S)$ and $q(\vf_{X_\ast},\vf_{X_+}|S')$ by $p(\vf_{X_\ast},\vf_{X_+})$. From the Bayesian theorem, the following equation holds:
\begin{align*}
    q(\vf_{X_\ast}|\vf_{X_+},S) =& \frac{p(\vf_{X_\ast},\vf_{X_+}|S)}{p(\vf_{X_+}|S)} \\
    =&\frac{p(Y|\vf_{X_+},X)p(\vf_{X_\ast}|\vf_{X_+})p(\vf_{X_+})}{p(Y|X)}  \frac{p(Y|X)}{p(Y|\vf_{X_+},X)p(\vf_{X_+})}  \\
    =&\frac{p(\vf_{X_\ast},\vf_{X_+})}{p(\vf_{X_+})} =p(\vf_{X_\ast}|\vf_{X_+}).
\end{align*}
Similarly, $q(\vf_{X_\ast}|\vf_{X_+},S')=p(\vf_{X_\ast}|\vf_{X_+})$ also holds. Therefore, if the prior of $q(f|S)$ is the same as that of $q(f|S')$, the second term of Eq.~\eqref{chain_rule} is zero.
\end{proof}

\begin{lemma}
Let $q(f|S_t)$ and $q(f|S_{t-1})$ be the GP posteriors given $S_t=\{(\vx_i,y_i)\}^{n_t}_{i=1}$ and $S_{t-1}=\{(\vx_i,y_i)\}^{n_{t-1}}_{i=1}$, respectively. We assume that the prior of $q(f|S_t)$ is the same as that of $q(f|S_{t-1})$.
Let $\mu_t$, $\sigma_t$ and $\beta$ be the mean and covariance functions of $q(f|S_t)$ and the accuracy of Gaussian noise, respectively. Then, the following equation holds:
\begin{align*}
    \KL{q(f|S_{t-1})}{q(f|S_t)}=&\frac{1}{2}\beta \sigma_{t-1}(\vx_{n_t},\vx_{n_t})-\frac{1}{2}\log{(1+\beta \sigma_{t-1}(\vx_{n_t},\vx_{n_t}))}  \\
    &+\frac{1}{2}\frac{\beta \sigma_{t-1}(\vx_{n_t},\vx_{n_t})}{\sigma_{t-1}(\vx_{n_t},\vx_{n_t})+\beta^{-1}}(y_{n_t}-\mu_{t-1}(\vx_{n_t}))^2.
\end{align*}
\label{lemma_gp_kl2}
\end{lemma}
\begin{proof}

From Lemma~\ref{lemma_bound_of_kl_divergence}, the following equation holds:
\begin{equation*}
    \KL{q(f|S_{t-1})}{q(f|S_t)}=\KL{q(\vf|S_{t-1})}{q(\vf|S_t)},
\end{equation*}
where $\vf:=(f(\vx_1),f(\vx_2),\cdots,f(\vx_{n_t}))$.
When $S_t=(X_t,Y_t)$ is observed, $q(\vf|S_t)$ can be described as
\begin{align*}
    q(\vf|S_t) &= \frac{p(Y_t|\vf,X_t)p(\vf)}{ p(Y_t|X_t)} \\
    &= \frac{p(y_{n_t}|\vf,\vx_{n_t})p(Y_{t-1}|\vf,X_{t-1})p(\vf)}{\int p(y_{n_t}|\vf',\vx_{n_t})p(Y_{t-1}|\vf',X_{t-1})p(\vf') d\vf'} \\
    &= \frac{p(y_{n_t}|\vf,\vx_{n_t})p(Y_{t-1}|X_{t-1})q(\vf|S_{t-1})}{\int p(y_{n_t}|\vf',\vx_{n_t})p(Y_{t-1}|X_{t-1})q(\vf'|S_{t-1}) d\vf'} \\
    &= \frac{p(y_{n_t}|\vf,\vx_{n_t})q(\vf|S_{t-1})}{p(y_{n_t}|\vx_{n_t})}.
\end{align*}
From this equation, $\KL{q(\vf|S_{t-1})}{q(\vf|S_t)}$ can be rewritten as
\begin{align}
    &\KL{q(\vf|S_{t-1})}{q(\vf|S_t)} \notag \\
    =& \Ex{q(\vf|S_{t-1})}{\log\frac{q(\vf|S_{t-1})p(y_{n_t}|\vx_{n_t})}{p(y_{n_t}|\vf,\vx_{n_t})q(\vf|S_{t-1})}} \notag \\
    =& \log{p(y_{n_t}|\vx_{n_t})}-\Ex{q(\vf|S_{t-1})}{\log{p(y_{n_t}|\vf,\vx_{n_t})}} \notag \\
    =& \log{\int q(f_t|S_{t-1})p(y_{n_t}|f_t) df_t}-\int q(f_t|S_{t-1})\log{p(y_{n_t}|f_t)}df_t,
    \label{eq_kl_expansion}
\end{align}
where $f_t:=f(\vx_{n_t})$. The first term of Eq.~\eqref{eq_kl_expansion} becomes the logarithm of a normal distribution since $p(y_{n_t}|f_t)$ and $q(f_t|S_{t-1})$ are normal distributions. Specifically, from $p(y_{n_t}|f_t)=\mathcal{N}(y_{n_t}|f_t,\beta^{-1})$ and $p(f_t|S_{t-1})=\mathcal{N}(f_t|\mu_{t-1}(\vx_{n_t}),\sigma_{t-1}(\vx_{n_t},\vx_{n_t}))$, the following equation holds:
\begin{align}
    &\log{\int p(y_{n_t}|f_t)q(f_t|S_{t-1}) df_t} \notag \\ =&\log{\mathcal{N}(y_{n_t}|\mu_{t-1}(\vx_{n_t}),\sigma_{t-1}(\vx_{n_t},\vx_{n_t})+\beta^{-1})}.
    \label{eq_log_exp_normal}
\end{align}

The second term can be rewritten as
\begin{align}
    -\int q(f_t|S_{t-1})\log{p(y_{n_t}|f_t)}df_t 
    =& \Ex{q(f_t|S_{t-1})}{\frac{\beta}{2}(y_{n_t}-f_t)^2}+\frac{1}{2}\log{2\pi\beta^{-1}} \notag \\
    =& {\frac{\beta}{2}\left(y^2_{n_t}-2y_{n_t}\mathbb{E}[f_t]+\mathbb{E}[f^2_t]\right)}+\frac{1}{2}\log{2\pi\beta^{-1}} \notag \\
    =& \frac{\beta}{2}(y_{n_t}-\mu_{t-1}(\vx_{n_t}))^2+\frac{\beta}{2}\sigma_{t-1}(\vx_{n_t},\vx_{n_t}) +\frac{1}{2}\log{2\pi\beta^{-1}}.
    \label{eq_exp_log_normal}
\end{align}
From the above equation, the lemma is derived as
\begin{align*}
    &\KL{q(f|S_{t-1})}{q(f|S_t)}  \\
    =&-\frac{(y_{n_t}-\mu_{t-1}(\vx_{n_t}))^2}{2(\sigma_{t-1}(\vx_{n_t},\vx_{n_t})+\beta^{-1})}-\frac{1}{2}\log{2\pi(\sigma_{t-1}(\vx_{n_t},\vx_{n_t})+\beta^{-1})}  \\
    &+\frac{\beta}{2}(y_{n_t}-\mu_{t-1}(\vx_{n_t}))^2+\frac{\beta}{2}\sigma_{t-1}(\vx_{n_t},\vx_{n_t})+\frac{1}{2}\log{2\pi\beta^{-1}} \\
    =&\frac{1}{2}\beta \sigma_{t-1}(\vx_{n_t},\vx_{n_t})-\frac{1}{2}\log{(1+\beta \sigma_{t-1}(\vx_{n_t},\vx_{n_t}))}  \\
    &+\frac{1}{2}\frac{\beta \sigma_{t-1}(\vx_{n_t},\vx_{n_t})}{\sigma_{t-1}(\vx_{n_t},\vx_{n_t})+\beta^{-1}}(y_{n_t}-\mu_{t-1}(\vx_{n_t}))^2.
\end{align*}
\end{proof}

\begin{lemma}
Let $q(f|S_t)$ and $q(f|S_{t-1})$ be the GP posteriors given $S_t=\{(\vx_i,y_i)\}^{n_t}_{i=1}$ and $S_{t-1}=\{(\vx_i,y_i)\}^{n_{t-1}}_{i=1}$, respectively. We assume that the prior of $q(f|S_t)$ is the same as that of $q(f|S_{t-1})$.
Let $\mu_t$, $\sigma_t$ and $\beta$ be the mean and covariance functions of $q(f|S_t)$, and the accuracy of Gaussian noise, respectively. Then, the following equation holds:
\begin{align*}
    \KL{p(f|S_t)}{p(f|S_{t-1})}=&\KL{p(\vf_{X_t}|S_t)}{p(\vf_{X_t}|S_{t-1})} \\
    =&\frac{1}{2}\log{(1+\beta\sigma_{t-1}(\vx_{n_t},\vx_{n_t}))}-\frac{1}{2}\frac{\sigma_{t-1}(\vx_{n_t},\vx_{n_t})}{\sigma_{t-1}(\vx_{n_t},\vx_{n_t})+\beta\inv} \\
    &+\frac{1}{2}\frac{\sigma_{t-1}(\vx_{n_t},\vx_{n_t})(y_{n_t}-\mu_{t-1}(\vx_{n_t}))^2}{(\sigma_{t-1}(\vx_{n_t},\vx_{n_t})+\beta\inv)^2}.
\end{align*}
\label{lemma_gp_kl}
\end{lemma}
\begin{proof}
In analogy with Eq.~\eqref{eq_kl_expansion}, the following equation holds:
\begin{align*}
    &\KL{q(f|S_t)}{q(f|S_{t-1})} \\
    =& \Ex{q(\vf|S_t)}{\log\frac{p(y_{n_t}|\vf,\vx_{n_t})q(\vf|S_{t-1})}{q(\vf|S_{t-1})p(y_{n_t}|\vx_{n_t})}} \\
    =& \Ex{q(\vf|S_t)}{\log{p(y_{n_t}|\vf,\vx_{n_t})}}-\log{p(y_{n_t}|\vx_{n_t})} \\
    =& \int q(f_t|S_t)\log{p(y_{n_t}|f_t,\vx_{n_t})}df_t -\log{\int q(f_t|S_{t-1})p(y_{n_t}|f_t,\vx_{n_t}) df_t},
\end{align*}
where $f_t:=f(\vx_{n_t})$. The second term is the same as Eq.~\eqref{eq_log_exp_normal}. In the same way as for Eq.~\eqref{eq_exp_log_normal}, the first term is derived as
\begin{align*}
    &\int q(f_t|S_t)\log{p(y_{n_t}|f_t)}df_t \\ 
    =& -\frac{\beta}{2}(y_{n_t}-\mu_t(\vx_{n_t}))^2-\frac{\beta}{2}\sigma_t(\vx_{n_t},\vx_{n_t}) -\frac{1}{2}\log{2\pi\beta^{-1}}.
\end{align*}
Regarding $q(f_t|S_t)$ as the posterior whose prior is $q(f_t|S_{t-1})$ observing $(\vx_{n_t},y_{n_t})$, the values of the mean function and covariance function corresponding to $\vx_{n_t}$ are derived as
\begin{align*}
    \mu_t(\vx_{n_t}) =& \mu_{t-1}(\vx_{n_t}) + \frac{\sigma_{t-1}(\vx_{n_t},\vx_{n_t})}{\sigma_{t-1}(\vx_{n_t},\vx_{n_t})+\beta\inv}(y_{n_t}-\mu_{t-1}(\vx_{n_t})) \\
    =& \frac{\beta\inv\mu_{t-1}(\vx_{n_t})+\sigma_{t-1}(\vx_{n_t},\vx_{n_t})y_{n_t}}{\sigma_{t-1}(\vx_{n_t},\vx_{n_t})+\beta\inv} \\
    \sigma_t(\vx_{n_t},\vx_{n_t}) =& \sigma_{t-1}(\vx_{n_t},\vx_{n_t}) - \frac{\sigma^2_{t-1}(\vx_{n_t},\vx_{n_t})}{\sigma_{t-1}(\vx_{n_t},\vx_{n_t})+\beta\inv} \\
    =& \frac{\beta\inv\sigma_{t-1}(\vx_{n_t},\vx_{n_t})}{\sigma_{t-1}(\vx_{n_t},\vx_{n_t})+\beta\inv}.
\end{align*}
From the result, the second term is rewritten as
\begin{align*}
    &\int q(f_t|S_t)\log{p(y_{n_t}|f_t)}df_t \\ 
    =& -\frac{1}{2}\frac{\beta\inv(y_{n_t}-\mu_{t-1}(\vx_{n_t}))^2}{(\sigma_{t-1}(\vx_{n_t},\vx_{n_t})+\beta\inv)^2}-\frac{1}{2}\frac{\sigma_{t-1}(\vx_{n_t},\vx_{n_t})}{\sigma_{t-1}(\vx_{n_t},\vx_{n_t})+\beta\inv} -\frac{1}{2}\log{2\pi\beta^{-1}}.
\end{align*}
Therefore, the following equation holds:
\begin{align*}
    &\KL{q(f|S_t)}{q(f|S_{t-1})} \\
    =& \int q(f_t|S_t)\log{p(y_{n_t}|f_t,\vx_{n_t})}df_t -\log{\int q(f_t|S_{t-1})p(y_{n_t}|f_t,\vx_{n_t}) df_t}  \\
    =& -\frac{1}{2}\frac{\beta\inv(y_{n_t}-\mu_{t-1}(\vx_{n_t}))^2}{(\sigma_{t-1}(\vx_{n_t},\vx_{n_t})+\beta\inv)^2}-\frac{1}{2}\frac{\sigma_{t-1}(\vx_{n_t},\vx_{n_t})}{\sigma_{t-1}(\vx_{n_t},\vx_{n_t})+\beta\inv} -\frac{1}{2}\log{2\pi\beta^{-1}} \\
    &+\frac{1}{2}\frac{(y_{n_t}-\mu_{t-1}(\vx_{n_t}))^2}{\sigma_{t-1}(\vx_{n_t},\vx_{n_t})+\beta^{-1}}+\frac{1}{2}\log{2\pi(\sigma_{t-1}(\vx_{n_t},\vx_{n_t})+\beta^{-1})} \\
    =&\frac{1}{2}\frac{\sigma_{t-1}(\vx_{n_t},\vx_{n_t})(y_{n_t}-\mu_{t-1}(\vx_{n_t}))^2}{(\sigma_{t-1}(\vx_{n_t},\vx_{n_t})+\beta\inv)^2}+\frac{1}{2}\log{(1+\beta\sigma_{t-1}(\vx_{n_t},\vx_{n_t}))}-\frac{1}{2}\frac{\sigma_{t-1}(\vx_{n_t},\vx_{n_t})}{\sigma_{t-1}(\vx_{n_t},\vx_{n_t})+\beta\inv}.
\end{align*}
\end{proof}

\section{Tight bound for the KL-divergence between dropout-based deep Bayes posteriors}
\label{sec_kl_bdls}
\begin{lemma}
Let $p(\theta)$ and $q(\theta)$ be posterior distributions with respect to a dropout-based Bayesian deep neural network. Namely, the posterior distribution is assumed to have the following form:
\begin{equation*}
    p(\theta)=\prod^L_{l=1}p(\vW_l)p(\vb_l),
\end{equation*}
where
\begin{align*}
    p(\vW_l)&=\prod^{H_{l-1}}_{h=1}p(\vw_{lh}) \\
    p(\vb_l)&=\mathcal{N}(\vnu_l,\sigma^2_l\vI)
\end{align*}
and
\begin{equation*}
    p(\vw_{lh})=p_l\mathcal{N}(\vm_{lh},\sigma^2_l\vI)+(1-p_l)\mathcal{N}(\vzero,\sigma^2_l\vI).
\end{equation*}
Now, $\KL{p(\theta)}{q(\theta)}$ is upper bounded as
\begin{align*}
    \KL{p(\theta)}{q(\theta)} \leq& \frac{1}{2}\sum^L_{l=1}\left\{\frac{p_l}{\sigma'^2_l}\|\vM_l-\vM'_l\|^2_F+\frac{1}{\sigma'^2_l}\|\vnu_l-\vnu'_l\|^2+H_{l-1}D_{kl}[p_l||q_l]\right\} \\
    &+\sum^L_{l=1}\frac{(1+p_lH_{l-1})H_{l}}{2}\left\{\frac{\sigma^2_l}{\sigma'^2_l}-\log{\frac{\sigma^2_l}{\sigma'^2_l}}-1\right\},
\end{align*}
where $\vM_l\coloneqq (\vm_{l1},\vm_{l2},\ldots,\vm_{lH_l})$ and $D_{kl}[p||q]\coloneqq p\log{\frac{p}{q}}+(1-p)\log{\frac{1-p}{1-q}}$.
\label{lemma_db_posterior_kl}
\end{lemma}
\begin{proof}
From $p(\theta)=\prod^L_{l=1}p(\vW_l)p(\vb_l)$, we obtain
\begin{align}
    D_{\rm KL}[p(\theta) || q(\theta)] =& \int \prod^L_{l=1}p(\vW_l)p(\vb_l) \log{\frac{\prod^L_{l=1}p(\vW_l)p(\vb_l)}{\prod^L_{l=1}q(\vW_l)q(\vb_l)}}d\vW_ld\vb_l \notag \\
    =& \sum^L_{l=1}\int p(\vW_l)p(\vb_l) \log{\frac{p(\vW_l)p(\vb_l)}{q(\vW_l)q(\vb_l)}}d\vW_ld\vb_l \notag \\
    =& \sum^L_{l=1}\int p(\vW_l) \log{\frac{p(\vW_l)}{q(\vW_l)}}d\vW_l + \sum^L_{l=1}\int p(\vb_l) \log{\frac{p(\vb_l)}{q(\vb_l)}}d\vb_l \notag \\
    =& \sum^L_{l=1}\sum^{H_{l-1}}_{h=1}D_{\rm KL}[p(\vw_{lh}) || q(\vw_{lh})] + \sum^L_{l=1}D_{\rm KL}[p(\vb_l) || q(\vb_l)].
    \label{eq_kl_db_posteriors}
\end{align}
The second term of the above formula is the summation of the KL-divergences between Gaussians and has an analytical expression, though the first term is the summation of the KL-divergences between mixture distributions and does not have an analytical solution. We consider the upper bound of the first term by introducing a latent variable $z_{lh} \in \{0,1\}$ into $p(\vw_{lh})$ to indicate the correspondence between the mixture and $\vw_{lh}$. The chain rule of the KL-divergence leads to 
\begin{equation}
    \KL{p(\vw_{lh},z_{lh})}{q(\vw_{lh},z_{lh})}\geq \KL{p(\vw_{lh})}{q(\vw_{lh})}.
    \label{eq_inequality_of_mixture_model}
\end{equation}
Since we do not know the correspondence between latent variables of GMMs, the joint distribution $p(\vw_{lh},z_{lh})$ is written as
\begin{align*}
    p(\vw_{lh},z_{lh})= p(\vw_{lh}|z_{lh})p(z_{lh}) 
    =p_{z_{lh}}\mathcal{N}(z_{lh}\vm_{lh},\sigma^2_l\vI),
\end{align*}
while the joint distribution $q(\vw_{lh},z_{lh})$ is written as
\begin{equation*}
    q(\vw_{lh},z_{lh})= q_{\pi(z_{lh})}\mathcal{N}(\pi(z_{lh})\vm_{lh},\sigma^2_l\vI),
\end{equation*}
where $\pi(z)$ is any permutation. Then, the KL-divergence between these joint distributions is analytically computable as
\begin{align*}
    &D_{\rm KL}[p(\vw_{lh},z_{lh}) || q(\vw_{lh},z_{lh})]  \\
    =&\int\sum_{z\in\{0,1\}} p(\vw_{lh},z_{lh}=z) \log{\frac{p(\vw_{lh},z_{lh}=z)}{q(\vw_{lh},z_{lh}=z)}}d\vw_{lh} \\
    =&\sum_{z\in\{0,1\}}D_{\rm KL}[p(\vw_{lh},z_{lh}=z) || q(\vw_{lh},z_{lh}=z)] \\
    =&\sum_{z\in\{0,1\}}p_z\left[D_{\rm KL}[\mathcal{N}(z\vm_{lh},\sigma^2_l\vI)  || \mathcal{N}(\pi(z)\vm'_{lh},\sigma'^2_l\vI)]+\log{\frac{p_z}{q_{\pi(z)}}}\right].
\end{align*}

Equation~\eqref{eq_inequality_of_mixture_model} holds for any $\pi(z)$, but it is preferable to set the latent variable with a tighter bound. In our case, there are tow mixture components and it is also known that the centroid of one of those two components must be zero and that the two components share the same variance. In this situation, by setting $\pi(z)=z$, we can minimize the KL-divergence between joint distributions. Now we have the following tighter upper bound:
\begin{align*}
    &D_{\rm KL}[p(\vw_{lh},z_{lh}) || q(\vw_{lh},z_{lh})] \\
    =&p_l D_{\rm KL}[\mathcal{N}(\vm_{lh},\sigma^2_l\vI)  || \mathcal{N}(\vm'_{lh},\sigma'^2_l\vI)]+D_{kl}[p_l||q_l]  \\
    =&\frac{p_l}{2}\left\{H_l\frac{\sigma^2_l}{\sigma'^2_l}-H_l\log{\frac{\sigma^2_l}{\sigma'^2_l}}+\frac{1}{\sigma'^2_l}\|\vm_{lh}-\vm'_{lh}\|^2-H_l+D_{kl}[p_l||q_l]\right\} \\
    \geq& \KL{p(\vw_{lh})}{q(\vw_{lh})}.
\end{align*}
Substituting the result into Eq.~\eqref{eq_kl_db_posteriors} leads to the following inequality:
\begin{align*}
    &D_{\rm KL}[p(\theta) || q(\theta)]  \\
    \leq& \sum^L_{l=1}\sum^{H_{l-1}}_{h=1}\frac{p_l}{2}\left\{H_l\frac{\sigma^2_l}{\sigma'^2_l}-H_l\log{\frac{\sigma^2_l}{\sigma'^2_l}}+\frac{1}{\sigma'^2_l}\|\vm_{lh}-\vm'_{lh}\|^2-H_l+D_{kl}[p_l||q_l]\right\} \\
    &+ \frac{1}{2}\sum^L_{l=1}\left\{H_l\frac{\sigma^2_l}{\sigma'^2_l}-H_l\log{\frac{\sigma^2_l}{\sigma'^2_l}}+\frac{1}{\sigma'^2_l}\|\vnu_l-\vnu'_l\|^2-H_l\right\} \\
    =&\frac{1}{2}\sum^L_{l=1}\left\{\frac{p_l}{\sigma'^2_l}\|\vM_l-\vM'_l\|^2_F+\frac{1}{\sigma'^2_l}\|\vnu_l-\vnu'_l\|^2+H_{l-1}D_{kl}[p_l||q_l]\right\} \\
    &+\sum^L_{l=1}\frac{(1+p_lH_{l-1})H_{l}}{2}\left\{\frac{\sigma^2_l}{\sigma'^2_l}-\log{\frac{\sigma^2_l}{\sigma'^2_l}}-1\right\}.
\end{align*}
\end{proof}

\section*{Acknowledgments}
This work was partially supported by the NEDO Grant Number JPNP18002, JST CREST Grant Number JPMJCR1761, JPMJCR2015 and JST-Mirai Program Grant Number JPMJMI19G1.

% \vskip 0.2in
\bibliography{reference}

\begin{thebibliography}{49}
\providecommand{\natexlab}[1]{#1}
\providecommand{\url}[1]{\texttt{#1}}
\expandafter\ifx\csname urlstyle\endcsname\relax
  \providecommand{\doi}[1]{doi: #1}\else
  \providecommand{\doi}{doi: \begingroup \urlstyle{rm}\Url}\fi

\bibitem[Alquier and Guedj(2018)]{Alquier2018}
P.~Alquier and B.~Guedj.
\newblock {Simpler PAC-Bayesian} bounds for hostile data.
\newblock \emph{{Machine Learning}}, 107\penalty0 (5):\penalty0 887--902, 2018.
\newblock \doi{10.1007/s10994-017-5690-0}.
\newblock URL \url{https://hal.inria.fr/hal-01385064}.

\bibitem[{Altschuler} and {Bloodgood}(2019)]{Altschuler2019}
M.~{Altschuler} and M.~{Bloodgood}.
\newblock Stopping active learning based on predicted change of {F} measure for
  text classification.
\newblock In \emph{2019 IEEE 13th International Conference on Semantic
  Computing (ICSC)}, pages 47--54, Jan 2019.
\newblock \doi{10.1109/ICOSC.2019.8665646}.

\bibitem[Balcan et~al.(2009)Balcan, Beygelzimer, and Langford]{Balcan2009}
M.~F. Balcan, A.~Beygelzimer, and J.~Langford.
\newblock Agnostic active learning.
\newblock \emph{Journal of Computer and System Sciences}, 75\penalty0
  (1):\penalty0 78--89, 2009.
\newblock ISSN 0022-0000.
\newblock \doi{https://doi.org/10.1016/j.jcss.2008.07.003}.
\newblock URL
  \url{https://www.sciencedirect.com/science/article/pii/S0022000008000652}.
\newblock Learning Theory 2006.

\bibitem[Bloodgood and Grothendieck(2013)]{Bloodgood2013}
M.~Bloodgood and J.~Grothendieck.
\newblock Analysis of stopping active learning based on stabilizing
  predictions.
\newblock In \emph{Proceedings of the Seventeenth Conference on Computational
  Natural Language Learning}, pages 10--19, Sofia, Bulgaria, August 2013.
  Association for Computational Linguistics.
\newblock URL \url{https://www.aclweb.org/anthology/W13-3502}.

\bibitem[Bloodgood and Vijay-Shanker(2009)]{Bloodgood2009}
M.~Bloodgood and K.~Vijay-Shanker.
\newblock A method for stopping active learning based on stabilizing
  predictions and the need for user-adjustable stopping.
\newblock In \emph{Proceedings of the Thirteenth Conference on Computational
  Natural Language Learning ({C}o{NLL}-2009)}, pages 39--47, Boulder, Colorado,
  June 2009. Association for Computational Linguistics.
\newblock URL \url{https://www.aclweb.org/anthology/W09-1107}.

\bibitem[Boucheron et~al.(2013)Boucheron, Lugosi, and Massart]{Boucheron2013}
S.~Boucheron, G.~Lugosi, and P.~Massart.
\newblock \emph{{Concentration inequalities : a non asymptotic theory of
  independence}}.
\newblock {Oxford University Press}, 2013.
\newblock URL \url{https://hal.inria.fr/hal-00942704}.

\bibitem[Bousquet and Elisseeff(2002)]{Bousquet2002}
O.~Bousquet and A.~Elisseeff.
\newblock Stability and generalization.
\newblock \emph{Journal of Machine Learning Research}, 2:\penalty0 499--526,
  2002.

\bibitem[Corless et~al.(1996)Corless, Gonnet, Hare, Jeffrey, and
  Knuth]{Corless1996}
R.~M. Corless, G.~H. Gonnet, D.~E.~G. Hare, D.~J. Jeffrey, and D.~E. Knuth.
\newblock On the {Lambert W} function.
\newblock In \emph{ADVANCES IN COMPUTATIONAL MATHEMATICS}, pages 329--359,
  1996.

\bibitem[Dasgupta(2011)]{Dasgupta2011}
S.~Dasgupta.
\newblock Two faces of active learning.
\newblock \emph{Theoretical Computer Science}, 412\penalty0 (19):\penalty0
  1767--1781, 2011.
\newblock ISSN 0304-3975.
\newblock \doi{https://doi.org/10.1016/j.tcs.2010.12.054}.
\newblock URL
  \url{https://www.sciencedirect.com/science/article/pii/S0304397510007620}.

\bibitem[{Do}(2003)]{Do2003}
M.~N. {Do}.
\newblock Fast approximation of {Kullback-Leibler} distance for dependence
  trees and hidden markov models.
\newblock \emph{IEEE Signal Processing Letters}, 10\penalty0 (4):\penalty0
  115--118, 2003.
\newblock \doi{10.1109/LSP.2003.809034}.

\bibitem[Donmez et~al.(2007)Donmez, Carbonell, and Bennett]{Donmez2007}
P.~Donmez, J.~G. Carbonell, and P.~N. Bennett.
\newblock Dual strategy active learning.
\newblock In J.~N. Kok, J.~Koronacki, R.~L. de~Mantaras, S.~Matwin,
  D.~Mladeni{\v{c}}, and A.~Skowron, editors, \emph{Machine Learning: ECML
  2007}, pages 116--127, Berlin, Heidelberg, 2007. Springer Berlin Heidelberg.
\newblock ISBN 978-3-540-74958-5.

\bibitem[Donsker and Varadhan(1975)]{Donsker1975}
M.~D. Donsker and S.~R.~S. Varadhan.
\newblock Asymptotic evaluation of certain markov process expectations for
  large time, {I}.
\newblock \emph{Communications on Pure and Applied Mathematics}, 28\penalty0
  (1):\penalty0 1--47, 1975.
\newblock \doi{10.1002/cpa.3160280102}.
\newblock URL
  \url{https://onlinelibrary.wiley.com/doi/abs/10.1002/cpa.3160280102}.

\bibitem[Dua and Graff(2017)]{Dua2017}
D.~Dua and C.~Graff.
\newblock {UCI} machine learning repository, 2017.
\newblock URL \url{http://archive.ics.uci.edu/ml}.

\bibitem[Freund et~al.(1992)Freund, Seung, Shamir, and Tishby]{Freund1992}
Y.~Freund, H.~S. Seung, E.~Shamir, and N.~Tishby.
\newblock Information, prediction, and query by committee.
\newblock In S.~Hanson, J.~Cowan, and C.~Giles, editors, \emph{Advances in
  Neural Information Processing Systems}, volume~5. Morgan-Kaufmann, 1992.
\newblock URL
  \url{https://proceedings.neurips.cc/paper/1992/file/3871bd64012152bfb53fdf04b401193f-Paper.pdf}.

\bibitem[Gal and Ghahramani(2016)]{Gal2016}
Y.~Gal and Z.~Ghahramani.
\newblock Dropout as a {Bayesian} approximation: Representing model uncertainty
  in deep learning.
\newblock In M.~F. Balcan and K.~Q. Weinberger, editors, \emph{Proceedings of
  The 33rd International Conference on Machine Learning}, volume~48 of
  \emph{Proceedings of Machine Learning Research}, pages 1050--1059, New York,
  New York, USA, 20--22 Jun 2016. PMLR.
\newblock URL \url{http://proceedings.mlr.press/v48/gal16.html}.

\bibitem[Germain et~al.(2016)Germain, Bach, Lacoste, and
  Lacoste-Julien]{Germain2016}
P.~Germain, F.~Bach, A.~Lacoste, and S.~Lacoste-Julien.
\newblock {PAC-Bayesian theory meets Bayesian inference}.
\newblock In \emph{Proceedings of the 30th International Conference on Neural
  Information Processing Systems}, NIPS'16, pages 1884--1892, USA, 2016. Curran
  Associates Inc.
\newblock ISBN 978-1-5108-3881-9.
\newblock URL \url{http://dl.acm.org/citation.cfm?id=3157096.3157307}.

\bibitem[Gray(2011)]{Gray2011}
R.~M. Gray.
\newblock \emph{Entropy and Information Theory}.
\newblock Springer-Verlag New York, Inc., 2011.
\newblock \doi{10.1007/978-1-4419-7970-4}.

\bibitem[Guedj(2019)]{Guedj2019}
B.~Guedj.
\newblock A primer on {PAC-Bayesian} learning.
\newblock In \emph{Proceedings of the second congress of the French
  Mathematical Society}, 2019.
\newblock URL \url{https://arxiv.org/abs/1901.05353}.

\bibitem[Hanneke(2014)]{Hanneke2014}
S.~Hanneke.
\newblock Theory of disagreement-based active learning.
\newblock \emph{Foundations and Trends^^c2^^ae in Machine Learning}, 7\penalty0
  (2-3):\penalty0 131--309, 2014.
\newblock ISSN 1935-8237.
\newblock \doi{10.1561/2200000037}.
\newblock URL \url{http://dx.doi.org/10.1561/2200000037}.

\bibitem[Hershey and Olsen(2007)]{Hershey2007}
J.~R. Hershey and P.~A. Olsen.
\newblock Approximating the {Kullback Leibler} divergence between gaussian
  mixture models.
\newblock In \emph{2007 IEEE International Conference on Acoustics, Speech and
  Signal Processing - ICASSP '07}, volume~4, pages IV--317--IV--320, 2007.
\newblock \doi{10.1109/ICASSP.2007.366913}.

\bibitem[Hino(2020)]{hino2020active}
H.~Hino.
\newblock Active learning: Problem settings and recent developments, 2020.

\bibitem[Houlsby et~al.(2011)Houlsby, Huszar, Ghahramani, and
  Lengyel]{Houlsby2011}
N.~Houlsby, F.~Huszar, Z.~Ghahramani, and M.~Lengyel.
\newblock Bayesian active learning for classification and preference learning.
\newblock \emph{{CoRR}}, abs/1112.5745, 2011.
\newblock URL
  \url{http://dblp.uni-trier.de/db/journals/corr/corr1112.html#abs-1112-5745}.

\bibitem[Huang et~al.(2010)Huang, Jin, and Zhou]{Huang2010}
S.~J. Huang, R.~Jin, and Z.~H. Zhou.
\newblock Active learning by querying informative and representative examples.
\newblock In J.~Lafferty, C.~Williams, J.~Shawe-Taylor, R.~Zemel, and
  A.~Culotta, editors, \emph{Advances in Neural Information Processing
  Systems}, volume~23. Curran Associates, Inc., 2010.
\newblock URL
  \url{https://proceedings.neurips.cc/paper/2010/file/5487315b1286f907165907aa8fc96619-Paper.pdf}.

\bibitem[Ishibashi and Hino(2020)]{Ishibashi2020}
H.~Ishibashi and H.~Hino.
\newblock Stopping criterion for active learning based on deterministic
  generalization bounds.
\newblock In \emph{The 23rd International Conference on Artificial Intelligence
  and Statistics, {AISTATS} 2020, 26-28 August 2020, Online [Palermo, Sicily,
  Italy]}, pages 386--397, 2020.
\newblock URL \url{http://proceedings.mlr.press/v108/ishibashi20a.html}.

\bibitem[Karzand and Nowak(2020)]{Karzand2020}
M.~Karzand and R.~Nowak.
\newblock Maximin active learning in overparameterized model classes.
\newblock \emph{IEEE Journal on Selected Areas in Information Theory},
  1:\penalty0 167--177, 2020.

\bibitem[Kirsch et~al.(2019)Kirsch, van Amersfoort, and Gal]{Kirsch2019}
A.~Kirsch, J.~van Amersfoort, and Y.~Gal.
\newblock {BatchBALD}: Efficient and diverse batch acquisition for deep
  {Bayesian} active learning.
\newblock In H.~Wallach, H.~Larochelle, A.~Beygelzimer, F.~d\textquotesingle
  Alch\'{e}-Buc, E.~Fox, and R.~Garnett, editors, \emph{Advances in Neural
  Information Processing Systems}, volume~32. Curran Associates, Inc., 2019.
\newblock URL
  \url{https://proceedings.neurips.cc/paper/2019/file/95323660ed2124450caaac2c46b5ed90-Paper.pdf}.

\bibitem[Konyushkova et~al.(2017)Konyushkova, Raphael, and
  Fua]{Konyushkova2017}
K.~Konyushkova, S.~Raphael, and P.~Fua.
\newblock Learning active learning from data.
\newblock In \emph{Proceedings of the 31st International Conference on Neural
  Information Processing Systems}, NIPS'17, page 4228^^e2^^80^^934238, Red
  Hook, NY, USA, 2017. Curran Associates Inc.
\newblock ISBN 9781510860964.

\bibitem[Krause and Guestrin(2007)]{Krause2007}
A.~Krause and C.~Guestrin.
\newblock Nonmyopic active learning of gaussian processes: An
  exploration-exploitation approach.
\newblock In \emph{Proceedings of the 24th International Conference on Machine
  Learning}, ICML '07, pages 449--456, New York, NY, USA, 2007. ACM.
\newblock ISBN 978-1-59593-793-3.
\newblock \doi{10.1145/1273496.1273553}.
\newblock URL \url{http://doi.acm.org/10.1145/1273496.1273553}.

\bibitem[Laws and Sch{\"u}tze(2008)]{Laws2008}
F.~Laws and H.~Sch{\"u}tze.
\newblock Stopping criteria for active learning of named entity recognition.
\newblock In \emph{{Proceedings of the 22nd International Conference on
  Computational Linguistics COLING 08}}, volume~1, pages 465--472. Association
  for Computational Linguistics, 2008.

\bibitem[Lewis and Gale(1994)]{Lewis1994}
D.~D. Lewis and W.~A. Gale.
\newblock A sequential algorithm for training text classifiers.
\newblock In B.~W. Croft and C.~J. van Rijsbergen, editors, \emph{SIGIR '94},
  pages 3--12, London, 1994. Springer London.
\newblock ISBN 978-1-4471-2099-5.

\bibitem[McAllester(1999)]{McAllester1999}
D.~A. McAllester.
\newblock {Some PAC-Bayesian theorems}.
\newblock \emph{Machine Learning}, 37\penalty0 (3):\penalty0 355--363, Dec
  1999.
\newblock ISSN 1573-0565.
\newblock \doi{10.1023/A:1007618624809}.
\newblock URL \url{https://doi.org/10.1023/A:1007618624809}.

\bibitem[McAllester(2003)]{McAllester2003}
David McAllester.
\newblock Simplified pac-bayesian margin bounds.
\newblock In Bernhard Sch{\"o}lkopf and Manfred~K. Warmuth, editors,
  \emph{Learning Theory and Kernel Machines}, pages 203--215, Berlin,
  Heidelberg, 2003. Springer Berlin Heidelberg.
\newblock ISBN 978-3-540-45167-9.

\bibitem[Nguyen and Smeulders(2004)]{Nguyen2004}
H.~T. Nguyen and A.~W.~M. Smeulders.
\newblock Active learning using pre-clustering.
\newblock In \emph{International Conference on Machine Learning}, pages
  623--630, 2004.
\newblock URL
  \url{https://ivi.fnwi.uva.nl/isis/publications/2004/NguyenICML2004}.

\bibitem[Olsson and Tomanek(2009)]{Olsson2009}
F.~Olsson and K.~Tomanek.
\newblock An intrinsic stopping criterion for committee-based active learning.
\newblock In \emph{Proceedings of the Thirteenth Conference on Computational
  Natural Language Learning}, CoNLL '09, pages 138--146, Stroudsburg, PA, USA,
  2009. Association for Computational Linguistics.
\newblock ISBN 978-1-932432-29-9.
\newblock URL \url{http://dl.acm.org/citation.cfm?id=1596374.1596398}.

\bibitem[Scheffer et~al.(2001)Scheffer, Decomain, and Wrobel]{Scheffer2001}
T.~Scheffer, C.~Decomain, and S.~Wrobel.
\newblock Active hidden markov models for information extraction.
\newblock In F.~Hoffmann, D.~J. Hand, N.~Adams, D.~Fisher, and G.~Guimaraes,
  editors, \emph{Advances in Intelligent Data Analysis}, pages 309--318,
  Berlin, Heidelberg, 2001. Springer Berlin Heidelberg.
\newblock ISBN 978-3-540-44816-7.

\bibitem[Schohn and Cohn(2000)]{Schohn2000}
G.~Schohn and D.~Cohn.
\newblock Less is more: Active learning with support vector machines.
\newblock In \emph{Proceedings of the Seventeenth International Conference on
  Machine Learning}, ICML '00, pages 839--846, San Francisco, CA, USA, 2000.
  Morgan Kaufmann Publishers Inc.
\newblock ISBN 1-55860-707-2.
\newblock URL \url{http://dl.acm.org/citation.cfm?id=645529.657802}.

\bibitem[Seldin et~al.(2012)Seldin, Cesa-Bianchi, Auer, Laviolette, and
  Shawe-Taylor]{Seldin2012}
Y.~Seldin, N.~Cesa-Bianchi, P.~Auer, F.~Laviolette, and J.~Shawe-Taylor.
\newblock {PAC-Bayes-Bernstein} inequality for martingales and its application
  to multiarmed bandits.
\newblock In D.~Glowacka, L.~Dorard, and J.~Shawe-Taylor, editors,
  \emph{Proceedings of the Workshop on On-line Trading of Exploration and
  Exploitation 2}, volume~26 of \emph{Proceedings of Machine Learning
  Research}, pages 98--111, Bellevue, Washington, USA, 02 Jul 2012. JMLR
  Workshop and Conference Proceedings.
\newblock URL \url{http://proceedings.mlr.press/v26/seldin12a.html}.

\bibitem[Sener and Savarese(2018)]{Sener2018}
O.~Sener and S.~Savarese.
\newblock Active learning for convolutional neural networks: A core-set
  approach.
\newblock In \emph{ICLR (Poster)}. OpenReview.net, 2018.
\newblock URL
  \url{http://dblp.uni-trier.de/db/conf/iclr/iclr2018.html#SenerS18}.

\bibitem[Settles(2009)]{Settles2009}
B.~Settles.
\newblock Active learning literature survey.
\newblock Computer Sciences Technical Report 1648, University of
  Wisconsin--Madison, 2009.
\newblock URL
  \url{http://axon.cs.byu.edu/~martinez/classes/778/Papers/settles.activelearning.pdf}.

\bibitem[Settles and Craven(2008)]{Settles2008}
B.~Settles and M.~Craven.
\newblock An analysis of active learning strategies for sequence labeling
  tasks.
\newblock In \emph{Proceedings of the 2008 Conference on Empirical Methods in
  Natural Language Processing}, pages 1070--1079, Honolulu, Hawaii, October
  2008. Association for Computational Linguistics.
\newblock URL \url{https://www.aclweb.org/anthology/D08-1112}.

\bibitem[Seung et~al.(1992)Seung, Opper, and Sompolinsky]{Seung1992}
H.~S. Seung, M.~Opper, and H.~Sompolinsky.
\newblock Query by committee.
\newblock In \emph{Proceedings of the Fifth Annual ACM Workshop on
  Computational Learning Theory}, Proceedings of the Fifth Annual ACM Workshop
  on Computational Learning Theory, pages 287--294. Publ by ACM, January 1992.
\newblock ISBN 089791497X.
\newblock \doi{10.1145/130385.130417}.
\newblock Proceedings of the Fifth Annual ACM Workshop on Computational
  Learning Theory ; Conference date: 27-07-1992 Through 29-07-1992.

\bibitem[Taguchi et~al.(2021)Taguchi, Hino, and Kameyama]{THK2021}
Y.~Taguchi, H.~Hino, and K.~Kameyama.
\newblock Pre-training acquisition functions by deep reinforcement learning for
  fixed budget active learning.
\newblock \emph{Neural Information Processing Letters}, 0:\penalty0 000--000,
  2021.
\newblock \doi{10.1007/s11063-021-10476-z}.

\bibitem[Tomanek et~al.(2007)Tomanek, Wermter, and Hahn]{Tomanek2007}
K.~Tomanek, J.~Wermter, and U.~Hahn.
\newblock An approach to text corpus construction which cuts annotation costs
  and maintains reusability of annotated data.
\newblock In \emph{Proceedings of the 2007 Joint Conference on Empirical
  Methods in Natural Language Processing and Computational Natural Language
  Learning ({EMNLP}-{C}o{NLL})}, pages 486--495, Prague, Czech Republic, June
  2007. Association for Computational Linguistics.
\newblock URL \url{https://www.aclweb.org/anthology/D07-1051}.

\bibitem[Vlachos(2008)]{Vlachos2008}
A.~Vlachos.
\newblock A stopping criterion for active learning.
\newblock \emph{Comput. Speech Lang.}, 22\penalty0 (3):\penalty0 295--312, July
  2008.
\newblock ISSN 0885-2308.
\newblock \doi{10.1016/j.csl.2007.12.001}.
\newblock URL \url{http://dx.doi.org/10.1016/j.csl.2007.12.001}.

\bibitem[Xu et~al.(2003)Xu, Yu, Tresp, Xu, and Wang]{Xu2003}
Z.~Xu, K.~Yu, V.~Tresp, X.~Xu, and J.~Wang.
\newblock Representative sampling for text classification using support vector
  machines.
\newblock In F.~Sebastiani, editor, \emph{Advances in Information Retrieval},
  pages 393--407, Berlin, Heidelberg, 2003. Springer Berlin Heidelberg.
\newblock ISBN 978-3-540-36618-8.

\bibitem[Yang et~al.(2015)Yang, Ma, Nie, Chang, and Hauptmann]{Yang2015}
Y.~Yang, Z.~Ma, F.~Nie, X.~Chang, and A.~G. Hauptmann.
\newblock Multi-class active learning by uncertainty sampling with diversity
  maximization.
\newblock \emph{International Journal of Computer Vision}, 113\penalty0
  (2):\penalty0 113--127, June 2015.
\newblock ISSN 0920-5691.
\newblock \doi{10.1007/s11263-014-0781-x}.

\bibitem[Zhu(2007)]{Zhu2007}
J.~Zhu.
\newblock Active learning for word sense disambiguation with methods for
  addressing the class imbalance problem.
\newblock In \emph{In Proceedings of ACL}, pages 783--790, 2007.

\bibitem[Zhu et~al.(2008{\natexlab{a}})Zhu, Wang, and Hovy]{Zhu2008}
J.~Zhu, H.~Wang, and E.~Hovy.
\newblock Learning a stopping criterion for active learning for word sense
  disambiguation and text classification.
\newblock In \emph{Proceedings of the Third International Joint Conference on
  Natural Language Processing: Volume-{I}}, 2008{\natexlab{a}}.
\newblock URL \url{https://www.aclweb.org/anthology/I08-1048}.

\bibitem[Zhu et~al.(2008{\natexlab{b}})Zhu, Wang, and Hovy]{Zhu2008b}
J.~Zhu, H.~Wang, and E.~Hovy.
\newblock Multi-criteria-based strategy to stop active learning for data
  annotation.
\newblock In \emph{Proceedings of the 22nd International Conference on
  Computational Linguistics (Coling 2008)}, pages 1129--1136, Manchester, UK,
  August 2008{\natexlab{b}}. Coling 2008 Organizing Committee.
\newblock URL \url{https://www.aclweb.org/anthology/C08-1142}.

\end{thebibliography}
\bibliographystyle{unsrtnat}

\end{document}